%% file: main.tex
\documentclass[sigconf]{acmart}
\setcopyright{none}
\renewcommand\footnotetextcopyrightpermission[1]{}
\acmConference{}{}{}

\usepackage{amsmath,amsthm}
\usepackage{graphicx}
\usepackage{booktabs,tabularx,multirow,longtable}
\usepackage{enumitem}
\usepackage{algorithm}
\usepackage{algpseudocode}
\usepackage{subcaption}
\usepackage{wrapfig}
\usepackage{afterpage}
\usepackage{dblfloatfix}
\usepackage{extdash}
\usepackage{pifont}
\usepackage[dvipsnames]{xcolor}

\theoremstyle{plain} % italic body, bold heading
\newtheorem{theorem}{Theorem}[section]
\newtheorem{proposition}{Proposition}
\theoremstyle{definition} % upright body, bold heading
\newtheorem{definition}{Definition}[section]
\definecolor{revision}{rgb}{0,0,0}

\AtBeginDocument{%
  }

\begin{document}

\title{ExDBSCAN: Explaining DBSCAN with \texorpdfstring{\\}{} Counterfactual Reasoning - Additional Material}

\author{Pernille Matthews}
\orcid{0000-0001-8694-5803}
\authornote{Authors contributed equally to this research.}
\email{matthews@cs.au.dk}
\affiliation{%
\institution{Aarhus University}
\department{Department of Computer Science}
  \city{Aarhus}
  \country{Denmark}
}

\author{Lena Krieger}
\orcid{0009-0005-1894-7539}
\authornotemark[1]
\email{l.krieger@fz-juelich.de}
\affiliation{\institution{IAS-8, Forschungszentrum Jülich}
\city{Jülich}
\country{Germany}
}
  \affiliation{
  \institution{LMU Munich, MCML}
  \city{Munich}
  \country{Germany}}
  %\affiliation{\institution{LMU Munich}
  %\city{Munich}
  %\country{Germany}}

\author{Tommaso Amico}
\orcid{0009-0000-4652-2793}
\authornotemark[1]
\email{tomam@cs.au.dk}

\affiliation{%
\institution{Aarhus University}
\department{Department of Computer Science}
  \city{Aarhus}
  \country{Denmark}
}

\author{Arthur Zimek}
\orcid{0000-0001-7713-4208}
\email{zimek@imada.sdu.dk}
\affiliation{%
  \institution{University of Southern Denmark}
  \department{Department of Computer Science and Mathematics}
  \city{Odense}
  \country{Denmark}
}

\author{Thomas Seidl}
\orcid{0000-0002-4861-1412}
\email{seidl@dbs.ifi.lmu.de}
\affiliation{%
  \institution{LMU Munich, DBS, MCML}
 % \department{Database Systems and Data Mining}
  \city{Munich}
  \country{Germany}
}

\author{Ira Assent}
\orcid{0000-0002-1091-9948}
 \email{ira@cs.au.dk}
\affiliation{%
 \institution{Aarhus University}
 \department{Department of Computer Science}
 \city{Aarhus}
 \country{Denmark}
 }
 \affiliation{
 \institution{IAS-8, Forschungszentrum Jülich}
 \city{Jülich}
 \country{Germany}}

\renewcommand{\shortauthors}{Matthews, Krieger, Amico et al.}

\begin{abstract}
 Clustering is an unsupervised technique for grouping data points by similarity. While explainability methods exist for supervised machine learning, they are not directly applicable to clustering, making it challenging to understand cluster assignments. This interpretability gap is particularly evident in the popular density-based method DBSCAN, which assigns points as inliers (cluster members in dense regions) or outliers (noise points in sparse regions). DBSCAN does not provide insight into why a particular point receives its assignment or whether its assignment is robust to small changes in the data. To address the lack of explainability, we introduce ExDBSCAN, a density-aware, post-hoc explanation method. ExDBSCAN offers actionable counterfactual explanations, with theoretical guarantees for validity. It generates multiple counterfactuals using a density-connected weighted graph, adopting a physics-inspired model that repels counterfactual candidates from one another (diversity), while pulling them toward the instance to explain (proximity). Empirical evaluation on $30$ tabular datasets comparing against four baselines shows that ExDBSCAN outperforms all baselines while attaining perfect validity and retrieving diverse, proximal counterfactuals.
\end{abstract}



\received{20 February 2007}
\received[revised]{12 March 2009}
\received[accepted]{5 June 2009}

\maketitle

\input{Sections/introduction}
\input{Sections/relatedWorks}
\input{Sections/methods}
\input{Sections/experiments}
\input{Sections/conclusions}
\begin{acks}
This work was partially funded by the \grantsponsor{1}{Pioneer Centre for AI}{}, DNRF grant number \grantnum[]{1}{P1}, partially supported by project \grantnum[]{2}{W2/W3-108} Initiative and Networking Fund of the \grantsponsor{2}{Helmholtz Association}{} and part-funded by the Marie Skłodowska-Curie Doctoral Network RELAX-DN, funded by \grantsponsor{3}{EU under Horizon Europe 2021-2027}{} FP Grant Agreement nr. \grantnum[]{3}{101072456} (www.relax-dn.eu/).
\end{acks}

\clearpage
\bibliographystyle{ACM-Reference-Format}
\bibliography{ref}
\appendix
\input{Sections/appendix}
\end{document}

%% file: Sections/introduction.tex
\section{Introduction}
    Clustering is a cornerstone of unsupervised machine learning, enabling discovery of patterns and groupings in data without relying on predefined labels. Our work focuses on Density-Based Spatial Clustering of Applications with Noise (DBSCAN)~\citep{Ester1996density}, an algorithm known for its ability to find clusters of arbitrary shape and to identify outliers (also denoted as noise) in the process. DBSCAN is popular across domains ranging from computer vision~\citep{figueiredo2024analyzing} to bioinformatics~\citep{francis2011simulation} and fraud detection~\citep{luo2024ai}. Nevertheless, a limitation of DBSCAN (and clustering methods generally) is their lack of explanation, i.e., the absence of information on \emph{why} a particular point belongs to a particular cluster or not, which is crucial for interpretability, user trust, and actionable insights~\citep{miller2019explanation, molnar2020interpretable}.
    %While statistical analysis may be helpful to interpret Gaussian-shaped clusterings, this is not the case for arbitrary-shaped clusters, that can generally not be described by their convex hull.

    In supervised learning, counterfactual explanations (CEs) have emerged as an impactful tool, identifying minimal changes to an input that would alter the model's prediction, offering intuitive ``what-if'' scenarios~\citep{guidotti2024counterfactual,poyiadzi2020FACE,karimi2020algorithmic,mothilal2020dice}. Counterfactual reasoning builds user trust in the model and reflects how human stakeholders formulate explanations, making counterfactuals easy to interpret~\citep{miller2019explanation}. For classification tasks, a typical CE shows how a rejected loan application could be altered to become approved.

    %We claim that CEs for clustering bring important benefits. Beyond explaining individual points, multiple counterfactuals can reveal alternative ways in which a point could transition between clusters, offering complementary perspectives on the underlying structure. Furthermore, analysing how the set of counterfactual changes across several explained points enables the possibility of highlighting which features are most influential for cluster membership. For instance, we can identify if outliers systematically need to increase on a particular feature to become inliers, or if points from one cluster need to change a specific attribute to belong to another. These patterns not only enhance interpretability but also enable practitioners to make informed decisions about data preprocessing, feature engineering, and outlier treatment based on a deeper understanding of the clustering structure. 

    CEs are beneficial for clustering. Beyond explaining individual points, multiple counterfactuals can reveal alternative ways in which a point could transition between clusters, offering complementary perspectives on the underlying structure. Furthermore, analysing how counterfactuals change across several explained points can suggest which features are most relevant to cluster membership. These patterns not only enhance interpretability but also have the potential to help practitioners make more informed decisions about data preprocessing, feature engineering, and outlier treatment. While a formal user study is beyond the scope of this work, we illustrate this interpretive potential through concrete counterfactual examples on tabular datasets.

    Although DBSCAN provides transparent mechanisms including interpretable rules, they do not answer the fundamental question of \emph{what should I change}. For noise points specifically, DBSCAN identifies isolation but cannot indicate which features lead to isolation or how to resolve it.
    Devising CEs for DBSCAN presents major challenges. First, the cluster assignment process is not based on a differentiable loss function, preventing the counterfactual method from having access to model gradients, often used to guide counterfactual search~\citep{mothilal2020dice, watcher2017counterfactual}. Moreover, while classification algorithms often produce class probabilities, a membership probability, used by several model-agnostic counterfactual methods~\citep{mothilal2020dice, yang2022mace}, is difficult to define in hard clustering contexts. Furthermore, DBSCAN can label isolated points as noise, which can be far from each other, scattered across the whole feature space and share few similarities, giving them a boundary very different from classes in a classification scenario. 

    Model-agnostic Bayesian and genetic counterfactual methods for classification generally do not need access to gradients or class probabilities. Their application, however, is mostly computationally expensive~\citep{romashov22baycon}. Furthermore, different clustering algorithms rely on different concepts of similarity. This can impact counterfactual search, for example when ensuring diversity, i.e.,\ making sure that different counterfactuals for the same input are not repetitive but instead add meaningful actionable alternatives. In the DBSCAN case, two samples could be close in the traditional Euclidean distance leveraged by most model-agnostic methods, while at the same time, being connected only by a long density-connected path, making the two points effectively distant.

    In this work, we address these challenges by leveraging counterfactual reasoning to explain density-based clustering. We propose \textbf{Ex}plaining \textbf{DBSCAN} with counterfactual reasoning (ExDBSCAN), the first method tailored to generate CEs for DBSCAN. Given a specific point, DBSCAN's assignments, and its parameter settings, ExDBSCAN generates counterfactuals that are both proximal (close to the original instance for realistic changes) and diverse (providing multiple non-redundant alternatives), while accounting for DBSCAN's similarity definition. To achieve this, ExDBSCAN models the problem as a physical optimisation system, defining an undirected weighted graph that captures DBSCAN's density-based similarity structure. The system treats candidate counterfactuals as charged-like particles that repel each other to ensure diversity, and uses spring-like forces to maintain proximity to the explained instance. Our approach provably provides perfect validity, i.e.,\ every counterfactual attains the correct cluster assignment, for both noise-to-cluster and cluster-to-cluster transitions. Furthermore, ExDBSCAN respects feature constraints by handling non-actionable attributes, e.g.,\ the number of previous hospital visits in a tabular dataset of clinical analysis, where changing these attributes would make the explanation unrealistic, with the goal of ensuring explanations remain practical and applicable. 
    
    Our contributions include:
    \begin{itemize}%[nosep,leftmargin=10pt]
        \item We introduce ExDBSCAN, the first method to generate counterfactual explanations designed for DBSCAN, covering both noise-to-cluster and cluster-to-cluster transitions.

        \item ExDBSCAN generates diverse counterfactuals under our novel physics-inspired model that repels candidate counterfactuals from one another (diversity) while pulling them toward the point to explain (proximity), while also accounting for cluster structure captured in our graph representation.

      %  \item We provide theoretical guarantees that every counterfactual is valid, which extends to non-actionable features for cases where a feasible counterfactual exists.

        \item On $30$ OpenML tabular datasets, ExDBSCAN outperforms all baselines by obtaining  lower proximity, higher multi-counterfactual diversity and perfect validity.
        %Performances remain superior when a subset of the features is deemed non-actionable.
    \end{itemize}

%% file: Sections/relatedWorks.tex
\section{Related Work}\label{sec:related_work}
    % In the following related work, we outline supervised model-agnostic CE methods, followed by explainable clustering analysis and finally, counterfactual explanations for clustering and outlier detection.
    %We review supervised, model-agnostic counterfactuals; explainable clustering; and counterfactuals for clustering and outliers.
    %
%
    \textbf{Supervised Counterfactual Explanations.} In supervised learning, \citet{watcher2017counterfactual} formally introduce CEs as answers to ``what is the smallest change to input $x$ that would alter the model's prediction?'', and subsequent work optimises criteria such as \emph{proximity}, \emph{validity}, \emph{diversity}, and \emph{actionability}~\citep{guidotti2024counterfactual,mehdiyev2024counterfactual,yuan2024multigranular}. The method by~\citet{watcher2017counterfactual} makes use of gradients to guide the counterfactual search and can thus only be applied to differentiable methods.
       
   DiCE~\citep{mothilal2020dice}, extending this concepts to account for \emph{diversity}, returns a set of diverse CEs. While DiCE was originally designed for differentiable models only, it has since been adapted to be model-agnostic. FACE~\citep{poyiadzi2020FACE} follows data-supported feasible paths on the empirical manifold and returns a plausible path together with a counterfactual endpoint that alters model prediction. MACE~\citep{yang2022mace} uses a reinforcement-learning policy and returns one or more counterfactuals with minimal changes. The model-agnostic Growing Spheres algorithm~\citep{laugel2023achieving} explores the input space by expanding hyperspheres until the model decision boundary is reached.

   Another model-agnostic counterfactual method is BayCon~\citep{romashov22baycon}. Unlike gradient-based methods, BayCon uses Bayesian optimisation with probabilistic feature sampling to explore the feature space without requiring differentiable model outputs. It treats the clustering algorithm as a black box, repeatedly querying it with candidate counterfactuals and using the discrete cluster assignments to guide its search. However, it does not exploit DBSCAN's specific density-connectivity structure that ExDBSCAN leverages for efficient and theoretically grounded counterfactual generation.

   All methods need to repeatedly query a supervised model for class probabilities at candidate points and then choose based on closeness to the target class. By contrast, DBSCAN provides only discrete, connectivity-based memberships (core, border, noise) defined by $\varepsilon$ and $minPts$, with no continuous quantity to optimise against. A workaround could be to fit a surrogate classifier to predict cluster memberships and then apply supervised CE methods; however, the surrogate’s decision regions would then reflect its own similarity notion rather than density-reachability.

    \textbf{Explainable Clustering.} 
        %The need to understand and explain clustering results has motivated various approaches to clustering interpretability. Unlike supervised learning, where explanations often focus on decision boundaries, clustering explanations must describe why groups form and what distinguishes them. 
    Global explainable clustering methods aim to summarise entire clusterings or cluster characteristics in human-understandable forms. \cite{moshkovitz2020explainable} replace $k$-means/$k$-medians clusters with compact interpretable decision trees. Similarly, \cite{bandyapadhyay2023find} search for small trees that match a given partitioning. Prototype-based methods characterise each cluster by representative examples or “criticism” points, e.g., \cite{kim2016mmdcritic} select prototypical instances and outliers for each cluster. Such global explanations convey what defines clusters overall, but they do not explain individual assignments.

    In contrast, local explainability focuses on specific data points. Agnostic attribution methods like FACT~\citep{scholbeck2023fact} quantify feature influence on a point's cluster assignment by perturbing features and observing changes in the clustering. Interactive tools like ClusterVision~\citep{kwon2018clustervision} visualise effects of input parameters on clusters and outliers. 

    While global summaries, prototypes, feature attributions, and visual analytics improve understanding of clustering results, they do not provide actionable local changes for a given point. No existing method can explain how a particular DBSCAN inlier or outlier could change assignment through minimal feature changes, with guarantees of valid membership under DBSCAN’s density-reachability. This gap motivates ExDBSCAN.%, and focuses on local actionability which generates concrete perturbations to an individual point guaranteeing a change in its DBSCAN assignment.
        %a capability that prior clustering explanation methods lack.
        % take away 
        %In DBSCAN, clusters can form complex non-convex shapes where membership depends on density-reachability, a point may satisfy all tree conditions yet be labelled as noise if it lacks sufficient neighbours, or belong to a cluster through a chain of density-connected points that no single threshold can capture. While these provide intuitive cluster summaries, prototype centrality fundamentally misaligns with DBSCAN's mechanisms as membership depends on local neighbour counts and connectivity paths rather than similarity to central points. Such attributions are useful globally, yet for density-based clustering, they do not provide actionable, point-wise instructions that guarantee a change of membership under DBSCAN's $\varepsilon$-reachability criterion. These systems aid understanding and model selection but do not answer the counterfactual question ``what minimal change moves this specific point into another cluster?'' nor do they provide DBSCAN-validity guarantees. These approaches provide a concise global structure but rely on threshold-based rules that poorly capture DBSCAN's behaviour.

    \textbf{Counterfactuals for Clustering or Outlier Detection.} 
        Some recent work studies CEs for clustering. For prototype/mixture-style clustering, \citet{vardakas2025cfkmeans} provide the first analytic formulations: for $k$-means, they derive a closed-form move that crosses the Voronoi boundary into a target cluster, and for Gaussian clustering, devise and solve a non-linear formulation. % (under full/diagonal/spherical covariances)
        These methods leverage geometry and thus do not apply to density-based algorithms, where membership is defined by density-reachability rather than distances to prototypes.
        \citet{spagnol2024counterfactual} build on BayCon and propose model-specific scores: a distance-based score for $k$-means (scaled distance to the target centroid) and a membership-probability score for HDBSCAN* (using the algorithm’s soft-clustering outputs). These scores guide a zero-order Bayesian optimisation loop to find CEs. While this makes BayCon-style search applicable to HDBSCAN*, it uses HDBSCAN*'s probabilistic model, unavailable for DBSCAN. Moreover, it provides no density-reachability guarantees and does not explain DBSCAN’s core/border/noise status.

        Beyond clustering, counterfactuals have been explored for unsupervised outlier detection. \citet{zhou2025eace} propose EACE which for a given outlier generates one or more counterfactual samples that an anomaly detection model would classify as an inlier. \citet{Angiulli2023counterfactuals} similarly "repair" an outlier via a subset of features and adjusted values (a ``mask'') that would turn it into an inlier. 
        Both EACE and the density-contrastive “repair” methods target outlier detection models that expose a continuous anomaly score (e.g., one-class SVM, kernel density estimators, autoencoders, LOF-style scores). Their goal is to decrease the anomaly score so that a point crosses a detector-specific decision threshold and becomes ``normal,'' suitable for detectors where ``inlierness'' is defined by that score. DBSCAN does not provide such a score, soft memberships, or gradients; it yields discrete core/border/noise assignments based on $\varepsilon$-reachability and $minPts$. Consequently, these counterfactual outlier detection methods optimise an objective that is orthogonal to ours, since they can make a point ``less anomalous'' in a detector’s sense without guaranteeing that it becomes density-connected to a DBSCAN cluster. In contrast, ExDBSCAN includes both cluster-to-cluster and noise-to-cluster transitions.

%% file: Sections/methods.tex
\section{Counterfactual Explanations for DBSCAN}\label{sec:method}
   % We describe DBSCAN and counterfactual explanations (Sect.~\ref{subsec:dbscan}), before  introducing density-based counterfactuals (Sect.~\ref{subsec:cec}) and our novel method, ExDBSCAN (Sect.~\ref{subsec:cfexp}). 
    %An overview of notation is provided in App.~\ref{app:notation}.
%
 %   \subsection{Preliminaries}\label{subsec:dbscan}
        Consider dataset $X=(x_1, \dots, x_l), x_i \in \mathbb{R}^n$ and cluster assignment from points $\mathcal{C}: X \rightarrow \{-1, 0,\dots, {m-1}\}$ to one of $m$ cluster labels, or to noise ($-1$). 
        The popular density-based clustering approach DBSCAN~\citep{Ester1996density} groups points based on a similarity notion modelled as density-reachability. Clusters are maximal areas of density-connected points, separated from other clusters through sparse areas. It uses two parameters: $\varepsilon$, which defines the radius of local density assessment, and $minPts$, which specifies the minimum number of points required to form a dense region. The \emph{$\varepsilon$-neighbourhood} are the points within radius $\varepsilon$ of a point $p$: $\mathcal{N}_{\varepsilon}(p) = \{ q \mid \|p-q\|_2\leq\varepsilon, \emph{ with } p,q \in X \}$. 
        DBSCAN distinguishes between dense \emph{core points},  \emph{border points} close to core points, and \emph{noise points} or outliers, that do not belong to any cluster: 
        \begin{definition}[Core, border, and noise points]
            The set of \emph{core points} $Q$ are points $p$ that have at least $minPts$ points in their $\varepsilon$-neighbourhood:
             \(%begin{equation}\label{eq:corepoint}
                    Q = \{ p \in X \mid \vert\mathcal{N}_{\varepsilon}(p)\vert \geq minPts \}.\,
            \)%end{equation}
          Points in $X \setminus Q$ are \emph{border points} $B$, if they have at least one core point in their neighbourhood, else they are \emph{noise}.
        \end{definition}

        \paragraph{Counterfactual Explanations (CEs).} CEs were first proposed as an optimisation problem by~\citet{watcher2017counterfactual} for classifiers:
        \begin{definition}[Counterfactual Explanations]\label{def:cfexp} 
                A set of $k$ CEs for a point $p$ under a classification model $f$, denoted as $C_k(p)$ minimises a cost function for which the classification outcome changes: 
            \begin{equation}\label{eq:cfexp}
                  C_k(p) = \underset{p_1',\dots,p_k'}{\mathrm{argmin}} \emph{ cost}(p, p') \emph{ subject to } f(p) \neq f(p')
            \end{equation}
        \end{definition}
        Counterfactual generation methods vary in the choice of cost function, typically balancing between desired properties, e.g., the distance between the original point $p$ and the counterfactual $p'$ or the pairwise distance between counterfactuals. 
        In the context of density-based clustering, the cost function in Definition~\ref{def:cfexp} takes into account the clustering algorithm's concept of similarity. 
            
        \subsection{Characteristics of Counterfactuals}\label{subsec:cec}
        Counterfactual explanations provide value for clustering tasks by addressing the interpretability challenges in unsupervised learning. Unlike supervised learning where model decisions can be evaluated against known labels, clustering methods function without predefined classes. This leaves the task of understanding why certain groups are formed or why points are assigned to a particular cluster. CEs bridge this gap by revealing the minimal changes required to alter a point's cluster assignment, thus explaining group differences and features that may drive cluster membership. Density-based clustering finds clusters of arbitrary shapes and can identify noise points. These capabilities can make the resulting clusterings difficult to interpret because the complex, non-convex cluster boundaries and the distinction between core, border, and noise points are not immediately obvious. Users cannot easily identify which specific features or density thresholds have led to the assignments. Noise-to-cluster and cluster-to-cluster CEs reveal key features that determine density-based cluster membership. % and explain noise-to-cluster and cluster-to-cluster counterfactuals.

        \textbf{Cluster Assignment.} 
            Counterfactual generation requires assigning out-of-dataset points so we can test whether a candidate lies in the target cluster. Because DBSCAN lacks a native assignment for new points, we keep the clustering fixed and define $\mathcal{C}$: %a point belongs to cluster $i$ if it falls within the $\varepsilon$-neighbourhood of any core point $q$ with $\mathcal{C}(q)=i$. $p$ belongs to cluster $i$ \textit{iff} there exists a core point $q$ with $\mathcal{C}(q) = i$ such that $p \in \mathcal{N}_{\varepsilon}(q)$.
            point $p$ belongs to cluster $i$ \textit{iff} there exists a core point $q$ with $\mathcal{C}(q) = i$ within the $\varepsilon$-neighbourhood, such that $q \in \mathcal{N}_{\varepsilon}(p)$.

%$$\mathcal{C}(p)=i \iff \exists q \in \mathcal{N}_{\varepsilon}(p), \text{s.t. } q \in  {C}_i \wedge \vert\mathcal{N}_{\varepsilon}(q)\vert \geq minPts $$

        %The counterfactual generation process requires defining cluster assignments for points outside the original dataset, such that we can determine whether a counterfactual candidate indeed is part of the target cluster. 
        %Unlike classifiers, DBSCAN lacks a native assignment function for new points. To ensure that we explain a given clustering without inadvertently manipulating it, and avoid having to re-cluster, we keep the original clustering fixed and define a cluster assignment function $\mathcal{C}$, where a point belongs to cluster $i$ if it falls within the $\varepsilon$-neighbourhood of any of the cluster's core points $q$:

        We restrict membership to core-point neighbourhoods, as placing a CE near a border does not ensure density-connectivity (and may violate DBSCAN's membership definition). Fixing the clustering also avoids the computational and conceptual pitfalls of re-clustering (e.g., merges or new clusters), ensuring we explain the original structure rather than one altered by the CEs. 

        %We restrict membership to neighbourhoods of core points, as placing a counterfactual only within the neighbourhood of a border point does not guarantee density-connectivity to the cluster (and thus may fail to satisfy DBSCAN's definition of cluster membership). Keeping a fixed clustering structure also avoids the computational and conceptual complications of re-clustering (e.g., potential cluster merging or formation of new clusters), ensuring that we explain the original clustering structure, and not the result of counterfactual explanations themselves.

        Generating multiple diverse CEs, rather than a single instance, enhances explanatory value by probing cluster-assignment boundaries from different directions, giving stakeholders a thorough view of the clustering algorithm's behaviour. Multiple examples help surface explanatory trends (e.g., features repeatedly changed likely matter for assignment) and reduce perceived arbitrariness, improving trust~\citep{miller2019explanation}. Regarding actionability (feasibility of the proposed change), options differ by context, e.g., in a medical setting, changing the respiratory rate vs.\ the oxygen saturation may be more or less realistic for a given patient; thus, offering several paths raises the chance that at least one is realistic and acceptable~\citep{stepin2021survey}. 
        
        %Offering several options increases the likelihood that at least one constitutes a realistic, acceptable intervention, enabling feasible and actionable paths to change the outcome of the model~\citep{stepin2021survey}.
        %e.g., adjusting respirayory rate vs.\ oxygen saturdation in clinical   
        %Generating multiple diverse CEs rather than a single instance enhances explanatory value by exploring cluster assignment boundaries from various directions, providing stakeholders with a more comprehensive understanding of the clustering algorithm's behaviour. Multiple examples helps in spotting explanatory trends, e.g., a feature which is heavily changed in several CEs likely carries high importance in the clustering method's assignment. From a cognitive stand point, not stopping with one CE increases user trust with the explanation looking less arbitrary and more robust as it is corroborated by multiple alternatives~\citep{miller2019explanation}. Regarding actionability, i.e.,\ how feasible the proposed counterfactual change is, different paths offer different feasibilities, which largely depend on the application scenario. For example in a medical setting, in the hypothetical scenario of changing the respiratory rate versus the oxygen saturation may be more or less realistic for a given patient. Offering several options increases the likelihood that at least one constitutes a realistic, acceptable intervention, enabling feasible and actionable paths to change the outcome of the model~\citep{stepin2021survey}.

        It is thus crucial that CEs demonstrate diversity~\citep{laugel2023achieving}. CEs that differ only marginally are redundant and fail to highlight differences among clusters~\citep{mothilal2020dice}. High redundancy adds no valuable information, for example, two counterfactuals that vary by only $0.1$ breaths per minute in respiratory rate are effectively indistinguishable and convey the same information.
        While diversity is essential, CEs that stray too far from the explained point risk suggesting unrealistic changes. If a CE bears little in common to the point being explained, human stakeholders may struggle to relate the two and reason about their differences. Proximity therefore remains fundamental to counterfactual reasoning, as it increases both the realism and feasibility while reducing cognitive load for stakeholders, as there are fewer overall changes~\citep{miller2019explanation}.
        %While noise points can be reassigned to clusters, we deliberately do not include the case of converting inliers to noise. 
 %
\subsection{ExDBSCAN Counterfactuals}\label{subsec:cfexp}
         %In the following, we now address the scenario where a specific target cluster is specified, slight modification of our algorithm in the setting we there is not a target cluster are discussed in Appendix X.
         ExDBSCAN addresses these requirements by generating valid sets of counterfactuals that balance proximity and diversity. Users can specify their number and the target cluster to meet their needs.

            \textbf{Inter-Cluster Concept of Distance.} 
            In CE literature, density is used as a proxy for the ease of moving between points or feasibility~\citep{poyiadzi2020FACE, zhou2025eace, yamao2024distribution, kanamori2020dace}. In DBSCAN, cluster members are density-connected, so a counterfactual that moves within a cluster stays in dense regions with similar neighbours, enabling realistic paths. However, because DBSCAN discovers non-convex shapes, two points can be close in Euclidean~\footnote{In this paper, we use the Euclidean distance as it is the most popular metric. A discussion on different metrics is in App.~\ref{subsec:metricchoice}} yet require a long density-connected path; straight-line distance can therefore overstate similarity. 

            %\footnote{In this paper, we use the Euclidean distance as it is the most popular metric. A discussion on different metrics is in App.~\ref{subsec:metricchoice}} 
            
            %In CE literature, density is used as a proxy for the ease of moving between points or feasibility~\citep{poyiadzi2020FACE, zhou2025eace, yamao2024distribution, kanamori2020dace}. In DBSCAN, points belonging to a cluster are density-connected. Thus, a counterfactual moving within a cluster remains in a dense region with similar points nearby, enabling movement through a realistic and feasible path between two points in a cluster. However, as DBSCAN finds arbitrary cluster shapes, a practical problem appears when two points in a cluster are close in Euclidean (straight-line) distance but still require a long density-connected path to reach one another. In other words, the straight-line distance may suggest high similarity, yet DBSCAN considers the points far apart because many intermediate steps are required to connect them.

            CEs that look redundant under Euclidean distance, i.e., proposals close in straight-line distance, may still be meaningfully diverse once DBSCAN's density structure is considered. Euclidean close points can lie in different density-reachable regions and thus represent distinct alternatives. Relying on Euclidean distance risks ``cutting through'' clusters and ignoring connectivity, yielding explanations that misstate feasibility. We illustrate this in Appendix~\ref{app:euclideanLimitation}.
            
            %Similarly, counterfactuals that appear redundant under Euclidean distance, that is, proposals located close together in straight-line and therefore seemingly offering little new information, may still provide meaningful diversity once DBSCAN's density-based structure is taken into account. Two points that are Euclidean close can belong to very different regions of a cluster when judged by density-reachability, and thus represent distinct alternatives. Hence, relying on Euclidean distance risks "cutting through" clusters and ignoring connectivity, which may result in explanations that misrepresent the true feasibility of moving within the cluster. We illustrate this concept Appendix~\ref{subsec:euclideanLimitation}.

            To address this challenge and incorporate DBSCAN's notion of similarity, we define proximity as the length of the weighted density-connected path. For each cluster, we build an undirected weighted graph $G(V,E)$ whose vertices $V$ are core points; edges connect two vertices if they are directly density-reachable (distance $< \varepsilon$) and edge weight is their distance. The distance between any two points is the weighted shortest-path in $G$. Thus, instances are ``close'' only if an easy path connects them, allowing evaluation of diversity while respecting the cluster's internal structure. %Ignoring this structure can misrepresent separations and, thereby, diversity. 

            %To address the challenges and incorporate DBSCAN's notion of similarity, we define proximity within a cluster as the length of the weighted density-connected path between points. For a given cluster, we construct an undirected weighted graph $G(V,E)$ with the cluster's core points as vertices $V$. Two vertices connect with an edge $E_i$ if they are directly density-reachable (distance $< \varepsilon$), with the distance between points as the edge weight. We then measure the distance between any two points in a cluster as the weighted shortest path distance in graph $G$. Through the weighted path distance, two instances will be considered close only if the path connecting the them is easily traversable. This approach allows us to evaluate diverse proposals while considering the cluster's inner structure. Diversity is evaluated and captured through pairwise distances, not accounting for the cluster's structure risks misrepresenting the distance between instances and consequently diversity.

        \textbf{Reference Core Point.} Given our definition of the assignment function, a CE is assigned to the target cluster iff it falls within the $\varepsilon$-neighbourhood of a cluster's core points. When selecting a point in the space as a CE, we naturally associate it with its nearest core point. To properly leverage the weighted graph $G$ that captures the cluster's structure, we identify the most appropriate reference core point for each proposed CE and generate the CE example within its $\varepsilon$-neighbourhood. Hence, a set of ExDBSCAN CEs corresponds directly to a set of vertices in graph $G$, namely the cluster's core points. We denote the set of CEs for point $p$ as $C_k(p)$, while $C_k'(p)$ represents the corresponding set of reference core points.

        \textbf{Proximity Through Attraction.} By definition, the most proximal CEs are the ones closest to the point to explain $p$. Given $k$ CEs, the set of core points $C'(p)$ maximising proximity will thus be composed of the $k$ nearest core points of $p$. As $p$ and candidate core points are not density connected, the distance metric used is the one employed for fitting DBSCAN.

        \textbf{Diversity Through Repulsion.} 
        In CEs, a set is diverse when its members are mutually dissimilar~\citep{mothilal2020dice}. One class of methods enforces this by maximising pairwise distances; variants of the maximum diversity problem (MDP)~\citep{laugel2023achieving, ley2022diverse}. Another models diversity via repulsion, minimising similarity between selected instances; determinantal point processes are common realisations, probabilistically favouring diverse subsets~\citep{mothilal2020dice, kulesza2012determinantal, afrabandpey2022feasible}.

        %In counterfactual literature, a set of points is diverse when they exhibit dissimilarity to one another~\citep{mothilal2020dice}. Various algorithms achieve this through different mathematical formulations. One class of algorithms captures diversity by maximising the sum or average pairwise distance between CEs~\citep{laugel2023achieving, ley2022diverse}, thereby solving variations of the maximum diversity problem (MDP): selecting a subset that maximises the sum of pairwise distances from a given set of elements~\citep{marti2013heuristics}. An alternative class of algorithms models diversity through repulsion-like interactions, effectively minimising similarity between chosen instances~\citep{mothilal2020dice, kulesza2012determinantal, afrabandpey2022feasible}. This latter family implements variations of determinantal point processes, which probabilistically favours the selection of diverse items through minimal similarity~\citep{kulesza2012determinantal}.

        %Euclidean distance alone is inadequate for DBSCAN, whose similarity is density-connectivity. We therefore adopt a repulsion-based approach under this notion. As candidates converge, repulsive forces grow, preventing redundant CEs. Unlike pure distance maximisation, which can place several near-

        We observe that simply adopting (Euclidean) distance falls short for DBSCAN: its very notion of similarity is density-connectivity. Thus, the challenge is to model proximity and diversity using density-connectivity. We adopt a repulsion-based approach that prevents redundant counterfactuals, as repulsion forces increase substantially as points converge. In contrast, maximising pairwise distances remains more forgiving toward closely positioned instances, provided that other point pairs maintain sufficient separation to preserve overall diversity. Also, repulsion interactions naturally produce more evenly distributed selections across the solution space. In our density-based setting, we implement this diversity measure using the weighted shortest path distance defined by graph $G$ to evaluate pairwise distances in a cluster. This ensures that our diversity metric respects the cluster's density-connectivity structure rather than relying on simple Euclidean distances that could cut through the cluster and misrepresent the actual relationships between points.

        \textbf{Balancing Proximity and Diversity.}   
        ExDBSCAN generates multiple CEs that ensure both proximity and diversity; an inherent trade-off since closer alternatives are often less diverse. We resolve this by modelling a physical system and selecting its equilibrium (the minimum-energy configuration), where the proximity-diversity balance is optimal. 

        %For multiple counterfactuals, ExDBSCAN ensures both proximity and diversity. Their balance constitutes an inherent trade-off as proximal alternatives are often not diverse and vice-versa. To solve this multi-objective optimisation, we reformulate the problem as a physical system for which we find the equilibrium, i.e., the configuration that minimises the total energy of the system. At equilibrium, the trade-off between diversity and proximal is optimal.
%
        For diversity, we treat candidate core points as like-charged particles with Coulomb's electrostatic law~\citep{halliday2013fundamentals}, pushing already selected cores apart so the minimum-energy set is maximally spread; making the minimum energy configuration the most diverse. 
        
        %For diversity, we treat the elements of the set of core points as charged-like particles repulsing each other. The repulsive force scales as the inverse of the squared distance, according to Coulomb's electrostatic law~\citep{halliday2013fundamentals}. This pushes the selected core points to be as distant among each other as possible, making the minimum energy configuration the most diverse. 
%
        For proximity, we bias chosen core points to be close to the original point, by connecting the latter with a spring to candidate core points. The further the candidate CE, the stronger the force pulling it closer to the original instance, growing linearly with the weighted path distance according to Hooke's law~\citep{halliday2013fundamentals}. Prop.~\ref{prop:knn} states how the minimum energy configuration of spring-like interactions  selects the closest core points, accounting for proximity.

        \begin{figure*}[tbp]
            \centering
            \includegraphics[width=1\linewidth]{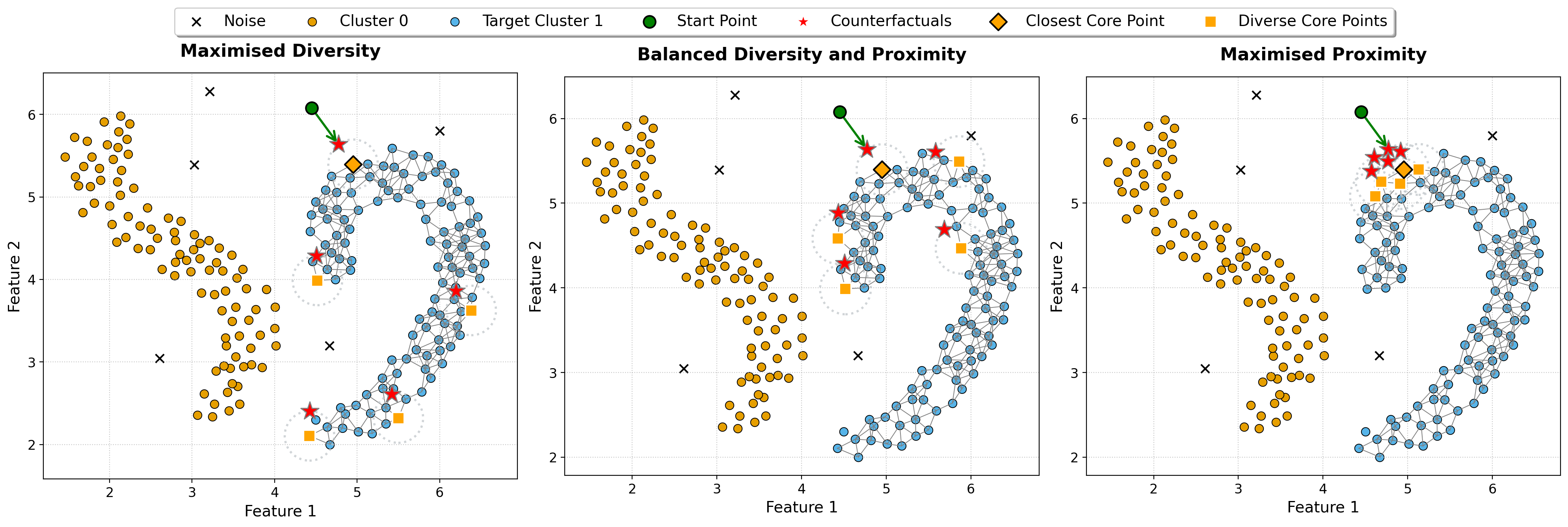}
            \caption{Optimising diversity (left), balanced (middle), optimising proximity (right). Cluster points in yellow and blue, resp.; noise marked with X; green point to be explained; red stars counterfactuals.}
            \Description{Conceptual image with three panels, that provide examples for optimising only for diversity (left), only proximity (right), and a combination of both (middle).}
            \label{fig:toyDataset}
        \end{figure*}

        The aim is to find the system's stable configuration, i.e., the one that minimises the energy of the physical system. Spring-like forces bias the system to select points closest to the original instance as springs push towards their rest state; i.e., when explained point and candidate core point have a null distance. Electrostatic-like repulsion favours core points to be pairwise distant (according to the shortest weighted path distance), as same-charged particles push each other away. The minimum energy configuration balances the springs' desire to go back to their rest configuration with the charges' desire to spread as much as possible. Hence, given a set $C'(p)$ of chosen core points, and point to explain $p$, the energy configuration of the system has the form:
        \begin{equation}\label{eq:energy}
            E_{C'} = \sum_{V_i\in C'}\sum_{V_j \in C' : j>i}\frac{1}{\mathcal{D}(V_i, V_j)} + \sum_{V_i \in C'}d^{2}(p, V_i)
        \end{equation} 
        $V_i$ is the $i$th vertex of graph $G$, $\mathcal{D}(V_i, V_j)$ denotes the weighted shortest path distance between vertices $V_i$ and $V_j$. Lastly, $d(p, V_i)$ is the Euclidean distance between the point to explain $p$ and vertex $V_i$.\footnote{Depending on data and cluster structure, distances $d(p, V_i)$ and $\mathcal{D}(V_i, V_j)$ might have different scales. Thus, we use normalisation as in App.~\ref{subsec:normalisation}.} We evaluate physical constants characterising spring and electrostatic terms in Eq.~\ref{eq:energy} as $1$, as we are only interested in how energy terms scale with distance. To find the optimal configuration of the system and thus, the set of $\mathcal{C}$, we find the subset $C'$ of vertices $V$ minimising system energy:
        \begin{equation}\label{eq:optimisationProblem}
            C_k'(p) = \operatorname*{argmin}_{B\in V \, : \, |B| = k} E_{B} \quad .
        \end{equation}
        The energy term in Eq.~\ref{eq:energy} can be seen as ExDBSCAN's cost function in the formalism introduced in Eq.~\ref{eq:cfexp}.
        Fig.~\ref{fig:toyDataset} shows how a solution of Eq.~\ref{eq:optimisationProblem} can be visualized on a toy dataset where counterfactuals are produced for a noise point. On the left of Fig.~\ref{fig:toyDataset}, considering the electrostatic term only leads to maximally diverse counterfactuals. On the right, using just spring-like terms make counterfactuals maximally proximal. In the middle, balancing repulsion and attraction considering the cluster's density structure, prompts proximal and spread-out counterfactuals. We observe the following:
        \begin{proposition}\label{prop:knn}
            Aligning with our desiderata on proximity, if considering only the spring-like term in Eq.~\ref{eq:energy}, solving the optimisation problem in Eq.~\ref{eq:optimisationProblem} is equivalent to selecting the $k$ nearest neighbours of the point to explain $p$.
        \end{proposition}
        \begin{proof} 
            From Eq.~\ref{eq:energy}, and by induction, when the set of core points is empty, the core minimising the elastic energy of the system is the closest to the point to explain $p$ (minimum  possible distance). As for the induction step, given a set of $k'< k$ selected core points, the next core point to be selected, i.e.,\ the one minimising the total elastic energy, is given by the closest non-selected core-point to $p$.
        \end{proof}
        Solving the optimisation problem in Eq.~\ref{eq:optimisationProblem} is computationally challenging.
        \begin{proposition}\label{prop:npHardness}
            Solving Eq.~\ref{eq:optimisationProblem}, i.e., the energy minimisation problem with both attraction and repulsion, is an NP-hard problem.
        \end{proposition}
        \begin{proof} (Sketch, full in Appendix~\ref{subsec:npHardness})

        Maximum diversity problems are NP-hard~\citep{parreno2021measuring}. Eq.~\ref{eq:optimisationProblem} restricted to the electrostatic term can be mapped to the MDP. Adding springs does not change NP-hardness: If it did, the MDP could be solved in polynomial time sending spring constants to zero.
        \end{proof}

         We approximate the solution using the following greedy algorithm. Given an incomplete set of selected core points $C_k'$ with $k'< k$, the next cluster point is chosen as the one minimising the total energy of the system according to Eq.~\ref{eq:energy}. When choosing the instance initialising the set of core points, $C'$ is still empty and thus no repulsive electrostatic-like force influences the choice. $C'$ will thus be initialised with the core point closest to the point to explain since it is the one greedily minimising the energy configuration, composed by a single spring-like term. 
         {\color{revision}Note in the case with only one counterfactual, the greedy optimisation corresponds to the optimal one according to Eq.~\ref{eq:energy}}.

        \textbf{Choosing the set of counterfactuals $C$.} 
        Our chosen set of core points $C'$ leverages the clusters' density structure to optimally balance proximity and diversity through the weighted undirected graph $G$. Given $C'$, we need to select where in each core point's $\varepsilon$-neighbourhood the counterfactual is going to be placed, practically mapping $C'$ to the final set of counterfactuals $C$.
        Given a single core point $q$ in $C'$, to further improve proximity and respecting our assignment function definition, we position the counterfactual $p' \in C$, $\varepsilon$ away in the direction of the original point $p$: 
        \begin{equation}\label{eq:cfgen}
            p'=p+(d_{pq}-\varepsilon) (q-p)/d_{pq}.
        \end{equation}
        %We illustrate this in Figure \ref{fig:concept_sketch}.
        The counterfactual point $p'$ is considered part of the cluster containing reference core point $q$, as it satisfies the density-based clustering criteria for the target cluster, effectively ensuring validity.

        \begin{theorem}\label{theo:1}
            Every counterfactual point $p'$ with reference core point $q$, such that $\mathcal{C}(q) = i,\, i\neq -1$, that is generated for a point $p 
             \in X \,:\, \mathcal{C}(p)\neq i$ toward cluster $i$ with the proposed method, is considered part of  cluster $i$, implying that $\mathcal{C}(p') \neq \mathcal{C}(p)$ and $C(p') = i$, and thus valid counterfactual.
        \end{theorem}

        \renewcommand\qedsymbol{$\blacksquare$}
        \begin{proof}\label{proof:1}
        (Sketch) The definition in Section~\ref{subsec:cec} denotes that $p'$ is in the same cluster as its reference core point $q$ with $\mathcal{C}(q)=i$, therefore $\mathcal{C}(p')\neq \mathcal{C}(p)$ and specifically $\mathcal{C}(p') = i$.
    \end{proof}%
   App.~\ref{subsec:nonActionable} details how to retrieve counterfactuals $C$ when some features are not actionable, and shows how ExDBSCAN creates a graph only considering reference core points whose $\varepsilon$ neighbourhood is reachable without changing non-actionable features.

%     \subsection{Limitations}\label{subsec:limitation}
% In some cases, generating counterfactuals towards clusters may not be the optimal solution. It may be that closer to the outlier more outliers are found and generating a counterfactual towards noise and thereby introducing a new virtual cluster may be the better option. Our current approach may identify this by investigating if the number of noise points may be sufficient enough including the `virtually corrected' noise point to form its own cluster. However, we currently do not investigate this due to the sequential approach. In future works, simultaneously generating counterfactuals specifically for noise and identifying possible `newly' formed clusters is an interesting direction.

%% file: Sections/experiments.tex
\section{Experiments}\label{sec:experiments}
    \paragraph{Experimental Setup.} We evaluate counterfactual explanations on $30$ tabular datasets from OpenML~\citep{OpenML2013}\footnote{Dataset details and reproducible code: \\ \href{https://github.com/tommasoamico/ExDBSCAN}{https://github.com/tommasoamico/ExDBSCAN}.}. For each dataset, DBSCAN hyperparameters $(\varepsilon,\text{minPts})$ are selected through a grid search to maximise the DBCV clustering validity index~\citep{dbcv}.

    \paragraph{Baseline Methods.} \textsc{ExDBSCAN} is compared against seven baseline methods and variants for counterfactual generation: 
    
    \begin{enumerate}
        \item \textbf{BayCon}~\citep{romashov22baycon}, a model-agnostic Bayesian counterfactual generator. BayCon uses Bayesian optimisation to propose candidate points; to model the objective posterior, it fits a Random Forest surrogate internally.

        \item \textbf{DiCE}~\citep{mothilal2020dice} uses an optimisation-based counterfactual formulation and incorporates diversity into the objective. DiCE is evaluated in two variants: \emph{DiCE-Direct}, where DBSCAN’s discrete assignment is queried during optimisation, and \emph{DiCE-Surrogate}, where DiCE optimises against a fitted surrogate classifier to obtain a smoother search space.

        \item \textbf{Growing Spheres (GS)}~\citep{laugel2023achieving} is a greedy method that samples instances on hyperspheres of increasing radius until a point crossing the decision boundary is found. We evaluate two variants, \emph{GS-Direct} on DBSCAN’s discrete assignments and \emph{GS-Surrogate} on top of a fitted surrogate model.

        \item \textbf{ExDBSCAN Random} is a heuristic baseline which samples $k$ reference core points uniformly at random from the target cluster’s core set, and constructs counterfactuals from these cores using the same core-point-based construction as \textsc{ExDBSCAN} (Section~\ref{sec:method}). This isolates the contribution of \textsc{ExDBSCAN}'s structured core-point selection (proximity-diversity optimisation) from the final counterfactual construction step.
    \end{enumerate}

    \paragraph{Surrogate Models and Training.}
        DBSCAN produces discrete assignments (inliers/outliers) and exposes neither class probabilities nor usable gradients. As a consequence, supervised counterfactual methods like DiCE and Growing Spheres do not have a smooth objective to optimise when applied directly; leading to unstable optimisation and difficulty crossing the clustering boundary. For fair comparison against these methods, evaluation occurs through a surrogate setting that fits a multi-class classifier that predicts DBSCAN cluster labels and generates counterfactuals with respect to the surrogate’s probability output. This is a standard workaround of treating cluster labels as \emph{pseudo-classes}~\cite{laugel2019dangers, karra2025generating}. While acknowledging that the resulting decision regions may reflect the surrogate’s inductive bias rather than DBSCAN’s underlying density-reachability structure, inclusion of surrogates has proven to be necessary due to the lack of methods able to adequately compute valid DBSCAN counterfactuals. For surrogate-based variants, \mbox{DBSCAN} cluster labels are used as training targets; except for BayCon, which natively fits a Random Forest surrogate. We select the surrogate via grid search over feed-forward neural networks, Random Forests, and XGBoost classifiers~\cite{chen2016xgboost}, choosing the best-performing model per dataset.

    \paragraph{Sampling Strategy.}
        For each method, we use the following strategy to select the samples that will be explained. For each cluster $C_i$ or noise partition $N$, we sample $10$ points randomly. Each sampled point $p\in C_i$, targets every other cluster $C_j$ with $j\neq i$ yielding $10\!\times\!\sum_{i=1}^m (m-1)+10\!\times\!m$ queries per dataset. Each query asks to produce a set of $k=10$ CEs in the specified target cluster. The only method that cannot output a precise number of counterfactuals is BayCon, which inherently produces a large set of explanations. Thus, we evaluate BayCon by taking both $10$ random counterfactuals and the $10$ closest to the original instance. We refer to the two cases with \emph{BayCon Random} and \emph{BayCon Closest} respectively.

    %\textbf{Noise$\to$Cluster sampling:} Sample $10$ points randomly from $N$. For each sampled point $p\in N$, create a target for each cluster $C_j$. The sampling strategy 
    %All methods are evaluated on the same query sets. As \textsc{BayCon} does not offer a parameter to specify the number of counterfactuals, we evaluate its first $10$ CEs. CEs are evaluated using:
    \paragraph{Counterfactual Evaluation Metrics.} 
        To evaluate the quality of counterfactuals across methods we utilise three common counterfactual evaluation metrics: \emph{Validity}, \emph{Proximity} and \emph{Diversity}.
    
    \begin{enumerate}%[leftmargin=12pt,label=\arabic*)]
        \item \textbf{Validity ($\uparrow$ is better):} Proportion of returned $k$ counterfactuals that reach the target cluster (i.e., the CE is part of the target cluster). When no CE is found, validity is zero for that query. Higher values are better for validity, as the goal is to obtain as many valid CEs as possible. 
        
        \item \textbf{Average Proximity ($\downarrow$ is better):} Mean distance from original point $p$ to each valid returned CE, measured in DBSCAN's clustering feature space. The distance is the same as the one used to fit DBSCAN, here Euclidean distance. Lower values indicate that the CE more closely resembles the explained point.
        
        \item \textbf{Average Diversity ($\uparrow$ is better):} We want $k$ CEs to offer diverse solutions in the target cluster, not minor variants of the same information. To properly evaluate diversity, we use the diversity metric based on the determinantal point process used in ~\citet{mothilal2020dice} and argued for by ~\citet{kulesza2012determinantal}. Specifically, we compute a kernel matrix $K_{ij} = 1/(1+d(\mathit{cf}_i, \mathit{cf}_j))$ and use its determinant as a measure of diversity.  Higher determinant values indicate greater diversity among proposed counterfactuals. The more diverse the vectors inspected, the higher the volume computed by the determinant to estimate diversity.
    \end{enumerate}

        \begin{figure*}[tbp]
            \centering
            \includegraphics[width=.98\linewidth]{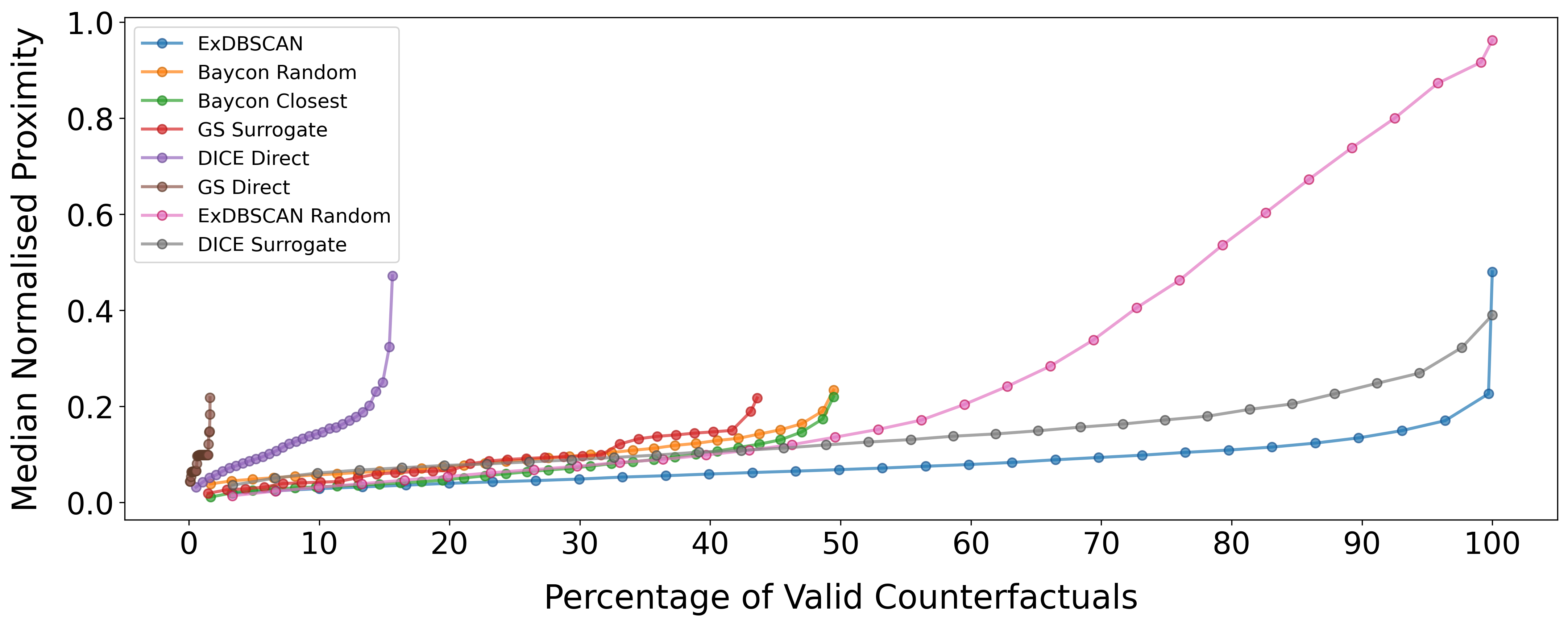}
            \caption{
            Validity/proximity. Cumulative percentage of successful queries (x-axis) vs. proximity (y-axis, $\downarrow$ \textit{lower is better}). ExDBSCAN achieves perfect validity with superior proximity. ExDBSCAN Random and GS-Direct show worse proximity. Surrogate-based methods and BayCon show poor validity ($< 50\%$) due to lack of alignment with DBSCAN model.}
            \Description{Empirical percentile curve showing the percentage of valid counterfactuals in relation to median normalised proximity.}
            \label{fig:proximity_validity_actionable}
        \end{figure*}
        
    %\subsection{Empirical Evaluation}
        We report counterfactual quality using empirical percentile curves, which visualise the distribution of each metric across queries while simultaneously reflecting method failure rates. For each method, per-query values are sorted and plotted with the metric value on the $y$-axis and the cumulative percentage of all queries on the $x$-axis, where the $x$-position of the $x$-axis is given by $100 \cdot r / N$ with $N$ denoting the total number of queries. Queries for which a method fails to return a valid counterfactual, yield \textit{undefined} metric values and do not contribute points to the curve. Consequently, curves that extend further to the right indicate higher percentage of validity i.e., more successful queries. Due to space, experiments for non-actionable scenarios can be found in Appendix~\ref{app:non_actionable_results}.

    \subsection{Validity and Proximity Evaluation}
    
        Fig.~\ref{fig:proximity_validity_actionable} reports results for \textit{validity} and \textit{proximity}. The $y$-axis indicates the \emph{proximity} metric values, while the $x$-axis reports cumulative percentages of queries for which a valid counterfactual has been obtained. Lower proximity values correspond to less changes from the original instance and thus better counterfactual quality. 

        \paragraph{Validity Performance.}
            ExDBSCAN achieves perfect validity ($100\%$) across all queries, successfully generating valid counterfactuals in every case. This is guaranteed by our theoretical analysis (Theorem~\ref{theo:1}), which ensures that every counterfactual obtained by Eq.~\ref{eq:cfgen} lies within $\varepsilon$ of a core point in the target cluster and thus satisfies DBSCAN's density-connectivity requirements. \emph{ExDBSCAN Random} and \emph{DiCE surrogate} also achieve $100\%$ validity. The full table of validity values are presented in the Appendix in Table~\ref{table:tabular_validity_results}. %For \emph{ExDBSCAN Random}, the method uses uniform sampling form valid core points, and for \emph{GS-Direct}, it's geometric expansion strategy eventually finds points within cluster boundaries. 
            In contrast, the surrogate-based methods, \emph{DiCE-Surrogate} and \emph{GS-Surrogate}, as well as \textit{BayCon} struggle severely with validity, mostly below $50\%$ successful counterfactual generation. This performance gap originates from the fundamental mismatch of these methods optimising against continuous decision boundaries learned by surrogate classifiers or probabilistic models, which do not faithfully capture DBSCAN's discrete, density-connectivity-based membership requirement. The surrogate's smooth decision regions may unintentionally exclude valid density-connected points or include invalid ones, which leads to generated counterfactuals that fail on DBSCAN's actual assignment function. \emph{DiCE-Direct} performs poorly due to the problem not being characterized by gradients or continuous probability functions. Its optimisation procedure struggles to navigate DBSCAN's discrete assignment space, meaning that it fails to find any valid transition to counterfactuals.
        
        \paragraph{Proximity Performance.} 
            Among methods achieving high validity, ExDBSCAN demonstrates superior proximity. ExDBSCAN's curve lies consistently below both ExDBSCAN Random and DiCE surrogate, indicating that ExDBSCAN produces counterfactuals closer to the original instances. This advantage arises directly from ExDBSCAN's physics-inspired optimisation approach where the spring-like attraction term in Eq.~\ref{eq:cfgen} explicitly biases the selection toward core points with a minimal weighted-path distance to the explained point $p$. By construction, ExDBSCAN's algorithm initialises with the nearest core points (Prop.~\ref{prop:knn}) and incrementally selects additional cores that minimise total system energy, balancing proximity against diversity constraints. 
       \begin{figure*}[tbp]
            \centering
            \includegraphics[width=.98\linewidth]{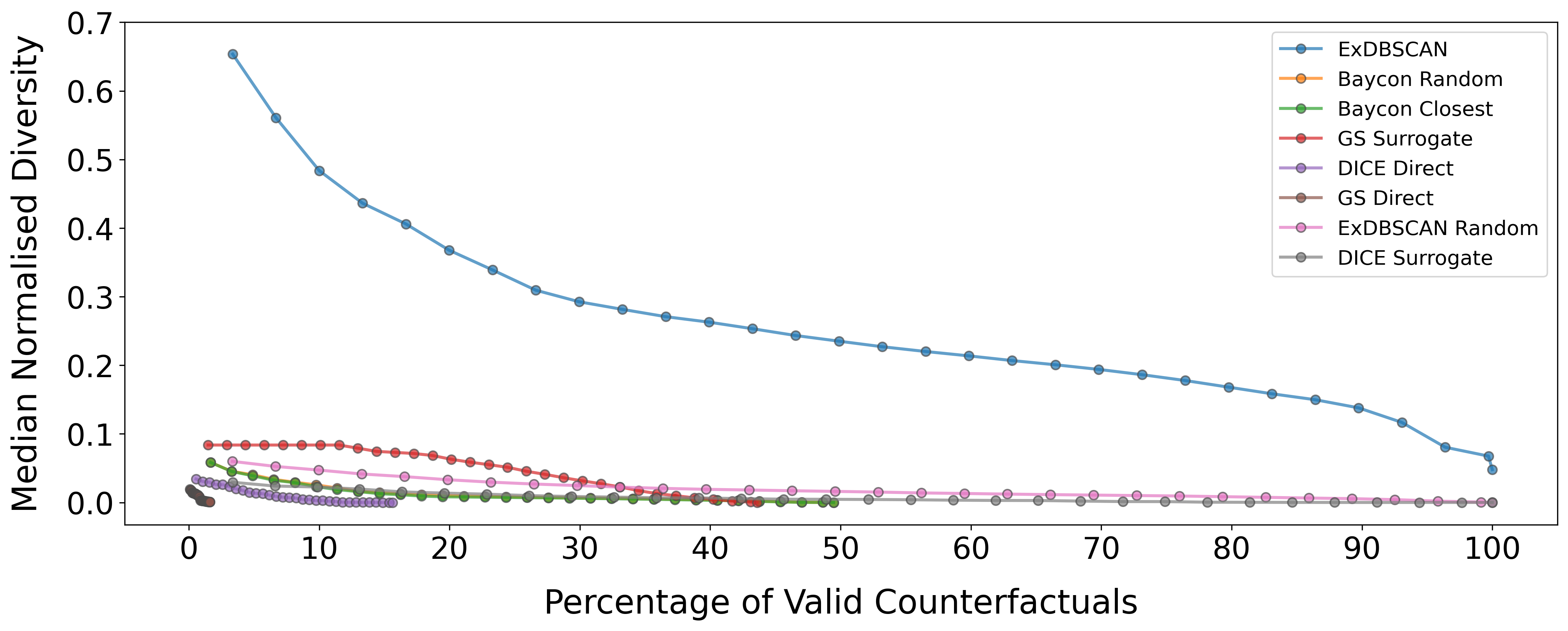}
            \caption{Diversity. Cumulative percentage of successful queries (x-axis) vs. diversity (y-axis, higher is better). ExDBSCAN achieves highest diversity while maintaining perfect validity, balancing proximity and diversity in its physics-inspired approach.}
            \Description{Empirical percentile curve showing the percentage of valid counterfactuals in relation to median normalised diversity.}
            \label{fig:diversity_actionable}
        \end{figure*}
            \emph{ExDBSCAN Random}, while achieving perfect validity, notably shows worse proximity because it samples core points uniformly at random from the target cluster without considering their distance to point $p$. This demonstrates that the structured core-point selection in ExDBSCAN is in fact essential for proximity optimisation.
            \emph{DiCE surrogate} achieves reasonable proximity. However, its optimisation approach does not take into consideration DBSCAN's density-connectivity structure, and thereby lacks a mechanism to select the most proximal valid points once multiple candidates are found. 
            Other \emph{surrogate-based} methods like \emph{BayCon}, when they do succeed, often produce counterfactuals with poor proximity values. This occurs because these methods must explore broadly to find any valid points at all, given their poor alignment with \emph{DBSCAN'}s true decision boundaries. Their successful counterfactuals tend to be those that happened to lie near large, easily-discoverable dense regions rather than the minimal changes that \textit{ExDBSCAN} identifies through its density-aware optimisation. The per-dataset tables showing raw proximity values are presented in the Appendix in Table~\ref{table:proximity_tabular_results}.
            
    \subsection{Diversity Evaluation}
        Fig.~\ref{fig:diversity_actionable} shows the corresponding results for the evaluation metric \emph{diversity}. As before, the $x$-axis represents the cumulative percentage of queries with successful counterfactual generation, while the $y$-axis reports the diversity score of the returned counterfactual set. Higher diversity values indicate more diverse counterfactual solutions, which is the desired outcome. Consequently, methods whose curves lie higher achieve greater diversity among the generated counterfactuals at the same level of validity, while curves that extend further to the right indicate successful counterfactual generation for a larger fraction of queries. 

        \paragraph{Diversity Performance.} ExDBSCAN achieves the highest diversity among all methods, with its curve lying consistently above all baselines. This superior diversity is because of its electrostatic-like repulsion mechanism (Eq.~\ref{eq:energy}) which explicitly enforces separation between selected core points. The repulsion force scales with the inverse of the weighted shortest-path distance $D(V_i, V_j)$ in the cluster's density-connectivity graph $G$, ensuring that selected counterfactuals are spread throughout the cluster structure rather than clustered together. Critically, this diversity measure respects DBSCAN's notion of similarity. If points appear close in Euclidean space, they may be far apart in terms of density-connectivity (also illustrated in App.~\ref{app:euclideanLimitation}), and ExDBSCAN's graph-based distance ensures that selected counterfactuals represent actually distinct alternatives. 
        \textit{ExDBSCAN Random} despite having perfect validity, does not match ExDBSCAN performances on diversity. Its random sampling from core points produces spread-out selections when the cluster contains many spread out cores. However, this comes at the cost of significantly worse proximity (as seen in Fig.~\ref{fig:proximity_validity_actionable}), highlighting the proximity-diversity trade-off that ExDBSCAN explicitly optimises for.% The random baseline's competitive diversity performance validates that the repulsion-based diversity modelling design choice prevents selection of nearby core points and is therefore an effective diversity strategy; and ExDBSCAN systematically achieves this while simultaneously optimising proximity. 
        \emph{GS-Direct} shows moderate diversity, considerably lower than both ExDBSCAN and \textit{ExDBSCAN Random}. \textit{GS-Direct}'s expanding sphere strategy tends to find counterfactuals along similar search directions, particularly when the original point is near a cluster boundary. Without an explicit diversity mechanism, \textit{GS-Direct}'s multiple counterfactuals often represent minor variations along the same transition path instead of truly alternative routes into the target cluster.
        The poor diversity performance of the surrogate-based methods and \textit{BayCon}, when they at all find valid counterfactuals, reflects that their fundamental approach is ill-aligned with DBSCAN. These methods optimise against surrogate decision boundaries that do not capture DBSCAN's density-connectivity structure. Even when they incorporate diversity mechanisms (e.g., \textit{DiCE}'s determinantal point process), the diversity is measured in Euclidean space or with respect to the surrogate's feature space, but not DBSCAN's density-connected distance that is the model to be explained here. Consequently, their counterfactuals may appear diverse under their internal metrics while being related or simply invalid from DBSCAN's perspective. The full per-dataset results are presented in in the Appendix in Table~\ref{table:diversity_tabular_results}.

        \paragraph{Balancing Proximity and Diversity.} The results from Fig.~\ref{fig:proximity_validity_actionable} and ~\ref{fig:diversity_actionable} demonstrate ExDBSCAN's success in optimising the inherent trade-off between proximity and diversity. By modelling this balance as a physical energy minimisation problem (Eq.~\ref{eq:optimisationProblem}), ExDBSCAN achieves superior performance on both metrics simultaneously. ExDBSCAN generates counterfactuals that are both close to the original instance (enabling realistic, minimal changes) and diverse from one another (providing multiple distinct alternatives). These empirical results validate that our physics-inspired approach for ExDBSCAN is effective for achieving multi-objective optimisation in this counterfactual generation problem. Thus, our counterfactuals truly support ``what-if'' scenarios even in the absence of differentiable loss functions or class memberships required for existing methods.

%% file: Sections/conclusions.tex
\section{Conclusion and Future Work}\label{sec:conclusion}
    We propose ExDBSCAN, a novel counterfactual method specifically designed for density-based clustering. ExDBSCAN addresses the interpretability gap in DBSCAN by generating noise-to-cluster and cluster-to-cluster counterfactual explanations with theoretical validity guarantee. Our novel physics-inspired optimisation approach models the proximity-diversity trade-off as an energy minimisation problem, taking into account the structure of the cluster.
    
   In extensive experiments on $30$ tabular datasets, ExDBSCAN demonstrates superior performance across all evaluation metrics, generating more proximal counterfactuals than all baselines, and producing more diverse counterfactual sets, while achieving perfect validity.

 In future work, ExDBSCAN could be extended to incremental or streaming scenarios with evolving clusters; particularly relevant to applications in continuous monitoring systems, fraud detection or real-time anomaly detection, where data arrive incrementally and cluster assignments may change over time. 

    Investigation into more sophisticated optimisation strategies beyond our greedy solution to the NP-hard energy minimisation problem might further improve counterfactual quality.

%We proposed ExDBSCAN, a novel counterfactual reasoning framework for density-based clustering. In extensive experiments, we showcase that ExDBSCAN  generates diverse, valid, and proximal counterfactuals. With our physics-inspired diversity model we balance diversity and proximity when generating sets of counterfactuals. ExDBSCAN enhances the application of DBSCAN to ensure interpretability and thus user trust.
% consider something?
%  \subsection{Limitations}\label{subsec:limitation}
% In some cases, generating counterfactuals towards clusters may not be the optimal solution. It may be that closer to the outlier more outliers are found and generating a counterfactual towards noise and thereby introducing a new virtual cluster may be the better option. Our current approach may identify this by investigating if the number of noise points may be sufficient enough including the `virtually corrected' noise point to form its own cluster. However, we currently do not investigate this due to the sequential approach. In future works, simultaneously generating counterfactuals specifically for noise and identifying possible `newly' formed clusters is an interesting direction.

%% file: Sections/appendix.tex
\section{Appendix}

\subsection{LLM statement}

No generative AI tools were used in the writing of the paper and at any stage of this research.

{\color{revision} \subsection{Algorithm Pseudocode}
To complement Section~\ref{subsec:cfexp} we include full pseudocode for the following two components of ExDBSCAN:

\begin{enumerate}
    \item Selection of $k$ reference core points via greedy minimisation of energy function in Eq.~\ref{eq:optimisationProblem}.

    \item Generation of counterfactual examples from the selected core points using Eq.~\ref{eq:cfgen}.
\end{enumerate}

These algorithms make explicit the operations described in the main text and show how proximity and diversity are jointly optimised.

\subsubsection{Greedy Selection of Core Points}

Given a point $p$ to be explained and a target cluster, the first step of
\textsc{ExDBSCAN} is to select a set of $k$ core points whose
$\varepsilon$-neighbourhoods will serve as the basis for constructing
counterfactuals.  
The energy of a candidate set $C'$ is:

\[
E_{C'} = 
\sum_{V_i\in C'}\sum_{V_j\in C': j>i} 
    \frac{1}{\mathcal{D}(V_i, V_j)} 
\;+\;
\sum_{V_i\in C'} d^2(p, V_i).
\]

The optimisation problem in Eq.~\ref{eq:optimisationProblem} is NP-hard
(Proposition~\ref{prop:npHardness}), so we use a greedy approximation.
For $k=1$, this procedure returns the optimal solution.
\begin{algorithm}[H]
\caption{Greedy Selection of Core Points}
\label{alg:core_selection}
\begin{algorithmic}[1]
\Require Point to explain $p$; set of core points $V$; number of counterfactuals $k$; weighted graph distance $\mathcal{D}(\cdot,\cdot)$; feature-space distance $d(\cdot,\cdot)$; target cluster $t$.
\Ensure Set $C'$ of $k$ selected core points
\State $C' \gets \emptyset$
\State $V^t = \{V_j : V_j \in t\}$
\While{$|C'| < k$}
    \State $\text{best\_core} \gets \text{None}$, \; $\text{best\_energy} \gets +\infty$
    \For{each core point $V_j \in V \setminus C'$}
        \State $S \gets C' \cup \{V_j\}$
        \State $E_S \gets 0$
        \Comment{Repulsion term (graph distances)}
        \ForAll{pairs $(V_i, V_\ell)$ in $S$ with $\ell > i$}
            \State $E_S \gets E_S + \dfrac{1}{\mathcal{D}(V_i, V_\ell)}$
        \EndFor
        \Comment{Attraction term (feature-space distance)}
        \ForAll{$V_i \in S$}
            \State $E_S \gets E_S + d^2(p, V_i)$
        \EndFor
        \If{$E_S < \text{best\_energy}$}
            \State $\text{best\_energy} \gets E_S$
            \State $\text{best\_core} \gets V_j$
        \EndIf
    \EndFor
    \State $C' \gets C' \cup \{\text{best\_core}\}$
\EndWhile
\State \Return $C'$
\end{algorithmic}
\end{algorithm}

%All shortest-path distances $\mathcal{D}(V_i, V_j)$ are precomputed once for the target cluster. Evaluating each candidate core point requires $O(|C'|)$, yielding a total complexity of $O(km)$ where $m = |V|$.

\subsubsection{Constructing Counterfactuals}
Once the set of reference core points $C' = \{q_1, \dots , q_k\}$ has been selected, each counterfactual $p'$ is placed inside the $\varepsilon$-neighbourhood of its corresponding core point $q$. To maximise proximity to the original point $p$, we move away from $p$ towards $q$ and stop at distance $\varepsilon$ from $q$. This corresponds exactly to Eq.~\ref{eq:cfgen} in the main text. Because each constructed point lies within $\varepsilon$ of a core point in the target cluster, Theorem~\ref{theo:1} guarantees that all counterfactuals are valid. Algorithm~\ref{alg:generate_cf} summarises the procedure through pseudocode.

\begin{algorithm}[H]
\caption{Generate Counterfactuals from Selected Core Points}
\label{alg:generate_cf}
\begin{algorithmic}[1]
\Require Point $p$; selected core points $C'=\{q_1,\dots,q_k\}$; DBSCAN radius $\varepsilon$; feature-space metric $d(\cdot,\cdot)$
\Ensure Set $C$ of $k$ counterfactuals
\State $C \gets \emptyset$
\ForAll{$q_j \in C'$}
    \State $d_{pq} \gets d(p, q_j)$
    \Comment{By construction, $d_{pq} > 0$ since $p$ and $q_j$ belong to different clusters}
    \State $\text{direction} \gets (q_j - p) / d_{pq}$
    \State $p' \gets p + (d_{pq} - \varepsilon)\cdot \text{direction}$ \Comment{Eq.~\ref{eq:cfgen}}
    \State $C \gets C \cup \{p'\}$
\EndFor
\State \Return $C$
\end{algorithmic}
\end{algorithm}

Because each $p'$ lies within distance $\varepsilon$ of a core point in the target cluster, the construction ensures all counterfactuals are valid.

}

\color{revision}\subsubsection{Target Cluster Not Specified}
Pseudocode~\ref{alg:core_selection} requires the target cluster in input. The user could prefer not to specify a certain target cluster, e.g.\ if the user wants to convert a noise point to a generic non-noise instance. In this case, pseudocode~\ref{alg:core_selection} would need to be slightly changed on line $8$. Specifically, instead of adding a repulsion term for each pair of core points already selected unconditionally, we add it only if they belong to the same cluster. Points belonging to different cluster are indeed already diverse by construction and they are not connected through any density connected path according to DBSCAN's algorithm.

\subsubsection{Complexity Analysis}
Given a DBSCAN clustering as input (with core points and weighted graph distances), ExDBSCAN has the following computational complexity. Core point selection requires $O(k^3* |V^t|)$, construction of CEs requires $O(kn)$, and  precomputing distances requires $O(n|V^t| + |V^t||E|)$, where $n$ denotes the number of features, $V^t$ the set of core points in target cluster $t$, $E$ the edges in connectivity graph, and $k$ the number of CEs. 

Validity is guaranteed by Theorem \ref{theo:1} and requires no additional verification step. Regarding scalability: high-dimensionality increases $n$, while large-scale datasets and the neighbourhood radius $\varepsilon$ do not directly influence scalability, but implicitly in $|V^t|\gg N$ and E, both of which scale linearly. Only $k$ enters cubically, however it is a small constant in practice, e.g., \citep{mothilal2020dice} uses 3-5. Overall, the runtime remains tractable in high-dimensional and large scale settings.

\subsection{Approximation Quality of the Greedy CE Selection}
The energy minimization problem (Eq. \eqref{eq:optimisationProblem}) is NP-hard and solved via a greedy strategy. To assess approximation quality, we additionally evaluate a local search variant (best-improvement swap initialized from greedy solution) and an exact solver (Gurobi MIQP on Eq. \eqref{eq:optimisationProblem}), across 10 datasets, 100 query points and $k\in\{4,10\}$.

For $k=4$, the greedy solution matches the exact optimum in 87\% cases (mean ratio 1.0014, worst-case 1.05, 98\% within 1\%). Local search initialized from the greedy solution converges to the exact optimum in 100\% cases, requiring a mean of only 0.2 swaps, which confirms that greedy is near-locally-optimal. 
For k=10, the gap shrinks, greedy achieves 90\% exact matches (mean ratio 1.0005, worst-case 1.02) and local search reaches 99\% with 0.1 mean swaps. Across both values of $k$, all cases fall within 5\% of optimum.

To assess whether objective differences translate to differences in the selected sets, we report mean Jaccard overlap between the greedy and exact solutions, for $k=4$ 0.94 and 0.97 for $k=10$. Even in cases where objective values differ, the selected CE sets are nearly identical. We attribute this to the sensitivity of the repulsion term to close pairs: once greedy avoids selecting nearby points, further optimization yields only marginal improvements.

\subsection{Normalisation Scheme}\label{subsec:normalisation}
The distance $d(p, V_i)$ from candidate core points $V_i$ to the explained point $p$ and the distances between candidate core points $\mathcal{D}(V_i, V_j)$, could have different scales depending on the the dataset, with $d(p, V_i)$ potentially being bigger as it connects points that are not in a cluster. In order to prevent the spring-like term to dominate, given the distance $d_c$ from $p$ to the closest core point, and the average graph weight $\bar{\mathcal{D}}$, we scale distances $\mathcal{D}(V_i, V_j)$ by $d_c/\bar{\mathcal{D}}$.

\subsection{Choice of Distance Metrics}\label{subsec:metricchoice}
    By design, ExDBSCAN is agnostic to metrics. That being said, if DBSCAN is fit on a metric different to euclidean, e.g., Manhattan, Cosine, Mahalanobis, or mixed-type metrics, ExDBSCAN can be easily adjusted to this. Instead of using euclidean distance, we can employ any other distance metric to determine reference core-points as well as solving the energy-minimisation to produce diverse and proximal counterfactuals. We also highly advise, to then also use the same metric in ExDBSCAN that was selected for DBSCAN clustering to maintain consistency in the concepts of density-reachability and graph connectivity. Furthermore, we advise employing this chosen metric to evaluate proximity and diversity, as the underlying structure of the data and clustering depends on it.

    For our evaluation, we use Euclidean distance to assess proximity, as it is the distance metric employed to fit DBSCAN. To measure diversity, we apply a Determinantal Point Process (DPP) kernel based on the weighted shortest-path distance within the cluster graph. For points that are density-connected, like our counterfactuals, graph distance effectively captures the similarity structure of DBSCAN. In contrast, Euclidean distance may not accurately represent density connectivity. Furthermore, between the explained point and its counterfactuals, there is no density connectivity, making Euclidean distance, the metric used for fitting DBSCAN, the most appropriate measure. 

    % In the upcoming subsubsection below, we show an example of utilising other distance metrics by showing results of ExDBSCAN for Manhattan distance.

    % \subsubsection{Manhattan Distance Results}\label{app:manhattan_experiments}
    %     xxxx
\subsection{Experimental Results Extended}\label{app:prox_val_tabels}
    In the following section we present the complete numerical results in tables for actionable experiments extending Section~\ref{subsec:nafResults}.
    %and results without non-actionable features (Sect.~\ref{subsec:afResults}).
    
    \paragraph{Validity}
        To extend the experimental Section~\ref{sec:experiments}, we show the individual dataset results for each method in Table~\ref{table:tabular_validity_results}. While the results in the table isolatedly look like ExDBSCAN is equivalently as good as ExDBSCAN Random or DiCE Surrogate, one must remember that ExDBSCAN performs better in both proximity and diversity.

    %% Validity
    \begin{table*}[tbp]
        \centering
        \caption{Validity scores across $30$ datasets. Values represent the proportion of queries (0.0-1.0) for which each method successfully generated valid counterfactuals that satisfy DBSCAN's density-connectivity criteria. Higher values indicate better performance. Error margins represent the standard error of the mean. ExDBSCAN, ExDBSCAN Random, and DiCE Surrogate achieve perfect validity (1.00) across all datasets. Model-agnostic baselines (BayCon, GS-Surrogate) show poor validity (<0.50 on average) due to misalignment between their optimisation objectives and DBSCAN's discrete assignment function. DiCE-Direct and GS-Direct perform poorly because they cannot effectively navigate DBSCAN's discrete landscape without continuous probabilities or gradients. Bold indicates best performance; while ExDBSCAN matches some baselines on validity, it substantially outperforms on proximity and diversity (Tables~\ref{table:proximity_tabular_results}, \ref{table:diversity_tabular_results}). The best-performing method is highlighted in bold while the second-best is underlined.}
        \label{table:tabular_validity_results}
        \footnotesize
        \begin{tabular}{lcccccccc}
        \toprule
        Dataset & ExDBSCAN & BayCon Random & BayCon Closest & GS Surrogate & DiCE Direct & GS Direct & ExDBSCAN Random & DiCE Surrogate \\
        \midrule
        autoPrice & \textbf{1.00} & 0.46 & 0.46 & 0.14 & 0.00 & 0.00 & \textbf{1.00} & \textbf{1.00} \\
        baskball & \textbf{1.00} & 0.78 & 0.78 & 0.65 & 0.50 & 0.00 & \textbf{1.00} & \textbf{1.00} \\
        blood-transfusion & \textbf{1.00} & 0.55 & 0.55 & 0.23 & 0.50 & 0.00 & \textbf{1.00} & \textbf{1.00} \\
        bodyfat & \textbf{1.00} & 0.36 & 0.36 & 0.15 & 0.49 & 0.00 & \textbf{1.00} & \textbf{1.00} \\
        breast-w & \textbf{1.00} & 0.93 & 0.93 & 0.45 & 0.50 & 0.00 & \textbf{1.00} & \textbf{1.00} \\
        chscase census2 & \textbf{1.00} & 0.58 & 0.58 & 0.00 & 0.47 & 0.00 & \textbf{1.00} & \textbf{1.00} \\
        chscase census6 & \textbf{1.00} & 0.00 & 0.00 & 0.00 & 0.50 & 0.00 & \textbf{1.00} & \textbf{1.00} \\
        chscase vine1 & \textbf{1.00} & 0.62 & 0.62 & 0.35 & 0.50 & 0.00 & \textbf{1.00} & \textbf{1.00} \\
        confidence & \textbf{1.00} & 0.35 & 0.35 & 0.00 & 0.00 & \textbf{1.00} & \textbf{1.00} & 0.00 \\
        diabetes & \textbf{1.00} & 0.42 & 0.42 & 0.35 & 0.00 & 0.00 & \textbf{1.00} & \textbf{1.00} \\
        diabetes numeric & \textbf{1.00} & 0.72 & 0.72 & 0.61 & 0.44 & 0.00 & \textbf{1.00} & \textbf{1.00} \\
        diggle table a1 & \textbf{1.00} & 0.18 & 0.18 & 0.68 & 0.54 & 0.00 & \textbf{1.00} & N/A \\
        disclosure x noise & \textbf{1.00} & 0.70 & 0.70 & 0.00 & 0.50 & 0.00 & \textbf{1.00} & \textbf{1.00} \\
        ecoli & \textbf{1.00} & 0.17 & 0.17 & 0.00 & 0.48 & 0.00 & \textbf{1.00} & \textbf{1.00} \\
        glass & \textbf{1.00} & 0.57 & 0.57 & 0.17 & 0.50 & 0.00 & \textbf{1.00} & \textbf{1.00} \\
        hayes-roth & \textbf{1.00} & 0.81 & 0.81 & 0.73 & 0.00 & 0.00 & \textbf{1.00} & \textbf{1.00} \\
        heart-statlog & \textbf{1.00} & 0.93 & 0.93 & 0.62 & 0.50 & 0.00 & \textbf{1.00} & \textbf{1.00} \\
        iris & \textbf{1.00} & 0.80 & 0.80 & 0.60 & 0.00 & \textbf{1.00} & \textbf{1.00} & \textbf{1.00} \\
        liver-disorders & \textbf{1.00} & 0.90 & 0.90 & 0.00 & 0.00 & 0.00 & \textbf{1.00} & 0.00 \\
        longley & \textbf{1.00} & 0.16 & 0.16 & 0.08 & 0.52 & 0.00 & \textbf{1.00} & \textbf{1.00} \\
        machine cpu & \textbf{1.00} & 0.62 & 0.62 & 0.47 & 0.50 & 0.00 & \textbf{1.00} & \textbf{1.00} \\
        mu284 & \textbf{1.00} & 0.48 & 0.48 & 0.23 & 0.00 & 0.00 & \textbf{1.00} & \textbf{1.00} \\
        no2 & \textbf{1.00} & 0.59 & 0.59 & 0.62 & 0.49 & 0.00 & \textbf{1.00} & \textbf{1.00} \\
        pm10 & \textbf{1.00} & 0.61 & 0.61 & 0.68 & 0.00 & 0.00 & \textbf{1.00} & 0.00 \\
        prnn fglass & \textbf{1.00} & 0.68 & 0.68 & 0.15 & 0.50 & 0.00 & \textbf{1.00} & 0.00 \\
        rabe 131 & \textbf{1.00} & 0.85 & 0.85 & 0.75 & 0.50 & 0.00 & \textbf{1.00} & \textbf{1.00} \\
        sleep & \textbf{1.00} & 0.19 & 0.19 & 0.04 & 0.52 & 0.00 & \textbf{1.00} & \textbf{1.00} \\
        strikes & \textbf{1.00} & 0.42 & 0.42 & 0.49 & 0.00 & 0.00 & \textbf{1.00} & 0.00 \\
        vehicle & \textbf{1.00} & 0.20 & 0.20 & 0.07 & 0.50 & 0.00 & \textbf{1.00} & 0.00 \\
        wine & \textbf{1.00} & 0.55 & 0.55 & 0.40 & 0.50 & 0.00 & \textbf{1.00} & \textbf{1.00} \\
        \bottomrule
        \end{tabular}
    \end{table*}
    
    \paragraph{Proximity}\label{subsec:proximity}
        %%% Proximity
        Per dataset proximity results for each method are shown in Table~\ref{table:proximity_tabular_results}. To properly interpret the results presented, Table~\ref{table:proximity_tabular_results} should be read jointly with Table~\ref{table:tabular_validity_results}. Indeed, the methods that are often the closest in terms of proximity w.r.t. to ExDBSCAN, have in general a low validity as displayed in Table~\ref{table:tabular_validity_results}. Low validity methods find counterfactuals only for ''easy'' instances, whose counterfactuals can be found relatively close-by. Therefore, taking into consideration also Figure~\ref{fig:proximity_validity_actionable}, low validity methods do not find more proximal counterfactual, but on the other end find counterfactuals only in the case in which they are relatively close to the initial instance.
        \begin{table*}[tbp]
            \centering
            \caption{Values represent the average Euclidean distance from original instances to their valid counterfactuals. Lower values indicate better proximity (more realistic, minimal changes). Error margins represent the standard error of the mean. ExDBSCAN achieves best or near-best proximity on the majority of datasets (bold indicates best, underline indicates second-best), demonstrating that our spring-like attraction mechanism effectively identifies the most proximal valid counterfactuals. BayCon Closest performs competitively by selecting the closest from its many generated counterfactuals, but this comes at the cost of poor diversity and high computational overhead. ExDBSCAN Random shows notably worse proximity, validating that our structured core-point selection via energy minimisation is essential. Surrogate-based methods show high proximity values when successful, indicating they find valid points through broad exploration rather than targeted minimal changes. N/A indicates datasets where the method found no valid counterfactuals.}
            \label{table:proximity_tabular_results}
            \footnotesize
            \begin{tabular}{lcccccccc}
            \toprule
            Dataset & ExDBSCAN & BayCon Random & BayCon Closest & GS Surrogate & DiCE Direct & GS Direct & ExDBSCAN Random & DiCE Surrogate \\
            \midrule
            autoPrice & 3.1 $\pm$ 0.2 & 3.4 $\pm$ 0.2 & \underline{2.7 $\pm$ 0.1} & \textbf{2.4 $\pm$ 0.4} & N/A & N/A & 3.3 $\pm$ 0.2 & 6.6 $\pm$ 0.2 \\
            baskball & 1.34 $\pm$ 0.08 & 1.92 $\pm$ 0.09 & \underline{1.29 $\pm$ 0.09} & 1.7 $\pm$ 0.1 & 2.4 $\pm$ 0.1 & N/A & \textbf{1.0 $\pm$ 0.1} & 2.33 $\pm$ 0.08 \\
            blood-transfusion-service-center & \underline{1.6 $\pm$ 0.1} & 1.8 $\pm$ 0.2 & \textbf{1.2 $\pm$ 0.2} & 1.8 $\pm$ 0.4 & 1.8 $\pm$ 0.3 & N/A & 2.1 $\pm$ 0.2 & 2.3 $\pm$ 0.2 \\
            bodyfat & 2.5 $\pm$ 0.4 & 1.7 $\pm$ 0.2 & \underline{0.9 $\pm$ 0.2} & \textbf{0.8 $\pm$ 0.2} & 3.3 $\pm$ 0.1 & N/A & 2.8 $\pm$ 0.2 & 4.8 $\pm$ 0.4 \\
            breast-w & \textbf{1.1 $\pm$ 0.1} & 3.0 $\pm$ 0.2 & 2.3 $\pm$ 0.2 & 1.9 $\pm$ 0.2 & 3.7 $\pm$ 0.2 & N/A & \underline{1.6 $\pm$ 0.2} & 3.4 $\pm$ 0.1 \\
            chscase census2 & \underline{1.4 $\pm$ 0.1} & 1.6 $\pm$ 0.1 & \textbf{1.1 $\pm$ 0.1} & N/A & 2.4 $\pm$ 0.2 & N/A & 1.9 $\pm$ 0.2 & 2.4 $\pm$ 0.1 \\
            chscase census6 & \textbf{1.1 $\pm$ 0.1} & N/A & N/A & N/A & 1.4 $\pm$ 0.2 & N/A & 2.0 $\pm$ 0.2 & \underline{1.4 $\pm$ 0.1} \\
            chscase vine1 & \underline{1.7 $\pm$ 0.2} & 3.3 $\pm$ 0.3 & 2.7 $\pm$ 0.3 & \textbf{1.7 $\pm$ 0.2} & 4.5 $\pm$ 0.2 & N/A & 1.9 $\pm$ 0.2 & 4.3 $\pm$ 0.2 \\
            confidence & \underline{2.49 $\pm$ 0.05} & 3.09 $\pm$ 0.04 & 2.76 $\pm$ 0.07 & N/A & N/A & 2.8 $\pm$ 0.1 & \textbf{0.59 $\pm$ 0.09} & N/A \\
            diabetes & \underline{2.9 $\pm$ 0.2} & 3.3 $\pm$ 0.2 & \textbf{1.8 $\pm$ 0.2} & 3.3 $\pm$ 0.4 & N/A & N/A & 3.7 $\pm$ 0.2 & 5.0 $\pm$ 0.2 \\
            diabetes numeric & 1.30 $\pm$ 0.10 & 1.5 $\pm$ 0.1 & \textbf{1.0 $\pm$ 0.1} & \underline{1.3 $\pm$ 0.1} & 1.7 $\pm$ 0.1 & N/A & 1.3 $\pm$ 0.2 & 1.70 $\pm$ 0.10 \\
            diggle table a1 & 1.7 $\pm$ 0.1 & 1.9 $\pm$ 0.3 & \underline{1.5 $\pm$ 0.3} & 2.1 $\pm$ 0.2 & 2.6 $\pm$ 0.2 & N/A & \textbf{1.3 $\pm$ 0.2} & N/A \\
            disclosure x noise & 1.03 $\pm$ 0.10 & 1.4 $\pm$ 0.1 & \textbf{0.8 $\pm$ 0.1} & N/A & \underline{0.84 $\pm$ 0.06} & N/A & 1.25 $\pm$ 0.09 & 1.3 $\pm$ 0.1 \\
            ecoli & \underline{4.2 $\pm$ 0.1} & 5.90 $\pm$ 0.08 & 5.7099 $\pm$ 0.0003 & N/A & 7.4 $\pm$ 1.3 & N/A & \textbf{3.8 $\pm$ 0.9} & 7.6 $\pm$ 0.8 \\
            glass & \underline{2.8 $\pm$ 0.3} & 3.9 $\pm$ 0.4 & 3.3 $\pm$ 0.5 & \textbf{2.1 $\pm$ 0.5} & 4.3 $\pm$ 0.4 & N/A & 2.9 $\pm$ 0.3 & 4.9 $\pm$ 0.3 \\
            hayes-roth & 1.40 $\pm$ 0.05 & 1.75 $\pm$ 0.06 & \textbf{1.02 $\pm$ 0.06} & \underline{1.29 $\pm$ 0.08} & N/A & N/A & 2.0 $\pm$ 0.1 & 2.48 $\pm$ 0.05 \\
            heart-statlog & \textbf{1.22 $\pm$ 0.10} & 3.4 $\pm$ 0.2 & 2.6 $\pm$ 0.2 & 2.7 $\pm$ 0.2 & 4.2 $\pm$ 0.1 & N/A & \underline{1.6 $\pm$ 0.2} & 4.23 $\pm$ 0.10 \\
            iris & \textbf{1.5 $\pm$ 0.2} & 2.4 $\pm$ 0.1 & 2.1 $\pm$ 0.1 & 2.2 $\pm$ 0.1 & N/A & 2.0 $\pm$ 0.2 & \underline{1.6 $\pm$ 0.2} & 3.2 $\pm$ 0.2 \\
            liver-disorders & \textbf{1.6 $\pm$ 0.4} & 2.5 $\pm$ 0.4 & \underline{1.7 $\pm$ 0.5} & N/A & N/A & N/A & 2.0 $\pm$ 0.3 & N/A \\
            longley & 2.1 $\pm$ 0.2 & 1.7 $\pm$ 0.2 & \underline{1.4 $\pm$ 0.2} & \textbf{0.9 $\pm$ 0.2} & 3.4 $\pm$ 0.3 & N/A & 1.8 $\pm$ 0.3 & 3.1 $\pm$ 0.2 \\
            machine cpu & 2.8 $\pm$ 0.3 & 2.2 $\pm$ 0.2 & \textbf{1.5 $\pm$ 0.2} & \underline{1.6 $\pm$ 0.2} & 2.9 $\pm$ 0.5 & N/A & 2.8 $\pm$ 0.3 & 3.4 $\pm$ 0.3 \\
            mu284 & 1.7 $\pm$ 0.1 & 1.67 $\pm$ 0.09 & \underline{1.2 $\pm$ 0.1} & \textbf{1.2 $\pm$ 0.1} & N/A & N/A & 3.4 $\pm$ 0.4 & 3.5 $\pm$ 0.4 \\
            no2 & \underline{1.5 $\pm$ 0.2} & 1.9 $\pm$ 0.2 & \textbf{1.0 $\pm$ 0.2} & 1.7 $\pm$ 0.3 & 2.5 $\pm$ 0.2 & N/A & 1.9 $\pm$ 0.2 & 3.2 $\pm$ 0.2 \\
            pm10 & \textbf{1.52 $\pm$ 0.09} & 2.4 $\pm$ 0.1 & \underline{1.6 $\pm$ 0.1} & 2.6 $\pm$ 0.2 & N/A & N/A & 2.7 $\pm$ 0.2 & N/A \\
            prnn fglass & \underline{3.1 $\pm$ 0.4} & 4.5 $\pm$ 0.5 & 4.0 $\pm$ 0.5 & \textbf{2.1 $\pm$ 0.6} & 4.7 $\pm$ 0.5 & N/A & 3.1 $\pm$ 0.4 & N/A \\
            rabe 131 & \underline{1.18 $\pm$ 0.09} & 2.0 $\pm$ 0.1 & 1.4 $\pm$ 0.1 & 1.6 $\pm$ 0.2 & 2.6 $\pm$ 0.1 & N/A & \textbf{1.1 $\pm$ 0.1} & 2.63 $\pm$ 0.10 \\
            sleep & 2.3 $\pm$ 0.2 & 1.9 $\pm$ 0.6 & \underline{1.8 $\pm$ 0.6} & N/A & 3.1 $\pm$ 0.3 & N/A & 2.0 $\pm$ 0.2 & 2.8 $\pm$ 0.2 \\
            strikes & \underline{2.08 $\pm$ 0.03} & 2.38 $\pm$ 0.04 & \textbf{1.84 $\pm$ 0.04} & 2.30 $\pm$ 0.04 & N/A & N/A & 11.8 $\pm$ 0.2 & N/A \\
            vehicle & 4.9 $\pm$ 0.7 & 3.7 $\pm$ 1.9 & \underline{3.1 $\pm$ 2.0} & \textbf{1.2 $\pm$ 1.0} & 5.5 $\pm$ 0.9 & N/A & 5.5 $\pm$ 0.5 & N/A \\
            wine & \textbf{1.7 $\pm$ 0.2} & 2.6 $\pm$ 0.3 & 1.9 $\pm$ 0.3 & 1.8 $\pm$ 0.3 & 4.5 $\pm$ 0.2 & N/A & \underline{1.7 $\pm$ 0.2} & 4.8 $\pm$ 0.1 \\
            \bottomrule
            \end{tabular}
        \end{table*}

        %Table~\ref{tab:allresults} shows how \textsc{ExDBSCAN} consistently outperforms \textsc{BayCon} in terms of proximity. Leveraging DBSCAN's structure, \textsc{ExDBSCAN} is able to find the best counterfactual across all datasets. The more proximal the counterfactual, the more the counterfactual probes the local decision boundary of the model, increasing the probability the the explanation reflects the real reason why the model made its decision.

    \paragraph{Diversity}\label{subsec:diversity}
        Per dataset diversity results for each method are shown in Table~\ref{table:diversity_tabular_results}. As with proximity, these results should be interpreted jointly with Table~\ref{table:tabular_validity_results} to fully understand each method's performance. 
        
        ExDBSCAN consistently achieves substantially higher diversity scores than all baseline methods across nearly all datasets. This superior performance directly results from our physics-inspired energy minimisation approach (Equation~\ref{eq:energy}), where the electrostatic-like repulsion term actively spreads selected core points throughout the target cluster's density-connectivity structure.

        The results in Table~\ref{table:diversity_tabular_results} demonstrate ExDBSCAN's unique ability to simultaneously optimise both proximity and diversity through our physics-inspired energy minimisation framework, a balance that no baseline method achieves.
    
    % In contrast, BayCon methods show extremely poor diversity (often $<0.1$) despite generating many counterfactual candidates. This occurs because BayCon's optimisation does not account for DBSCAN's density-connected distance structure—it optimises in Euclidean space without considering the cluster's graph topology. Consequently, BayCon tends to generate many similar counterfactuals clustered in the same region, failing to explore the full range of valid alternatives.
    
    % GS-Surrogate occasionally achieves competitive diversity when it successfully finds valid counterfactuals. However, as shown in Table~\ref{table:tabular_validity_results}, its low validity ($<50\%$ on average) severely limits its practical utility. When GS-Surrogate does find valid counterfactuals, the broad exploration inherent to the growing spheres approach can lead to diverse solutions, but this comes at the cost of unreliable validity and often poor proximity.
    
    % ExDBSCAN Random performs reasonably well on some datasets, validating that our repulsion-based approach contributes significantly to diversity. However, without the attraction term to ensure proximity to the original point $p$, ExDBSCAN Random sacrifices proximity for diversity. This demonstrates that both components of our energy function are essential: the repulsion term ensures diversity while the spring-like attraction term maintains proximity.

        %% Diversity
        \begin{table*}[tbp]
            \centering
            \caption{Table shows all diversity results per dataset. Higher values indicate greater diversity (more distinct, non-redundant alternatives). Error margins represent the standard error of the mean. ExDBSCAN achieves substantially higher diversity than all baselines across nearly all datasets; bold indicates best, underline indicates second-best. BayCon shows extremely poor diversity despite generating many counterfactuals, because its optimisation does not account for DBSCAN's density-connected distance. GS-Surrogate occasionally shows competitive diversity when successful, but its low validity limits practical utility. The results demonstrate ExDBSCAN's unique ability to balance both proximity and diversity through physics-inspired energy minimisation. The best-performing method is highlighted in bold while the second-best is underlined.}
            \label{table:diversity_tabular_results}
            \footnotesize
            \begin{tabular}{lcccccccc}
            \toprule
            Dataset & ExDBSCAN & BayCon Random & BayCon Closest & GS Surrogate & DiCE Direct & GS Direct & ExDBSCAN Random & DiCE Surrogate \\
            \midrule
            autoPrice & \textbf{7.2 $\pm$ 0.2} & 0.58 $\pm$ 0.01 & 0.57 $\pm$ 0.01 & \underline{0.92 $\pm$ 0.03} & N/A & N/A & 0.727 $\pm$ 0.008 & 0.13 $\pm$ 0.01 \\
            baskball & \textbf{1.63 $\pm$ 0.05} & 0.031 $\pm$ 0.002 & 0.033 $\pm$ 0.004 & \underline{0.24 $\pm$ 0.06} & 0.003 $\pm$ 0.000 & N/A & 0.045 $\pm$ 0.003 & 0.0087 $\pm$ 0.0010 \\
            blood-transfusion & \textbf{1.6 $\pm$ 0.1} & 0.0057 $\pm$ 0.0009 & 0.0036 $\pm$ 0.0006 & \underline{0.5 $\pm$ 0.2} & 0.00038 $\pm$ 0.00008 & N/A & 0.027 $\pm$ 0.006 & 0.00024 $\pm$ 0.00004 \\
            bodyfat & \textbf{4.8 $\pm$ 0.2} & 0.29 $\pm$ 0.01 & 0.31 $\pm$ 0.02 & \underline{0.66 $\pm$ 0.09} & 0.26 $\pm$ 0.02 & N/A & 0.44 $\pm$ 0.02 & 0.220 $\pm$ 0.010 \\
            breast-w & \textbf{5.0 $\pm$ 0.2} & 0.30 $\pm$ 0.01 & 0.27 $\pm$ 0.02 & \underline{0.65 $\pm$ 0.08} & 0.13 $\pm$ 0.01 & N/A & 0.42 $\pm$ 0.02 & 0.128 $\pm$ 0.009 \\
            chscase census2 & \textbf{2.29 $\pm$ 0.07} & 0.09 $\pm$ 0.01 & 0.091 $\pm$ 0.009 & N/A & 0.015 $\pm$ 0.000 & N/A & \underline{0.147 $\pm$ 0.009} & 0.028 $\pm$ 0.002 \\
            chscase census6 & \textbf{0.84 $\pm$ 0.02} & 0.060 $\pm$ 0.000 & 0.084 $\pm$ 0.000 & N/A & 0.158 $\pm$ 0.000 & N/A & \underline{0.41 $\pm$ 0.01} & 0.151 $\pm$ 0.001 \\
            chscase vine1 & \textbf{5.1 $\pm$ 0.2} & 0.42 $\pm$ 0.03 & 0.39 $\pm$ 0.03 & \underline{0.81 $\pm$ 0.06} & 0.088 $\pm$ 0.003 & N/A & 0.45 $\pm$ 0.02 & 0.127 $\pm$ 0.007 \\
            confidence & \textbf{1.81 $\pm$ 0.05} & \underline{0.09 $\pm$ 0.03} & 0.08 $\pm$ 0.02 & N/A & N/A & 0.022 $\pm$ 0.002 & 0.051 $\pm$ 0.002 & N/A \\
            diabetes & \textbf{5.3 $\pm$ 0.1} & 0.30 $\pm$ 0.01 & 0.30 $\pm$ 0.01 & \underline{0.79 $\pm$ 0.05} & N/A & N/A & 0.53 $\pm$ 0.01 & 0.062 $\pm$ 0.004 \\
            diabetes numeric & \textbf{1.4 $\pm$ 0.1} & 0.0006 $\pm$ 0.0001 & 0.0023 $\pm$ 0.0006 & \underline{0.05 $\pm$ 0.03} & 0.00006 $\pm$ 0.00002 & N/A & 0.035 $\pm$ 0.004 & 0.0026 $\pm$ 0.0004 \\
            diggle table a1 & \textbf{1.60 $\pm$ 0.06} & 0.41 $\pm$ 0.09 & 0.42 $\pm$ 0.09 & 0.40 $\pm$ 0.07 & 0.310 $\pm$ 0.000 & N/A & \underline{0.49 $\pm$ 0.03} & N/A \\
            disclosure x noise & \textbf{1.16 $\pm$ 0.08} & 0.0031 $\pm$ 0.0005 & 0.0006 $\pm$ 0.0002 & N/A & 0.00026 $\pm$ 0.00008 & N/A & \underline{0.011 $\pm$ 0.002} & 0.00014 $\pm$ 0.00005 \\
            ecoli & \textbf{6.9 $\pm$ 0.1} & 0.03 $\pm$ 0.02 & 0.02 $\pm$ 0.02 & N/A & 0.032 $\pm$ 0.006 & N/A & \underline{0.630 $\pm$ 0.009} & 0.058 $\pm$ 0.006 \\
            glass & \textbf{5.4 $\pm$ 0.2} & 0.39 $\pm$ 0.02 & 0.39 $\pm$ 0.02 & \underline{0.8 $\pm$ 0.1} & 0.10 $\pm$ 0.01 & N/A & 0.44 $\pm$ 0.02 & 0.056 $\pm$ 0.009 \\
            hayes-roth & \textbf{3.9 $\pm$ 0.1} & 0.084 $\pm$ 0.005 & 0.069 $\pm$ 0.005 & \underline{0.29 $\pm$ 0.04} & N/A & N/A & 0.221 $\pm$ 0.010 & 0.058 $\pm$ 0.002 \\
            heart-statlog & \textbf{7.7 $\pm$ 0.4} & 0.51 $\pm$ 0.02 & 0.50 $\pm$ 0.02 & \underline{0.94 $\pm$ 0.03} & 0.32 $\pm$ 0.03 & N/A & 0.61 $\pm$ 0.02 & 0.29 $\pm$ 0.02 \\
            iris & \textbf{2.90 $\pm$ 0.07} & 0.19 $\pm$ 0.05 & 0.19 $\pm$ 0.05 & \underline{0.36 $\pm$ 0.10} & N/A & 0.159 $\pm$ 0.009 & 0.207 $\pm$ 0.007 & 0.0014 $\pm$ 0.0004 \\
            liver-disorders & \textbf{2.3 $\pm$ 0.2} & 0.034 $\pm$ 0.006 & 0.022 $\pm$ 0.004 & N/A & N/A & N/A & \underline{0.10 $\pm$ 0.02} & N/A \\
            longley & \textbf{2.1 $\pm$ 0.1} & 0.039 $\pm$ 0.004 & 0.037 $\pm$ 0.005 & \underline{0.70 $\pm$ 0.03} & 0.352 $\pm$ 0.000 & N/A & 0.63 $\pm$ 0.03 & 0.307 $\pm$ 0.009 \\
            machine cpu & \textbf{1.6 $\pm$ 0.1} & 0.015 $\pm$ 0.001 & 0.015 $\pm$ 0.002 & \underline{0.28 $\pm$ 0.07} & 0.0016 $\pm$ 0.0005 & N/A & 0.038 $\pm$ 0.007 & 0.0009 $\pm$ 0.0003 \\
            mu284 & \textbf{2.24 $\pm$ 0.09} & 0.077 $\pm$ 0.004 & 0.068 $\pm$ 0.003 & \underline{0.79 $\pm$ 0.07} & N/A & N/A & 0.26 $\pm$ 0.01 & 0.056 $\pm$ 0.005 \\
            no2 & \textbf{4.1 $\pm$ 0.3} & 0.18 $\pm$ 0.02 & 0.15 $\pm$ 0.02 & \underline{0.65 $\pm$ 0.07} & 0.14 $\pm$ 0.02 & N/A & 0.28 $\pm$ 0.02 & 0.08 $\pm$ 0.01 \\
            pm10 & \textbf{4.0 $\pm$ 0.1} & 0.153 $\pm$ 0.007 & 0.136 $\pm$ 0.006 & \underline{0.69 $\pm$ 0.04} & N/A & N/A & 0.30 $\pm$ 0.01 & N/A \\
            prnn fglass & \textbf{5.3 $\pm$ 0.2} & 0.44 $\pm$ 0.03 & 0.42 $\pm$ 0.03 & \underline{0.8 $\pm$ 0.1} & 0.08 $\pm$ 0.01 & N/A & 0.43 $\pm$ 0.02 & N/A \\
            rabe 131 & \textbf{2.50 $\pm$ 0.08} & 0.131 $\pm$ 0.008 & 0.12 $\pm$ 0.01 & \underline{0.51 $\pm$ 0.06} & 0.034 $\pm$ 0.000 & N/A & 0.148 $\pm$ 0.008 & 0.0279 $\pm$ 0.0010 \\
            sleep & \textbf{1.03 $\pm$ 0.04} & 0.19 $\pm$ 0.06 & 0.19 $\pm$ 0.06 & N/A & 0.050 $\pm$ 0.000 & N/A & 0.41 $\pm$ 0.02 & 0.071 $\pm$ 0.004 \\
            strikes & \textbf{2.76 $\pm$ 0.02} & 0.152 $\pm$ 0.006 & 0.139 $\pm$ 0.006 & \underline{0.65 $\pm$ 0.01} & N/A & N/A & 0.161 $\pm$ 0.002 & N/A \\
            vehicle & \textbf{4.7 $\pm$ 0.3} & 0.36 $\pm$ 0.04 & 0.31 $\pm$ 0.04 & \underline{0.6 $\pm$ 0.2} & 0.18 $\pm$ 0.02 & N/A & 0.38 $\pm$ 0.02 & N/A \\
            wine & \textbf{7.9 $\pm$ 0.4} & 0.53 $\pm$ 0.01 & 0.51 $\pm$ 0.01 & \underline{0.75 $\pm$ 0.05} & 0.33 $\pm$ 0.02 & N/A & 0.60 $\pm$ 0.02 & 0.30 $\pm$ 0.01 \\
            \bottomrule
            \end{tabular}
        \end{table*}

\subsection{Run-Time Analysis}\label{app:run_time_analysis}
        To show ExDBSCAN's run-time performance over the baselines we run all methods and measure the average time in seconds per-dataset needed to find counterfactuals.

        We present the run-time analysis in a distribution plot in Figure~\ref{fig:runTime}, showing the run-time over each method given their ability to find valid counterfactuals. If a counterfactual is not found, it is not included in the plot and the proportion of valid counterfactuals is thereby adjusted. 

        From Figure~\ref{fig:runTime}, it is evident how DiCE Surrogate and DiCE Direct are the two heaviest methods computationally. More detailed insights on run-time are deducible from Table~\ref{table:runTime}, in which average per-dataset run-time are shown, including run-times for instances in which the method fails to retrieve a counterfactual.

        Specifically, Table~\ref{table:runTime} shows how ExDBSCAN Random is consistently the highest performing method in terms of run-time. Its simple heuristic allow the method to produce counterfactuals without a proper optimisation, speeding-up counterfactual search. While ExDBSCAN and the Growing Spheres algorithms are often the second-fastest methods, DiCE and BayCon consistently rank at the bottom.
        \begin{figure}[ht]
            \centering
            \includegraphics[width=0.9\linewidth]{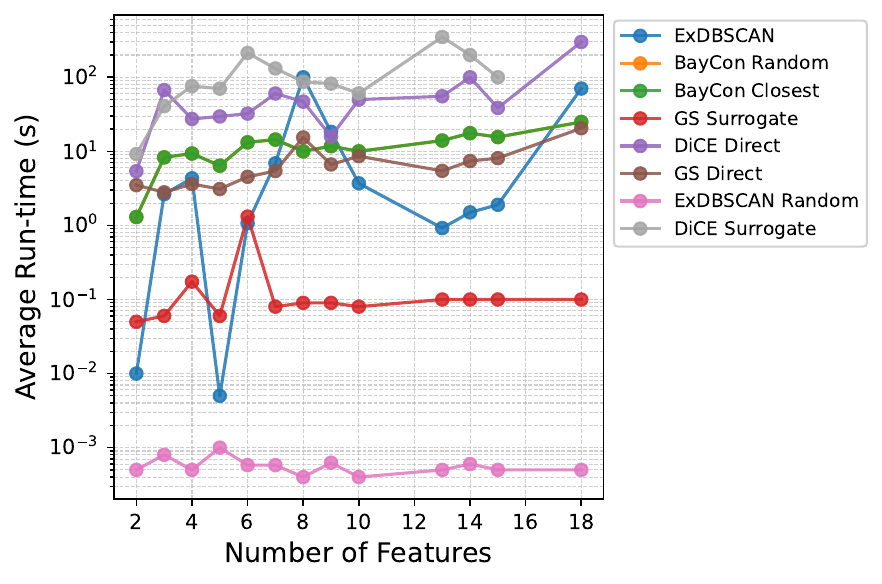}
            \caption{Lineplot showing the relationship between runtime in seconds (y-axis, log-scaled) and number of features (x-axis).}
            \Description{Lineplot showing the relationship between runtime in seconds (y-axis, log-scaled) and number of features (x-axis).}
            \label{fig:runtime_dimensions}
        \end{figure}
        \begin{figure*}[tbp]
            \centering
            \includegraphics[width=1\linewidth]{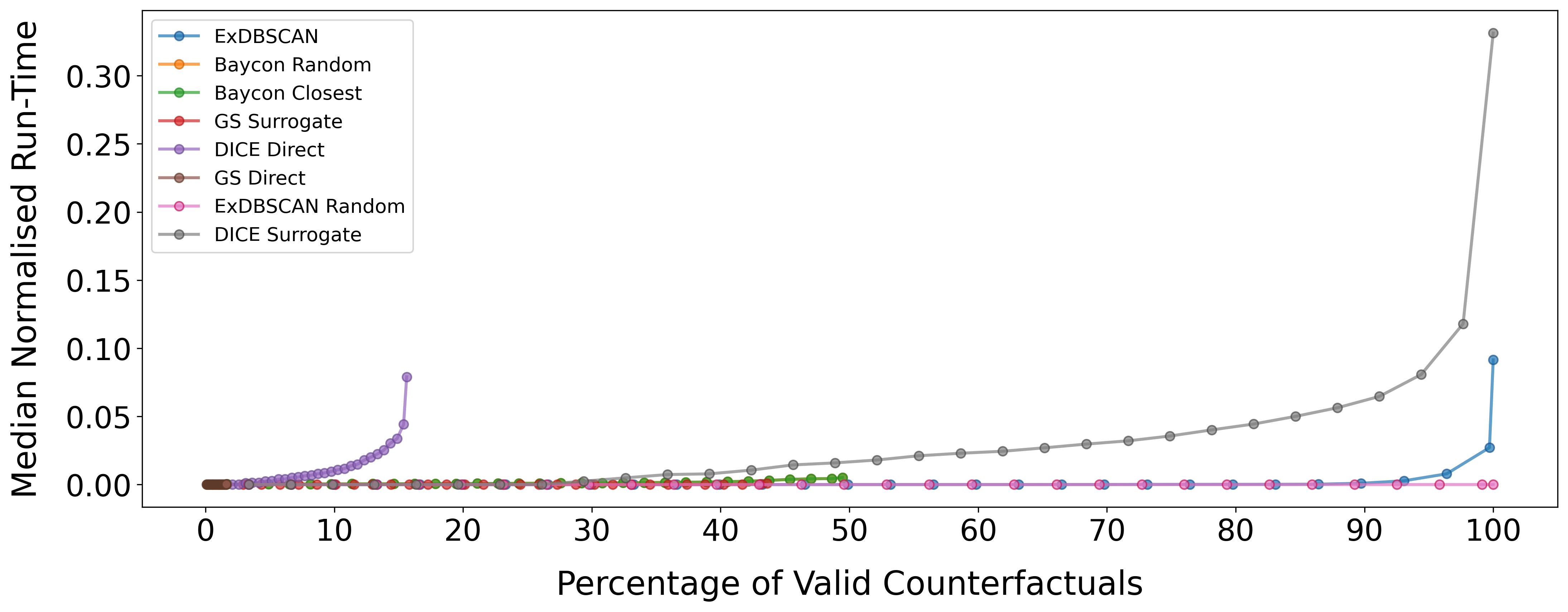}
            \caption{Run-time evaluation. Cumulative curves showing percentage of successful queries (x-axis) vs. median run-time in seconds (y-axis, lower is better). The two DiCE's variation are the heaviest methods computationally.}
            \Description{Empirical percentile curve showing the percentage of valid counterfactuals in relation to median normalised run time.}
            \label{fig:runTime}
        \end{figure*}

\begin{table*}[tbp]
    \centering
    \caption{The tables shows the average run-time in seconds per each-method across datasets. Lower values indicate faster run-time. Error margins represent the standard error of the mean. ExDBSCAN Random achieves the best run-time almost everywhere, as its simple heuristic does not involve an optimisation search. The Growing spheres algorithm and ExDBSCAN are the next fastest methods by average run-time. BayCon and DiCE are instead the methods consistently having slower run-times. N/A indicates queries where the method found no valid counterfactuals. The best-performing method is highlighted in bold while the second-best is underlined}
    \label{table:runTime}
    \scriptsize
    \begin{tabular}{lcccccccc}
    \toprule
    Dataset & ExDBSCAN & BayCon Random & BayCon Closest & GS Surrogate & DiCE Direct & GS Direct & ExDBSCAN Random & DiCE Surrogate \\
    \midrule
    autoPrice & 1.9 $\pm$ 0.4 & 15.6 $\pm$ 1.4 & 15.6 $\pm$ 1.4 & $\underline{1 \cdot 10^{-1} \pm 2 \cdot 10^{-3}}$ & 38.6 $\pm$ 5.4 & 8.1 $\pm$ 0.7 & $\boldsymbol{5 \cdot 10^{-4} \pm 3 \cdot 10^{-5}}$ & $1 \cdot 10^{2} \pm 1 \cdot 10^{1}$ \\
    baskball & $\underline{3 \cdot 10^{-3} \pm 9 \cdot 10^{-5}}$ & 8.3 $\pm$ 0.9 & 8.3 $\pm$ 0.9 & $6 \cdot 10^{-2} \pm 2 \cdot 10^{-3}$ & 35.9 $\pm$ 3.4 & 3.0 $\pm$ 0.4 & $\boldsymbol{7 \cdot 10^{-4} \pm 6 \cdot 10^{-6}}$ & $1 \cdot 10^{2} \pm 1 \cdot 10^{1}$ \\
    blood-transfusion & 19.5 $\pm$ 3.2 & 9.6 $\pm$ 1.2 & 9.6 $\pm$ 1.2 & \underline{0.63 $\pm$ 0.10} & 39.3 $\pm$ 4.0 & 8.31 $\pm$ 1.01 & $\boldsymbol{5 \cdot 10^{-4} \pm 2 \cdot 10^{-5}}$ & 74.5 $\pm$ 17.8 \\
    bodyfat & 1.5 $\pm$ 0.3 & 17.6 $\pm$ 1.7 & 17.6 $\pm$ 1.7 & $\underline{1 \cdot 10^{-1} \pm 6 \cdot 10^{-3}}$ & $1 \cdot 10^{2} \pm 7 \cdot 10^{0}$ & 7.4 $\pm$ 0.7 & $\boldsymbol{6 \cdot 10^{-4} \pm 3 \cdot 10^{-5}}$ & $2 \cdot 10^{2} \pm 2 \cdot 10^{1}$ \\
    breast-w & 52.0 $\pm$ 8.1 & 8.8 $\pm$ 1.0 & 8.8 $\pm$ 1.0 & $9 \cdot 10^{-2} \pm 2 \cdot 10^{-3}$ & $\underline{7 \cdot 10^{-2} \pm 2 \cdot 10^{-4}}$ & 6.5 $\pm$ 1.0 & $\boldsymbol{5 \cdot 10^{-4} \pm 2 \cdot 10^{-5}}$ & $1 \cdot 10^{2} \pm 2 \cdot 10^{1}$ \\
    chscase census2 & $\underline{2 \cdot 10^{-3} \pm 1 \cdot 10^{-4}}$ & 11.7 $\pm$ 1.2 & 11.7 $\pm$ 1.2 & N/A & $5 \cdot 10^{1} \pm 4 \cdot 10^{0}$ & 3.7 $\pm$ 0.5 & $\boldsymbol{8 \cdot 10^{-4} \pm 4 \cdot 10^{-5}}$ & $1 \cdot 10^{2} \pm 2 \cdot 10^{1}$ \\
    chscase census6 & $\boldsymbol{4 \cdot 10^{-4} \pm 4 \cdot 10^{-6}}$ & $2 \cdot 10^{1} \pm 3 \cdot 10^{-2}$ & $2 \cdot 10^{1} \pm 3 \cdot 10^{-2}$ & N/A & 32.7 $\pm$ 5.9 & 4.6 $\pm$ 0.4 & $\underline{5 \cdot 10^{-4} \pm 4 \cdot 10^{-5}}$ & 19.4 $\pm$ 4.6 \\
    chscase vine1 & $\underline{3 \cdot 10^{-2} \pm 3 \cdot 10^{-3}}$ & 14.8 $\pm$ 1.3 & 14.8 $\pm$ 1.3 & $9 \cdot 10^{-2} \pm 2 \cdot 10^{-3}$ & 23.1 $\pm$ 2.9 & 3.6 $\pm$ 0.4 & $\boldsymbol{10 \cdot 10^{-4} \pm 5 \cdot 10^{-5}}$ & $1 \cdot 10^{2} \pm 2 \cdot 10^{1}$ \\
    confidence & 0.18 $\pm$ 0.01 & 9.6 $\pm$ 1.4 & 9.6 $\pm$ 1.4 & $\underline{6 \cdot 10^{-2} \pm 6 \cdot 10^{-4}}$ & $1 \cdot 10^{2} \pm 7 \cdot 10^{0}$ & $7 \cdot 10^{-1} \pm 7 \cdot 10^{-3}$ & $\boldsymbol{10 \cdot 10^{-4} \pm 6 \cdot 10^{-5}}$ & N/A \\
    diabetes & $1 \cdot 10^{2} \pm 2 \cdot 10^{1}$ & $1 \cdot 10^{1} \pm 1 \cdot 10^{0}$ & $1 \cdot 10^{1} \pm 1 \cdot 10^{0}$ & $\underline{9 \cdot 10^{-2} \pm 2 \cdot 10^{-3}}$ & 46.9 $\pm$ 2.9 & 15.4 $\pm$ 1.1 & $\boldsymbol{4 \cdot 10^{-4} \pm 7 \cdot 10^{-6}}$ & 87.0 $\pm$ 9.5 \\
    diabetes numeric & $\underline{1 \cdot 10^{-2} \pm 2 \cdot 10^{-3}}$ & 1.3 $\pm$ 0.1 & 1.3 $\pm$ 0.1 & $5 \cdot 10^{-2} \pm 2 \cdot 10^{-3}$ & 5.4 $\pm$ 0.9 & 3.5 $\pm$ 0.5 & $\boldsymbol{5 \cdot 10^{-4} \pm 4 \cdot 10^{-5}}$ & 9.2 $\pm$ 4.2 \\
    diggle table a1 & $\underline{5 \cdot 10^{-4} \pm 4 \cdot 10^{-5}}$ & 15.3 $\pm$ 1.0 & 15.3 $\pm$ 1.0 & $6 \cdot 10^{-2} \pm 2 \cdot 10^{-3}$ & 15.0 $\pm$ 2.8 & 2.9 $\pm$ 0.5 & $\boldsymbol{4 \cdot 10^{-4} \pm 2 \cdot 10^{-5}}$ & N/A \\
    disclosure x noise & 5.1 $\pm$ 0.8 & 7.00 $\pm$ 1.07 & 7.00 $\pm$ 1.07 & N/A & 34.4 $\pm$ 2.5 & \underline{4.9 $\pm$ 0.7} & $\boldsymbol{6 \cdot 10^{-4} \pm 3 \cdot 10^{-5}}$ & 40.7 $\pm$ 11.5 \\
    ecoli & 15.3 $\pm$ 10.4 & $2 \cdot 10^{1} \pm 2 \cdot 10^{0}$ & $2 \cdot 10^{1} \pm 2 \cdot 10^{0}$ & $\underline{9 \cdot 10^{-2} \pm 7 \cdot 10^{-4}}$ & 32.6 $\pm$ 8.4 & 8.4 $\pm$ 1.6 & $\boldsymbol{5 \cdot 10^{-4} \pm 4 \cdot 10^{-5}}$ & $2 \cdot 10^{2} \pm 2 \cdot 10^{1}$ \\
    glass & 11.0 $\pm$ 1.8 & 12.5 $\pm$ 1.4 & 12.5 $\pm$ 1.4 & $\underline{9 \cdot 10^{-2} \pm 2 \cdot 10^{-3}}$ & 22.2 $\pm$ 2.2 & 8.5 $\pm$ 0.8 & $\boldsymbol{5 \cdot 10^{-4} \pm 2 \cdot 10^{-5}}$ & 46.5 $\pm$ 9.1 \\
    hayes-roth & $\underline{3 \cdot 10^{-2} \pm 3 \cdot 10^{-3}}$ & 5.4 $\pm$ 0.6 & 5.4 $\pm$ 0.6 & $6 \cdot 10^{-2} \pm 9 \cdot 10^{-4}$ & $3 \cdot 10^{-2} \pm 2 \cdot 10^{-4}$ & 3.1 $\pm$ 0.4 & $\boldsymbol{4 \cdot 10^{-4} \pm 3 \cdot 10^{-6}}$ & $1 \cdot 10^{2} \pm 6 \cdot 10^{0}$ \\
    heart-statlog & 1.2 $\pm$ 0.2 & 12.1 $\pm$ 1.3 & 12.1 $\pm$ 1.3 & $\underline{1 \cdot 10^{-1} \pm 2 \cdot 10^{-3}}$ & 26.3 $\pm$ 2.7 & 4.8 $\pm$ 0.6 & $\boldsymbol{5 \cdot 10^{-4} \pm 1 \cdot 10^{-5}}$ & $5 \cdot 10^{2} \pm 3 \cdot 10^{1}$ \\
    iris & 2.1 $\pm$ 0.4 & 8.2 $\pm$ 1.3 & 8.2 $\pm$ 1.3 & $\underline{6 \cdot 10^{-2} \pm 2 \cdot 10^{-3}}$ & 46.7 $\pm$ 4.4 & 0.79 $\pm$ 0.02 & $\boldsymbol{5 \cdot 10^{-4} \pm 5 \cdot 10^{-6}}$ & $3 \cdot 10^{1} \pm 1 \cdot 10^{1}$ \\
    liver-disorders & \underline{2.95 $\pm$ 0.04} & 5.1 $\pm$ 1.5 & 5.1 $\pm$ 1.5 & N/A & N/A & N/A & $\boldsymbol{10 \cdot 10^{-4} \pm 9 \cdot 10^{-5}}$ & N/A \\
    longley & $\boldsymbol{4 \cdot 10^{-4} \pm 2 \cdot 10^{-5}}$ & 17.5 $\pm$ 0.9 & 17.5 $\pm$ 0.9 & $7 \cdot 10^{-2} \pm 2 \cdot 10^{-3}$ & 15.0 $\pm$ 3.6 & 3.9 $\pm$ 0.5 & $\underline{4 \cdot 10^{-4} \pm 2 \cdot 10^{-5}}$ & 20.6 $\pm$ 6.5 \\
    machine cpu & \underline{2.2 $\pm$ 0.3} & 9.2 $\pm$ 1.3 & 9.2 $\pm$ 1.3 & 3.8 $\pm$ 0.1 & 39.2 $\pm$ 4.5 & N/A & $\boldsymbol{6 \cdot 10^{-4} \pm 5 \cdot 10^{-5}}$ & $6 \cdot 10^{2} \pm 2 \cdot 10^{2}$ \\
    mu284 & 3.7 $\pm$ 0.7 & $1 \cdot 10^{1} \pm 1 \cdot 10^{0}$ & $1 \cdot 10^{1} \pm 1 \cdot 10^{0}$ & $\underline{8 \cdot 10^{-2} \pm 2 \cdot 10^{-3}}$ & $5 \cdot 10^{1} \pm 6 \cdot 10^{0}$ & 8.6 $\pm$ 0.6 & $\boldsymbol{4 \cdot 10^{-4} \pm 2 \cdot 10^{-5}}$ & 60.6 $\pm$ 9.3 \\
    no2 & 15.2 $\pm$ 2.6 & 10.2 $\pm$ 1.4 & 10.2 $\pm$ 1.4 & $\underline{7 \cdot 10^{-2} \pm 3 \cdot 10^{-3}}$ & $1 \cdot 10^{2} \pm 5 \cdot 10^{0}$ & N/A & $\boldsymbol{5 \cdot 10^{-4} \pm 4 \cdot 10^{-6}}$ & $2 \cdot 10^{2} \pm 2 \cdot 10^{1}$ \\
    pm10 & 4.0 $\pm$ 0.6 & 11.2 $\pm$ 0.8 & 11.2 $\pm$ 0.8 & $\underline{8 \cdot 10^{-2} \pm 2 \cdot 10^{-3}}$ & $1 \cdot 10^{2} \pm 5 \cdot 10^{0}$ & 4.9 $\pm$ 0.4 & $\boldsymbol{5 \cdot 10^{-4} \pm 1 \cdot 10^{-5}}$ & N/A \\
    prnn fglass & 10.3 $\pm$ 1.7 & 10.9 $\pm$ 1.3 & 10.9 $\pm$ 1.3 & $\underline{9 \cdot 10^{-2} \pm 2 \cdot 10^{-3}}$ & 18.1 $\pm$ 1.7 & 8.0 $\pm$ 0.8 & $\boldsymbol{5 \cdot 10^{-4} \pm 3 \cdot 10^{-5}}$ & N/A \\
    rabe 131 & $\underline{5 \cdot 10^{-3} \pm 3 \cdot 10^{-4}}$ & 6.4 $\pm$ 0.9 & 6.4 $\pm$ 0.9 & $6 \cdot 10^{-2} \pm 2 \cdot 10^{-3}$ & 29.6 $\pm$ 2.6 & 3.1 $\pm$ 0.4 & $\boldsymbol{10 \cdot 10^{-4} \pm 4 \cdot 10^{-5}}$ & 70.3 $\pm$ 11.7 \\
    sleep & $\boldsymbol{3 \cdot 10^{-4} \pm 5 \cdot 10^{-6}}$ & 19.0 $\pm$ 0.8 & 19.0 $\pm$ 0.8 & $8 \cdot 10^{-2} \pm 2 \cdot 10^{-3}$ & 19.5 $\pm$ 5.8 & 4.7 $\pm$ 0.4 & $\underline{6 \cdot 10^{-4} \pm 2 \cdot 10^{-5}}$ & 26.6 $\pm$ 6.8 \\
    strikes & $2 \cdot 10^{-1} \pm 8 \cdot 10^{-3}$ & 14.3 $\pm$ 0.2 & 14.3 $\pm$ 0.2 & $\underline{8 \cdot 10^{-2} \pm 4 \cdot 10^{-4}}$ & 42.5 $\pm$ 0.9 & 5.1 $\pm$ 0.1 & $\boldsymbol{4 \cdot 10^{-4} \pm 1 \cdot 10^{-6}}$ & N/A \\
    vehicle & 70.5 $\pm$ 11.5 & 24.9 $\pm$ 1.3 & 24.9 $\pm$ 1.3 & $\underline{1 \cdot 10^{-1} \pm 5 \cdot 10^{-3}}$ & $3 \cdot 10^{2} \pm 2 \cdot 10^{1}$ & 20.5 $\pm$ 2.0 & $\boldsymbol{5 \cdot 10^{-4} \pm 5 \cdot 10^{-6}}$ & N/A \\
    wine & 0.64 $\pm$ 0.08 & 15.9 $\pm$ 1.5 & 15.9 $\pm$ 1.5 & $\underline{1 \cdot 10^{-1} \pm 2 \cdot 10^{-3}}$ & 84.6 $\pm$ 5.3 & 6.1 $\pm$ 0.7 & $\boldsymbol{5 \cdot 10^{-4} \pm 3 \cdot 10^{-6}}$ & $2 \cdot 10^{2} \pm 2 \cdot 10^{1}$ \\
    \bottomrule
    \end{tabular}
    \end{table*}

\subsection{Non-Actionable Features}\label{subsec:nonActionable}
    Actionability ensures that the generated CEs are actionable, i.e., a user can actually realistically change the features. Non-actionable features such as genetic information or inherent personal attributes, including age and gender, make explanations unrealistic and prevent human stake-holders to simulate and apply the proposed changes.

    \subsubsection{Non-Actionable Features in ExDBSCAN}
        To account for non-actionable features, we restrict the set of core points to consider, when identifying the set $C'$. The constrained set of core points includes all core points that are not further than $\varepsilon$ away from the original point when measuring distances in the subspace defined by the non-actionable features $\mathbb{R}^{n'}$, where $n' < n$ denotes the number of non-actionable features. In this way, we take into consideration only core points whose $\varepsilon$ neighbourhood is reachable without changing any of the non-actionable features.
    
        The undirected weighted graph $G$ is then built and the set of reference core points chosen $C'$ as described in Section~\ref{subsec:cfexp}.

        %\begin{equation}
         %   Q_{\text{constrained}} = \{ q \in Q \mid \|p_{dim},q_{dim}\| \leq \varepsilon, \forall \text{non actionable dimensions} \}.
        %\end{equation}
        
        When moving from the set of core points $C'$ to the set of counterfactuals $C$, it is not guaranteed that we can move in the direction of the point to explain $p$, as we now have the constraint given by the non-actionable features. Therefore, we move toward point $p$ only in the subspace $\mathbb{R}^{n-n'}$ formed by the actionable features, still putting the counterfactual $\varepsilon$ away from the reference core point, which is always possible given our core point filtering procedure.
        We visualise the approach in Figure~\ref{fig:actionabilityToy} where we show how a single counterfactual is produced for a toy dataset when feature $Y$ is non-actionable.

        \begin{figure*}[ht]
            \centering
            \includegraphics[width=1\linewidth]{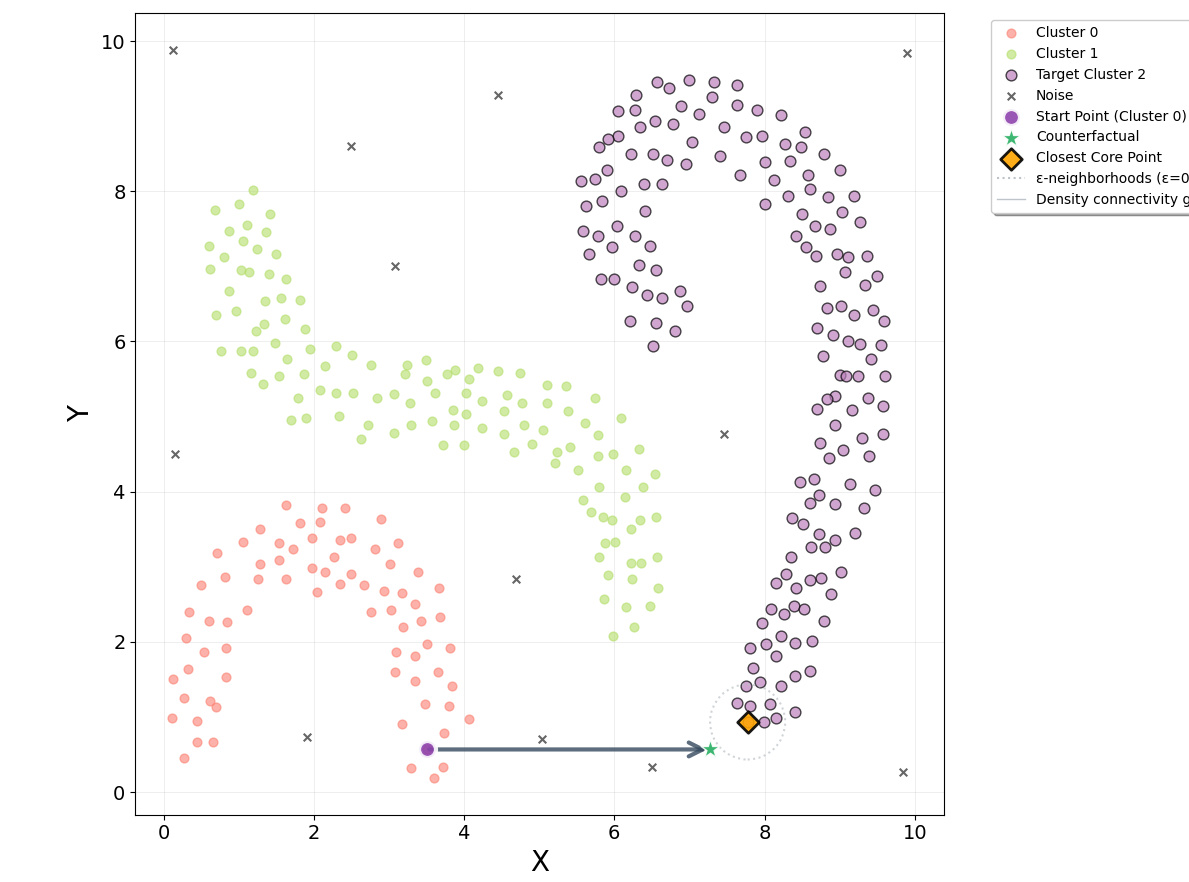}
            \caption{Conceptual Figure illustrating non-actionability. Coloured points denote cluster points assigned to different clusters. The x-axis shows the actionable and y-axis shows non-actionable feature. The orange rhombus shows the closest core point, that is in the non-actionable direction not further than $\varepsilon$ away from the point to explain (purple). The green star is the identified counterfactual.}
            \Description{Conceptual figure with three clusters showing how a counterfactual explanation is computed when some features are non-actionable. The search direction is restricted and reduces the search space for valid counterfactuals.}
            \label{fig:actionabilityToy}
        \end{figure*}

        \subsubsection{Real-world Constraints}
            There are different strategies to account for real-world constraints. First of all, if the non-actionable features include categorical features, we suggest changing the metric employed in DBSCAN as well as ExDBSCAN, for more details to employing different metrics see Section \ref{subsec:metricchoice}.
    
            Other constraints, i.e., monotonicity or feature bounds can be easily incorporated through further filtering of the core point set. Instead of restricting the set of reference core points to actionable dimensions only, we can additionally restrict this to be a specific direction or range depending on the constraints defined by the user. This allows us to also incorporate already known domain knowledge into our model.
        
    \subsubsection{Non-Actionable features}\label{subsec:nafResults}
        We here evaluate \textsc{ExDBSCAN} in the setting where a subset of features are non-actionable. We utilise the experimental pipeline used in Section~\ref{sec:experiments}. How many features are non actionable is a parameter chosen at random between one $1$ and $\frac{n}{2}$ features, with $n$ the total number of features. 

    \subsubsection{Non-Actionable Results}\label{app:non_actionable_results}
    We present here the experimental analysis benchmarking ExDBSCAN performance when a subset of the features is non-actionable, i.e.\ it cannot be modified by the counterfactual method. We use the same experimental setting used in Section~\ref{sec:experiments}. Differently from Section~\ref{sec:experiments}, we randomly select a subset of the features of cardinality between $0$ and half of the number of features, to be non-actionable.

     Empirical percentile curves, visualising the distribution of proximity while simultaneously reflecting method failure rates are presented in Figure~\ref{fig:proximity_validity_non_actionable}; diversity information is instead visualised in Figure~\ref{fig:diversity_non_actionable}. Results mirror the ones presented in Section~\ref{sec:experiments}, with ExDBSCAN outperforming competitors in validity, proximity and diversity. Full per-dataset tables are presented in Tables~\ref{table:tabular_non_actionability_validity_results},~\ref{table:diversity_tabular_non_actionability_results}, and~\ref{table:tabular_proximity_non_actionable_results}.

    \begin{figure*}[tbp]
        \centering
        \includegraphics[width=1\linewidth]{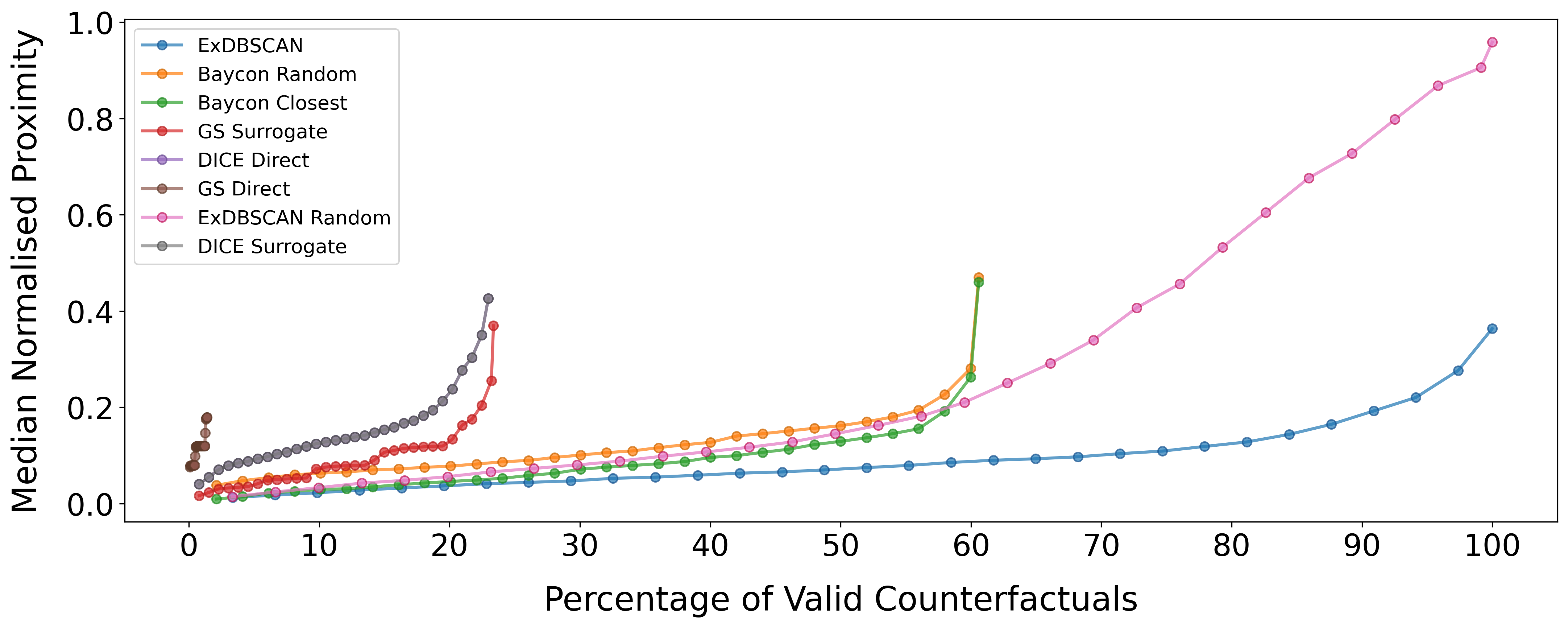}
        \caption{
        Validity/proximity when a subset of the features is non-actionable. Cumulative percentage of successful queries (x-axis) vs. proximity (y-axis, $\downarrow$ \textit{lower is better}). ExDBSCAN achieves perfect validity with superior proximity. ExDBSCAN Random and GS-Direct show worse proximity. Surrogate-based methods and BayCon show poor validity ($< 50\%$) due to lack of alignment with DBSCAN model.}
        \Description{Empirical percentile curve showing the percentage of valid counterfactuals in relation to median normalised proximity when the set of features is restricted due to non-actionable features.}
        \label{fig:proximity_validity_non_actionable}
    \end{figure*}

    \begin{figure*}[tbp]
        \centering
        \includegraphics[width=1\linewidth]{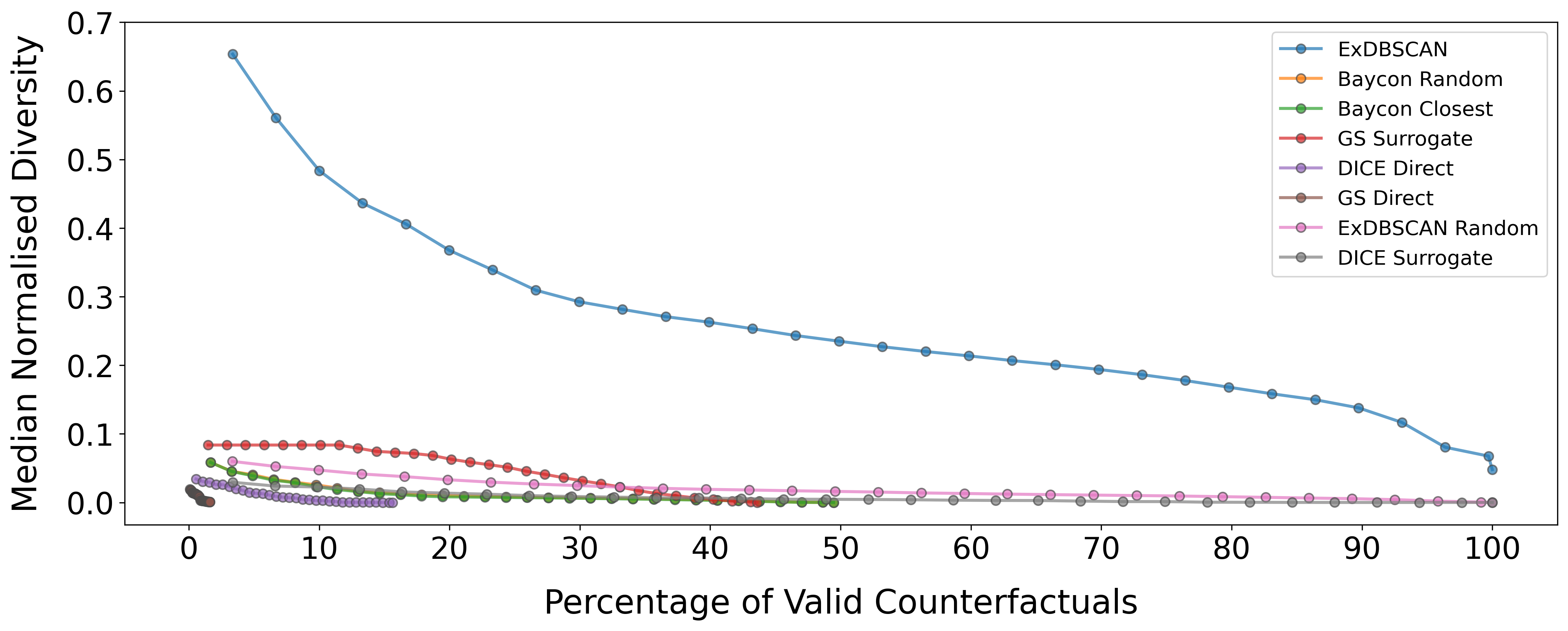}
        \caption{Diversity when a subset of the features is non-actionable. Cumulative percentage of successful queries (x-axis) vs. diversity (y-axis, higher is better). ExDBSCAN achieves highest diversity while maintaining perfect validity, balancing proximity and diversity in its physics-inspired approach.}
        \Description{Empirical percentile curve showing the percentage of valid counterfactuals in relation to median normalised diversity when the set of features is restricted due to non-actionable features.}
        \label{fig:diversity_non_actionable}
    \end{figure*}

\begin{table*}[tbp]
    \centering
    \caption{Proximity scores across all datasets when a subset of the features is non-actionable. Values represent the proportion of queries for which each method successfully generated valid counterfactuals that satisfy DBSCAN's density-connectivity criteria. Higher values indicate better performance. Error margins represent the standard error of the mean. ExDBSCAN, ExDBSCAN Random, and DiCE Surrogate achieve perfect validity across all datasets. Model-agnostic baselines (BayCon, GS-Surrogate) show poor validity due to misalignment between their optimisation objectives and DBSCAN's discrete assignment function. DiCE-Direct and GS-Direct perform poorly because they cannot effectively navigate DBSCAN's discrete landscape without continuous probabilities or gradients. Bold indicates best performance, while the second-best is underlined.}
    \label{table:tabular_proximity_non_actionable_results}
    \footnotesize
        \begin{tabular}{lcccccccc}
        \toprule
        Dataset & ExDBSCAN & BayCon Random & BayCon Closest & GS Surrogate & DiCE Surrogate & GS Direct & DiCE Direct & ExDBSCAN Random \\
        \midrule
        autoPrice & 3.1 $\pm$ 0.2 & 3.4 $\pm$ 0.2 & \textbf{2.7 $\pm$ 0.1} & \underline{2.7 $\pm$ 0.5} & 5.0 $\pm$ 0.3 & N/A & N/A & 3.0 $\pm$ 0.2 \\
        baskball & 1.34 $\pm$ 0.08 & 1.92 $\pm$ 0.09 & \underline{1.29 $\pm$ 0.09} & 1.6 $\pm$ 0.2 & 2.0 $\pm$ 0.2 & N/A & 2.1 $\pm$ 0.2 & \textbf{0.9 $\pm$ 0.1} \\
        blood-transfusion & 1.5 $\pm$ 0.2 & 1.8 $\pm$ 0.2 & \textbf{1.2 $\pm$ 0.2} & 1.8 $\pm$ 0.6 & 2.1 $\pm$ 0.7 & N/A & \underline{1.2 $\pm$ 0.1} & 1.7 $\pm$ 0.2 \\
        bodyfat & 2.5 $\pm$ 0.4 & 1.7 $\pm$ 0.2 & \underline{0.9 $\pm$ 0.2} & \textbf{0.8 $\pm$ 0.1} & N/A & N/A & N/A & 2.3 $\pm$ 0.1 \\
        breast-w & \textbf{1.1 $\pm$ 0.1} & 3.0 $\pm$ 0.2 & 2.3 $\pm$ 0.2 & 2.0 $\pm$ 0.2 & 3.4 $\pm$ 0.1 & N/A & 3.7 $\pm$ 0.2 & \underline{1.3 $\pm$ 0.1} \\
        chscase census2 & \underline{1.2 $\pm$ 0.2} & 1.2 $\pm$ 0.1 & \textbf{0.9 $\pm$ 0.1} & 1.3 $\pm$ 0.2 & N/A & N/A & N/A & 1.9 $\pm$ 0.2 \\
        chscase census6 & 1.3 $\pm$ 0.6 & N/A & N/A & 0.74 $\pm$ 0.01 & N/A & N/A & N/A & 1.6 $\pm$ 0.2 \\
        chscase vine1 & 2.0 $\pm$ 0.2 & 2.9 $\pm$ 0.2 & 2.1 $\pm$ 0.3 & \textbf{1.3 $\pm$ 0.3} & N/A & N/A & N/A & \underline{1.8 $\pm$ 0.2} \\
        confidence & \underline{2.29 $\pm$ 0.05} & 3.09 $\pm$ 0.04 & 2.76 $\pm$ 0.07 & N/A & 2.99 $\pm$ 0.05 & 2.7 $\pm$ 0.1 & N/A & \textbf{0.59 $\pm$ 0.09} \\
        diabetes & 2.9 $\pm$ 0.2 & 3.3 $\pm$ 0.2 & \textbf{1.8 $\pm$ 0.2} & \underline{2.9 $\pm$ 0.3} & 4.0 $\pm$ 0.3 & N/A & N/A & 3.7 $\pm$ 0.2 \\
        diabetes numeric & \underline{1.19 $\pm$ 0.10} & 1.5 $\pm$ 0.1 & \textbf{1.0 $\pm$ 0.1} & 1.4 $\pm$ 0.2 & 1.70 $\pm$ 0.10 & N/A & 1.7 $\pm$ 0.1 & 1.3 $\pm$ 0.2 \\
        diggle table a1 & 1.7 $\pm$ 0.1 & 1.9 $\pm$ 0.3 & \underline{1.5 $\pm$ 0.3} & 1.8 $\pm$ 0.2 & 2.1 $\pm$ 0.2 & N/A & 2.3 $\pm$ 0.1 & \textbf{1.3 $\pm$ 0.2} \\
        disclosure x noise & \textbf{0.8 $\pm$ 0.1} & 1.4 $\pm$ 0.1 & \underline{0.8 $\pm$ 0.1} & 1.3 $\pm$ 0.2 & 1.3 $\pm$ 0.1 & N/A & 0.84 $\pm$ 0.06 & 1.25 $\pm$ 0.09 \\
        ecoli & 4.4 $\pm$ 0.2 & 5.91 $\pm$ 0.05 & 5.710 $\pm$ 0.000 & N/A & 6.19 $\pm$ 0.07 & N/A & 6.19 $\pm$ 0.07 & \underline{3.7 $\pm$ 0.9} \\
        glass & 2.8 $\pm$ 0.3 & 3.9 $\pm$ 0.4 & 3.3 $\pm$ 0.5 & \textbf{1.7 $\pm$ 0.6} & 4.9 $\pm$ 0.3 & N/A & 4.3 $\pm$ 0.4 & \underline{2.5 $\pm$ 0.3} \\
        hayes-roth & 1.39 $\pm$ 0.05 & 1.75 $\pm$ 0.06 & \textbf{1.02 $\pm$ 0.06} & \underline{1.27 $\pm$ 0.08} & 2.48 $\pm$ 0.05 & N/A & N/A & 2.0 $\pm$ 0.1 \\
        heart-statlog & \underline{1.8 $\pm$ 0.1} & 2.9 $\pm$ 0.2 & 2.3 $\pm$ 0.3 & 2.6 $\pm$ 0.6 & 4.3 $\pm$ 0.3 & N/A & 4.06 $\pm$ 0.03 & \textbf{1.3 $\pm$ 0.2} \\
        iris & \textbf{1.3 $\pm$ 0.2} & 2.4 $\pm$ 0.1 & 2.1 $\pm$ 0.1 & 2.4 $\pm$ 0.2 & N/A & 2.0 $\pm$ 0.2 & N/A & \underline{1.6 $\pm$ 0.2} \\
        liver-disorders & \underline{1.6 $\pm$ 0.5} & 2.4 $\pm$ 0.6 & \textbf{1.6 $\pm$ 0.7} & N/A & N/A & N/A & N/A & 1.8 $\pm$ 0.3 \\
        longley & \underline{1.0 $\pm$ 0.1} & 1.7 $\pm$ 0.2 & 1.4 $\pm$ 0.2 & \textbf{0.9 $\pm$ 0.2} & 3.1 $\pm$ 0.2 & N/A & 3.4 $\pm$ 0.3 & 1.5 $\pm$ 0.2 \\
        machine cpu & 1.9 $\pm$ 0.2 & 2.3 $\pm$ 0.2 & \underline{1.7 $\pm$ 0.2} & \textbf{1.5 $\pm$ 0.2} & 3.4 $\pm$ 0.3 & N/A & 2.9 $\pm$ 0.5 & 2.7 $\pm$ 0.3 \\
        mu284 & 1.7 $\pm$ 0.1 & 1.62 $\pm$ 0.09 & \underline{1.17 $\pm$ 0.09} & \textbf{0.8 $\pm$ 0.1} & N/A & N/A & N/A & 3.4 $\pm$ 0.4 \\
        no2 & \underline{1.2 $\pm$ 0.2} & 1.8 $\pm$ 0.2 & \textbf{0.9 $\pm$ 0.1} & 1.2 $\pm$ 0.4 & N/A & N/A & N/A & 1.5 $\pm$ 0.1 \\
        pm10 & \textbf{1.47 $\pm$ 0.10} & 2.4 $\pm$ 0.1 & \underline{1.6 $\pm$ 0.1} & 1.8 $\pm$ 0.2 & N/A & N/A & N/A & 2.3 $\pm$ 0.1 \\
        prnn fglass & 3.1 $\pm$ 0.4 & 4.5 $\pm$ 0.5 & 4.0 $\pm$ 0.5 & \underline{2.8 $\pm$ 1.0} & N/A & N/A & N/A & \textbf{2.6 $\pm$ 0.3} \\
        rabe 131 & \underline{1.17 $\pm$ 0.09} & 2.0 $\pm$ 0.1 & 1.4 $\pm$ 0.1 & 1.6 $\pm$ 0.2 & 2.63 $\pm$ 0.10 & N/A & 2.6 $\pm$ 0.1 & \textbf{0.9 $\pm$ 0.1} \\
        sleep & \textbf{1.1 $\pm$ 0.2} & 1.5 $\pm$ 0.2 & \underline{1.3 $\pm$ 0.2} & N/A & N/A & N/A & N/A & 1.9 $\pm$ 0.2 \\
        strikes & \underline{1.71 $\pm$ 0.04} & 2.12 $\pm$ 0.05 & \textbf{1.59 $\pm$ 0.05} & 1.90 $\pm$ 0.06 & 2.64 $\pm$ 0.08 & N/A & N/A & 9.7 $\pm$ 0.2 \\
        vehicle & 2.3 $\pm$ 0.3 & 1.4 $\pm$ 0.3 & \underline{0.6 $\pm$ 0.2} & \textbf{0.20 $\pm$ 0.08} & N/A & N/A & N/A & 4.9 $\pm$ 0.5 \\
        wine & 1.7 $\pm$ 0.2 & 2.6 $\pm$ 0.3 & 1.9 $\pm$ 0.3 & \underline{1.4 $\pm$ 0.3} & N/A & N/A & N/A & \textbf{1.4 $\pm$ 0.1} \\
        \bottomrule
        \end{tabular}
        \end{table*}

        \begin{table*}[tbp]
            \centering
            \caption{Values represent the determinant of a kernel matrix based on pairwise distances between counterfactuals in each set, where higher values indicate greater diversity (more distinct, non-redundant alternatives). Error margins represent the standard error of the mean. ExDBSCAN achieves substantially higher diversity than all baselines across nearly all datasets (bold indicates best, underline indicates second-best), with typical improvements of $5-20 \times$ over model-agnostic methods. This superior performance stems from our electrostatic-like repulsion mechanism that spreads selected core points throughout the cluster's density-connectivity structure (Equation~\ref{eq:energy}). BayCon shows extremely poor diversity (often $<0.1$) despite generating many counterfactuals, because its optimisation does not account for DBSCAN's density-connected distance. GS-Surrogate occasionally shows competitive diversity when successful, but its low validity ($<50\%$) limits practical utility. ExDBSCAN Random performs well on some datasets, validating our repulsion-based approach, but sacrifices proximity. The results demonstrate ExDBSCAN's unique ability to balance both proximity and diversity through physics-inspired energy minimization. The best-performing method is highlighted in bold while the second-best is underlined.}
            \label{table:diversity_tabular_non_actionability_results}
            \footnotesize

    \begin{tabular}{lcccccccc}
    \toprule
    Dataset & ExDBSCAN & BayCon Random & BayCon Closest & GS Surrogate & DiCE Surrogate & GS Direct & DiCE Direct & ExDBSCAN Random \\
    \midrule
    autoPrice & \textbf{7.2 $\pm$ 0.2} & 0.58 $\pm$ 0.02 & 0.57 $\pm$ 0.01 & 0.95 $\pm$ 0.03 & \underline{1.000 $\pm$ 0.000} & N/A & N/A & 0.42 $\pm$ 0.08 \\
    baskball & \textbf{1.63 $\pm$ 0.05} & 0.031 $\pm$ 0.002 & 0.033 $\pm$ 0.004 & 0.27 $\pm$ 0.08 & 0.84 $\pm$ 0.06 & N/A & \underline{0.89 $\pm$ 0.10} & 0.046 $\pm$ 0.009 \\
    blood-transfusion & \textbf{1.3 $\pm$ 0.1} & 0.0057 $\pm$ 0.0009 & 0.0036 $\pm$ 0.0006 & 0.4 $\pm$ 0.2 & \underline{0.6 $\pm$ 0.2} & N/A & 0.5 $\pm$ 0.2 & 0.004 $\pm$ 0.001 \\
    bodyfat & \textbf{4.7 $\pm$ 0.2} & 0.33 $\pm$ 0.01 & 0.30 $\pm$ 0.02 & \underline{0.98 $\pm$ 0.02} & N/A & N/A & N/A & 0.30 $\pm$ 0.01 \\
    breast-w & \textbf{4.8 $\pm$ 0.2} & 0.30 $\pm$ 0.01 & 0.27 $\pm$ 0.02 & \underline{0.65 $\pm$ 0.07} & 0.128 $\pm$ 0.009 & N/A & 0.13 $\pm$ 0.01 & 0.28 $\pm$ 0.01 \\
    chscase census2 & \textbf{2.2 $\pm$ 0.1} & 0.081 $\pm$ 0.006 & 0.071 $\pm$ 0.008 & \underline{0.35 $\pm$ 0.07} & N/A & N/A & N/A & 0.079 $\pm$ 0.008 \\
    chscase census6 & \textbf{0.84 $\pm$ 0.04} & N/A & N/A & 0.3 $\pm$ 0.1 & N/A & N/A & N/A & N/A \\
    chscase vine1 & \textbf{5.3 $\pm$ 0.2} & 0.38 $\pm$ 0.02 & 0.36 $\pm$ 0.02 & \underline{0.8 $\pm$ 0.1} & N/A & N/A & N/A & 0.28 $\pm$ 0.02 \\
    confidence & \textbf{1.44 $\pm$ 0.05} & 0.023 $\pm$ 0.005 & \underline{0.035 $\pm$ 0.003} & N/A & 0.000008 $\pm$ 0.000003 & 0.026 $\pm$ 0.003 & N/A & N/A \\
    diabetes & \textbf{5.1 $\pm$ 0.1} & 0.27 $\pm$ 0.02 & 0.26 $\pm$ 0.02 & 0.74 $\pm$ 0.05 & \underline{1.000 $\pm$ 0.000} & N/A & N/A & 0.15 $\pm$ 0.02 \\
    diabetes numeric & \textbf{1.17 $\pm$ 0.08} & 0.0004 $\pm$ 0.0001 & 0.0025 $\pm$ 0.0006 & \underline{0.03 $\pm$ 0.02} & 0.0026 $\pm$ 0.0004 & N/A & 0.00006 $\pm$ 0.00002 & 0.006 $\pm$ 0.001 \\
    diggle table a1 & \textbf{1.60 $\pm$ 0.06} & 0.2 $\pm$ 0.2 & 0.2 $\pm$ 0.2 & 0.4 $\pm$ 0.1 & 0.93 $\pm$ 0.07 & N/A & \underline{1.000 $\pm$ 0.000} & N/A \\
    disclosure x noise & \textbf{0.97 $\pm$ 0.07} & 0.0033 $\pm$ 0.0006 & 0.0007 $\pm$ 0.0002 & \underline{0.12 $\pm$ 0.06} & 0.00014 $\pm$ 0.00005 & N/A & 0.00026 $\pm$ 0.00008 & 0.00021 $\pm$ 0.00005 \\
    ecoli & \textbf{6.6 $\pm$ 0.2} & 0.08 $\pm$ 0.02 & 0.019 $\pm$ 0.010 & N/A & 0.7 $\pm$ 0.1 & N/A & 0.7 $\pm$ 0.1 & N/A \\
    glass & \textbf{4.8 $\pm$ 0.2} & 0.38 $\pm$ 0.03 & 0.37 $\pm$ 0.03 & \underline{0.7 $\pm$ 0.1} & 0.056 $\pm$ 0.009 & N/A & 0.10 $\pm$ 0.01 & 0.18 $\pm$ 0.03 \\
    hayes-roth & \textbf{3.6 $\pm$ 0.1} & 0.084 $\pm$ 0.005 & 0.069 $\pm$ 0.005 & \underline{0.31 $\pm$ 0.04} & 0.058 $\pm$ 0.002 & N/A & N/A & 0.062 $\pm$ 0.006 \\
    heart-statlog & \textbf{7.0 $\pm$ 0.3} & 0.49 $\pm$ 0.03 & 0.49 $\pm$ 0.03 & 0.85 $\pm$ 0.09 & \underline{0.993 $\pm$ 0.007} & N/A & 0.99 $\pm$ 0.01 & 0.45 $\pm$ 0.02 \\
    iris & \textbf{2.92 $\pm$ 0.05} & 0.21 $\pm$ 0.06 & 0.19 $\pm$ 0.05 & \underline{0.46 $\pm$ 0.08} & N/A & 0.21 $\pm$ 0.03 & N/A & 0.022 $\pm$ 0.007 \\
    liver-disorders & \textbf{1.36 $\pm$ 0.08} & \underline{0.024 $\pm$ 0.005} & 0.021 $\pm$ 0.005 & N/A & N/A & N/A & N/A & 0.009 $\pm$ 0.003 \\
    longley & \textbf{2.1 $\pm$ 0.2} & 0.07 $\pm$ 0.03 & 0.08 $\pm$ 0.03 & \underline{0.46 $\pm$ 0.03} & 0.307 $\pm$ 0.009 & N/A & 0.352 $\pm$ 0.000 & N/A \\
    machine cpu & \textbf{1.08 $\pm$ 0.07} & 0.017 $\pm$ 0.002 & 0.015 $\pm$ 0.003 & \underline{0.33 $\pm$ 0.08} & 0.0009 $\pm$ 0.0003 & N/A & 0.0016 $\pm$ 0.0005 & 0.031 $\pm$ 0.002 \\
    mu284 & \textbf{2.15 $\pm$ 0.07} & 0.10 $\pm$ 0.03 & 0.10 $\pm$ 0.03 & \underline{0.7 $\pm$ 0.1} & N/A & N/A & N/A & 0.038 $\pm$ 0.004 \\
    no2 & \textbf{3.3 $\pm$ 0.3} & 0.17 $\pm$ 0.01 & 0.16 $\pm$ 0.02 & \underline{0.6 $\pm$ 0.1} & N/A & N/A & N/A & 0.17 $\pm$ 0.02 \\
    pm10 & \textbf{3.7 $\pm$ 0.1} & 0.151 $\pm$ 0.009 & 0.138 $\pm$ 0.008 & \underline{0.55 $\pm$ 0.07} & N/A & N/A & N/A & 0.156 $\pm$ 0.010 \\
    prnn fglass & \textbf{4.7 $\pm$ 0.2} & 0.44 $\pm$ 0.04 & 0.42 $\pm$ 0.04 & \underline{0.83 $\pm$ 0.10} & N/A & N/A & N/A & 0.25 $\pm$ 0.02 \\
    rabe 131 & \textbf{2.50 $\pm$ 0.08} & 0.134 $\pm$ 0.008 & 0.13 $\pm$ 0.01 & \underline{0.47 $\pm$ 0.07} & 0.0279 $\pm$ 0.0010 & N/A & 0.034 $\pm$ 0.000 & 0.21 $\pm$ 0.02 \\
    sleep & \textbf{0.90 $\pm$ 0.07} & 0.04 $\pm$ 0.04 & 0.04 $\pm$ 0.04 & N/A & N/A & N/A & N/A & N/A \\
    strikes & \textbf{2.30 $\pm$ 0.02} & 0.13 $\pm$ 0.01 & 0.12 $\pm$ 0.01 & 0.58 $\pm$ 0.02 & \underline{0.96 $\pm$ 0.01} & N/A & N/A & 0.062 $\pm$ 0.007 \\
    vehicle & \textbf{4.7 $\pm$ 0.3} & 0.26 $\pm$ 0.04 & 0.21 $\pm$ 0.05 & \underline{0.44 $\pm$ 0.05} & N/A & N/A & N/A & 0.27 $\pm$ 0.04 \\
    wine & \textbf{7.6 $\pm$ 0.4} & 0.50 $\pm$ 0.01 & 0.51 $\pm$ 0.01 & \underline{0.71 $\pm$ 0.06} & N/A & N/A & N/A & 0.43 $\pm$ 0.02 \\
    \bottomrule
    \end{tabular}
    \end{table*}

    \begin{table*}[tbp]
        \centering
        \caption{Validity scores across $30$ datasets when a subset of the features is non-actionable. Values represent the proportion of queries ($0.0-1.0$) for which each method successfully generated valid counterfactuals that satisfy DBSCAN's density-connectivity criteria. Higher values indicate better performance. Error margins represent the standard error of the mean. ExDBSCAN, ExDBSCAN Random, and DiCE Surrogate achieve perfect validity ($1.00$) across all datasets. Model-agnostic baselines (BayCon, GS-Surrogate) show poor validity (<0.50 on average) due to misalignment between their optimisation objectives and DBSCAN's discrete assignment function. DiCE-Direct and GS-Direct perform poorly because they cannot effectively navigate DBSCAN's discrete landscape without continuous probabilities or gradients. Bold indicates best performance and the second-best is underlined.}
        \label{table:tabular_non_actionability_validity_results}
        \footnotesize
        \begin{tabular}{lcccccccc}
        \toprule
        Dataset & ExDBSCAN & BayCon Random & BayCon Closest & GS Surrogate & DiCE Surrogate & GS Direct & DiCE Direct & ExDBSCAN Random \\
        \midrule
        autoPrice & \textbf{1.00} & 0.46 & 0.46 & 0.18 & 0.03 & 0.00 & 0.00 & \textbf{1.00} \\
        baskball & \textbf{1.00} & 0.78 & 0.78 & 0.33 & 0.45 & 0.00 & 0.25 & \textbf{1.00} \\
        blood-transfusion & \textbf{1.00} & 0.55 & 0.55 & 0.12 & 0.20 & 0.00 & 0.12 & \textbf{1.00} \\
        bodyfat & \textbf{1.00} & 0.36 & 0.36 & 0.08 & 0.00 & 0.00 & 0.00 & \textbf{1.00} \\
        breast-w & \textbf{1.00} & 0.93 & 0.93 & 0.45 & \textbf{1.00} & 0.00 & 0.50 & 1.00 \\
        chscase census2 & \textbf{1.00} & 0.29 & 0.29 & 0.39 & 0.00 & 0.00 & 0.00 & \textbf{1.00} \\
        chscase census6 & \textbf{1.00} & 0.04 & 0.04 & 0.12 & 0.00 & 0.00 & 0.00 & \textbf{1.00} \\
        chscase vine1 & \textbf{1.00} & 0.47 & 0.47 & 0.12 & 0.00 & 0.00 & 0.00 & \textbf{1.00} \\
        confidence & \textbf{1.00} & 0.35 & 0.35 & 0.00 & \textbf{1.00} & 1.00 & 0.00 & \textbf{1.00} \\
        diabetes & \textbf{1.00} & 0.42 & 0.42 & 0.40 & 0.05 & 0.00 & 0.00 & \textbf{1.00} \\
        diabetes numeric & \textbf{1.00} & 0.72 & 0.72 & 0.61 & \textbf{1.00} & 0.00 & 0.44 & 1.00 \\
        diggle table a1 & \textbf{1.00} & 0.18 & 0.18 & 0.43 & 0.18 & 0.00 & 0.11 & \textbf{1.00} \\
        disclosure x noise & \textbf{1.00} & 0.70 & 0.70 & 0.55 & \textbf{1.00} & 0.00 & 0.50 & 1.00 \\
        ecoli & \textbf{1.00} & 0.17 & 0.17 & 0.05 & 0.33 & 0.00 & 0.33 & \textbf{1.00} \\
        glass & \textbf{1.00} & 0.57 & 0.57 & 0.12 & \textbf{1.00} & 0.00 & 0.50 & \textbf{1.00} \\
        hayes-roth & \textbf{1.00} & 0.81 & 0.81 & 0.73 & \textbf{1.00} & 0.00 & 0.00 & \textbf{1.00} \\
        heart-statlog & \textbf{1.00} & 0.53 & 0.53 & 0.23 & 0.10 & 0.00 & 0.05 & \textbf{1.00} \\
        iris & \textbf{1.00} & 0.80 & 0.80 & 0.50 & 0.00 & 0.75 & 0.00 & \textbf{1.00} \\
        liver-disorders & \textbf{1.00} & 0.70 & 0.70 & N/A & N/A & N/A & N/A & \textbf{1.00} \\
        longley & \textbf{1.00} & 0.16 & 0.16 & 0.08 & \textbf{1.00} & 0.00 & 0.52 & \textbf{1.00} \\
        machine cpu & \textbf{1.00} & 0.65 & 0.65 & 0.50 & \textbf{1.00} & 0.00 & 0.50 & \textbf{1.00} \\
        mu284 & \textbf{1.00} & 0.48 & 0.48 & 0.09 & 0.00 & 0.00 & 0.00 & \textbf{1.00} \\
        no2 & \textbf{1.00} & 0.59 & 0.59 & 0.21 & 0.00 & 0.00 & 0.00 & \textbf{1.00} \\
        pm10 & \textbf{1.00} & 0.61 & 0.61 & 0.27 & 0.00 & 0.00 & 0.00 & \textbf{1.00} \\
        prnn fglass & \textbf{1.00} & 0.68 & 0.68 & 0.12 & 0.00 & 0.00 & 0.00 & \textbf{1.00} \\
        rabe 131 & \textbf{1.00} & 0.85 & 0.85 & 0.70 & \textbf{1.00} & 0.00 & 0.50 & 1.00 \\
        sleep & \textbf{1.00} & 0.19 & 0.19 & 0.04 & 0.00 & 0.00 & 0.00 & \textbf{1.00} \\
        strikes & \textbf{1.00} & 0.23 & 0.23 & 0.17 & 0.06 & 0.00 & 0.00 & \textbf{1.00} \\
        vehicle & \textbf{1.00} & 0.12 & 0.12 & 0.05 & 0.00 & 0.00 & 0.00 & \textbf{1.00} \\
        wine & \textbf{1.00} & 0.55 & 0.55 & 0.25 & 0.00 & 0.00 & 0.00 & \textbf{1.00} \\
        \bottomrule
        \end{tabular}
\end{table*}

\subsection{Euclidean distance's limitations}\label{app:euclideanLimitation}
    
\begin{figure*}[tbp]
\centering
\includegraphics[width=0.7\linewidth]{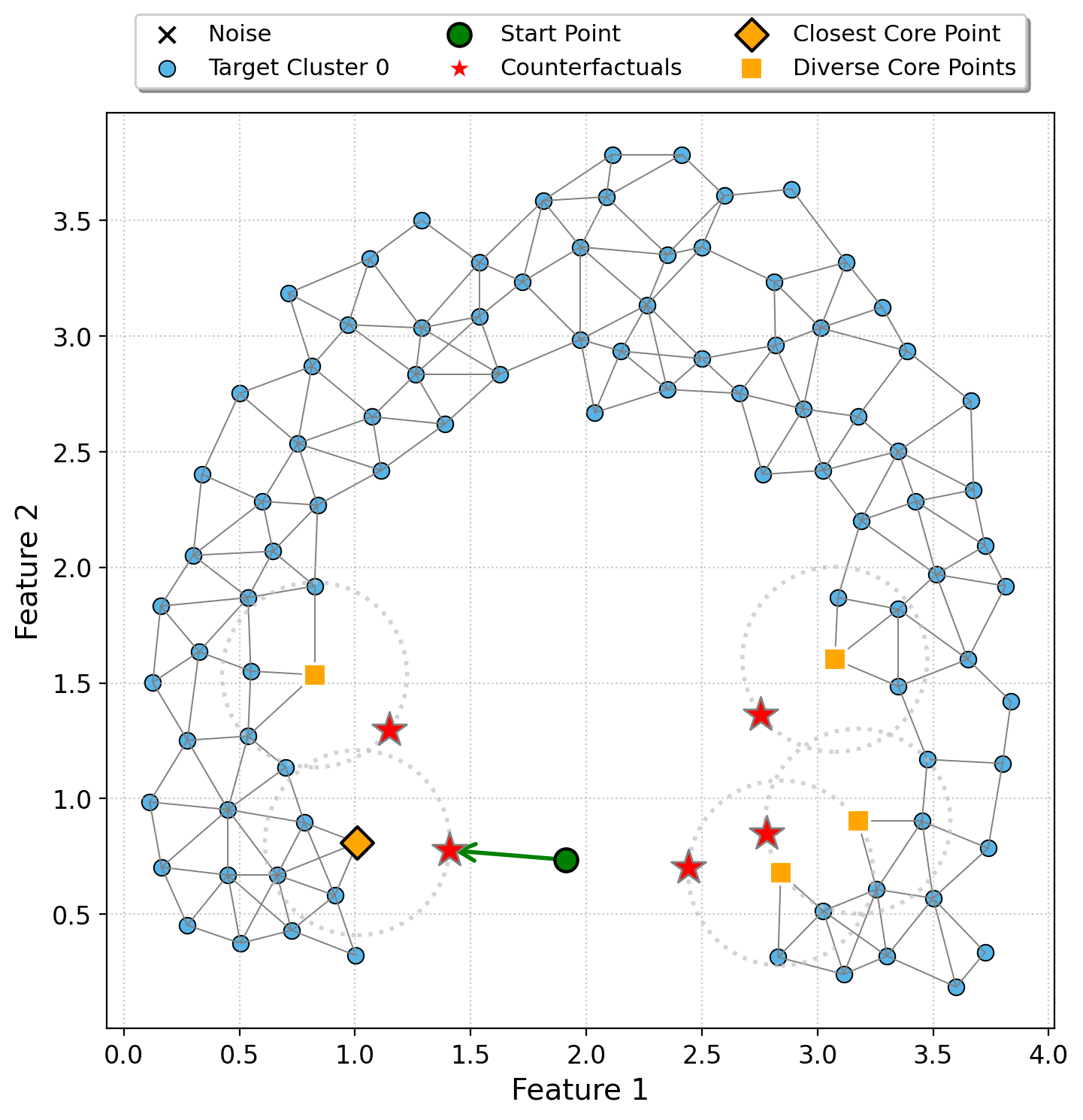}
\caption{Counterfactuals (red stars) generated for a noise point (green circle) targeting a half-circle cluster. Orange squares indicate the reference core points for each counterfactual. Points at opposite ends of the half-circle arc (e.g., leftmost and rightmost counterfactuals) appear close in Euclidean space but are far apart in terms of DBSCAN's density-connected path distance and traversing from one end to the other requires many intermediate density-connected steps along the arc. This illustrates why ExDBSCAN uses weighted shortest-path distance in the cluster graph $G$ rather than Euclidean distance for measuring diversity: counterfactuals that appear redundant under Euclidean distance actually represent distinct alternatives when considering the cluster's true density-connectivity structure. Using Euclidean distance for diversity measurement would incorrectly penalise these counterfactuals as redundant, failing to recognise that they explore different density-connected regions of the target cluster.}
\Description{Counterfactuals (red stars) generated for a noise point (green circle) targeting a half-circle cluster. Orange squares indicate the reference core points for each counterfactual. Points at opposite ends of the half-circle arc (e.g., leftmost and rightmost counterfactuals) appear close in Euclidean space but are far apart in terms of DBSCAN's density-connected path distance and traversing from one end to the other requires many intermediate density-connected steps along the arc. This illustrates why ExDBSCAN uses weighted shortest-path distance in the cluster graph $G$ rather than Euclidean distance for measuring diversity: counterfactuals that appear redundant under Euclidean distance actually represent distinct alternatives when considering the cluster's true density-connectivity structure. Using Euclidean distance for diversity measurement would incorrectly penalise these counterfactuals as redundant, failing to recognise that they explore different density-connected regions of the target cluster.}
\label{fig:half_circle}
\end{figure*}

Instances that appear redundant based on Euclidean distance, might add valuable diversity when considering the cluster's structure. Using our shortest path weighted distance properly takes into account the clusters' structure.
We exemplify this in Figure~\ref{fig:half_circle}, where counterfactuals are proposed for a noise point toward a half-circle cluster. Points near the opposite ends of the arc appear close in the Euclidean space, but under DBSCAN they are density distant; a critical distinction when evaluating diversity for CEs.

            %      \caption{
            %  %Counterfactuals (red stars) generated for a noise point (green circle) targeting a half-circle cluster. Orange squares indicate the reference core points for each counterfactual. 
            %  %Points at opposite ends of the half-circle arc (e.g., leftmost and rightmost counterfactuals) appear close in Euclidean space but are far apart in terms of DBSCAN's density-connected path distance and traversing from one end to the other requires many intermediate density-connected steps along the arc. This illustrates why ExDBSCAN uses weighted shortest-path distance in the cluster graph $G$ rather than Euclidean distance for measuring diversity: counterfactuals that appear redundant under Euclidean distance actually represent distinct alternatives when considering the cluster's true density-connectivity structure. Using Euclidean distance for diversity measurement would incorrectly penalize these counterfactuals as redundant, failing to recognize that they explore different density-connected regions of the target cluster.}
            % }

 \subsection{NP-hardness}~\label{subsec:npHardness}
     Proposition~\ref{prop:npHardness}, states that the optimisation problem in Equation~\ref{eq:optimisationProblem} is \textit{NP-hard}, we present here a more extended proof.
 
     \begin{proof}
         Consider the energy landscape in Equation~\ref{eq:energy}, the electrostatic term reads (in contracted notation) 
         $\sum\limits_{i, j >i} \frac{1}{\mathcal{D}(V_i, V_j)}$
        , which needs to be minimised to find the optimal energy configuration.
        Minimising the sum of the energy terms is equivalent to maximise their negative. Thus, the problem can be mapped to a maximum diversity problem where the object summed is the inverse distance, MDP optimisation is known to be NP-hard~\citep{parreno2021measuring}. Consider now adding the spring terms, considering Equation~\ref{eq:energy} in full. We can again map the minimisation problem to a maximisation one by taking the negative energy. If this new optimisation problem would be solvable in polynomial time, we could move the spring constants (which we consider equal to $1$) to $0$. In this limit, the MDP could be also solved in polynomial time, contradicting its NP-hardness.
     \end{proof}

\subsection{Ablation: Isolating Attraction and Repulsion} 
    Our objective (Eq.~\ref{eq:energy}) combines a spring-like \emph{attraction} term which biases selection toward proximal core points, and an electrostatic-like \emph{repulsion} term, spreading the selected cores for diversity. To isolate the contribution of each term, we evaluate two ablations that share ExDBSCAN's optimisation and construction steps, but retain only one term:
    
    \begin{itemize}
        \item \textbf{ExDBSCAN Nearest} keeps only the spring-like attraction term, dropping repulsion from Eq.~\ref{eq:energy}; the optimisation thus targets proximity alone. By Prop.~\ref{prop:knn}, this reduces to selecting the $k$ nearest core points of $p$.
    
        \item \textbf{ExDBSCAN Furthest} keeps only the repulsion term, dropping attraction. Selection then maximises spread across the target cluster, pursuing diversity while disregarding proximity.
    \end{itemize}

    Together with \emph{ExDBSCAN Random} (random core selection, no optimisation; Sec.~\ref{subsec:cfexp}) and full \emph{ExDBSCAN} (both terms), these variants give a term-by-term decomposition of the objective. As all three reuse ExDBSCAN's validity-guaranteeing construction (Eq.~\ref{eq:cfgen}, Theo.~\ref{theo:1}), each attains perfect validity; the comparison thus isolates how proximity and diversity responds to each energy term in isolation. 

    \paragraph{Proximity (Fig.~\ref{fig:proximity_new_baselines}).} As expected, ExDBSCAN Nearest yields the most proximal counterfactuals and consistently improves over ExDBSCAN Furthest. Notably, full ExDBSCAN tracks ExDBSCAN Nearest closely across the entire range of queries, showing that weighting diversity costs little proximity as ExDBSCAN remains highly competitive on this metric. 

    \paragraph{Diversity (Fig.~\ref{fig:proximity_new_baselines}).} ExDBSCAN Furthest is more diverse than ExDBSCAN Nearest throughout, confirming that the repulsion term drives diversity. More interestingly, full ExDBSCAN attains the highest diversity across the majority of queries, although the two single-term variants are marginally higher on the most diverse $20\%$ of queries, ExDBSCAN overtakes both as more queries are considered and remains highest thereafter. We attribute this to ExDBSCAN's anchored initialisation. The attraction term seeds the greedy selection with the core point nearest to $p$ (Sec.~\ref{sec:method}), a stable, well-positioned starting point from which repulsion spreads the remaining cores. Pure repulsion lacks such an anchor and is more variable, excelling on a few queries but spreading less consistently across the rest.

    Overall, the ablation shows the two terms are complementary rather than a mere compromise. Attraction recovers near-optimal proximity (matching ExDBSCAN Nearest), while the combined objective achieves the best aggregate diversity, surpassing even the repulsion-only variant. 

    \begin{figure*}[tbp]
        \centering
        \includegraphics[width=.98\linewidth]{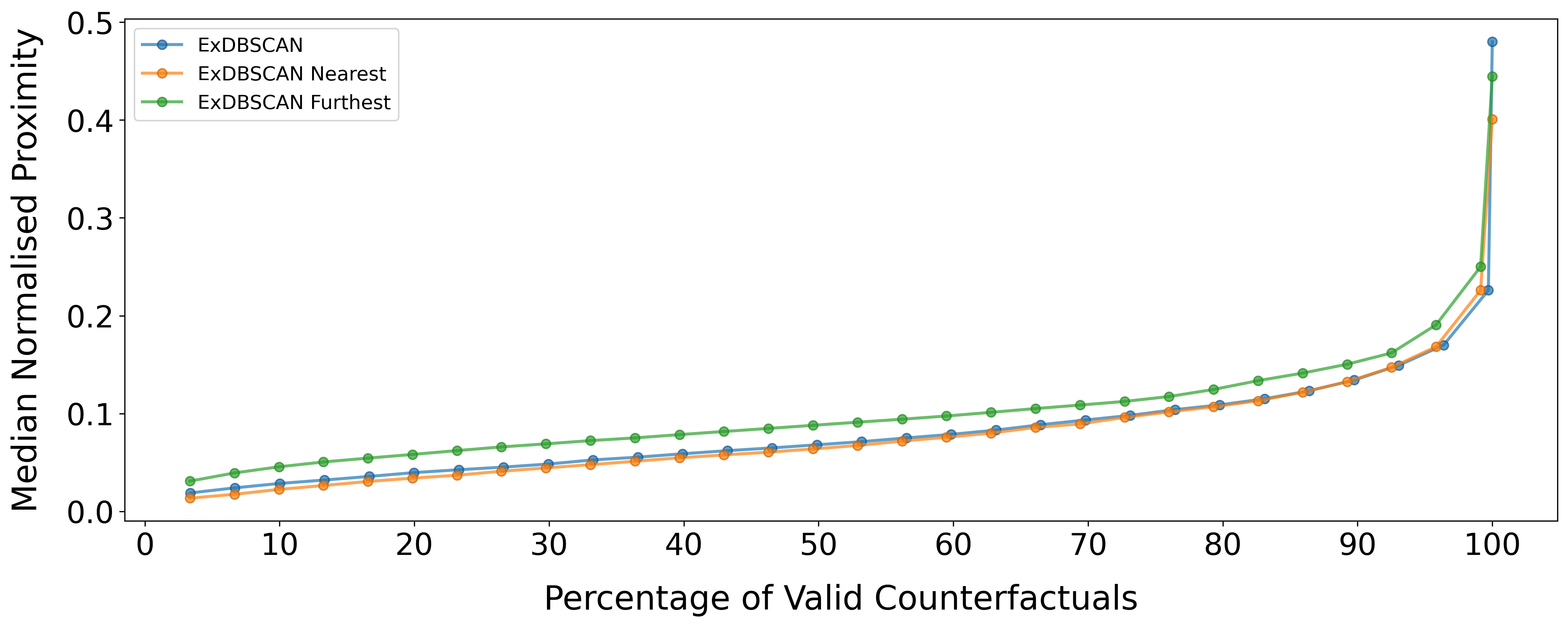}
        \caption{Proximity ablation. Cumulative percentage of successful queries (x-axis) vs.\ diversity (y-axis, $\downarrow$ is better). ExDBSCAN and ExDBSCAN Nearest achieve near-identical proximity, both improving over ExDBSCAN Furthest.}
        \Description{Empirical percentile curve showing the percentage of valid counterfactuals in relation to median normalised proximity.}
        \label{fig:proximity_new_baselines}
    \end{figure*}

    \begin{figure*}[tbp]
        \centering
        \includegraphics[width=.98\linewidth]{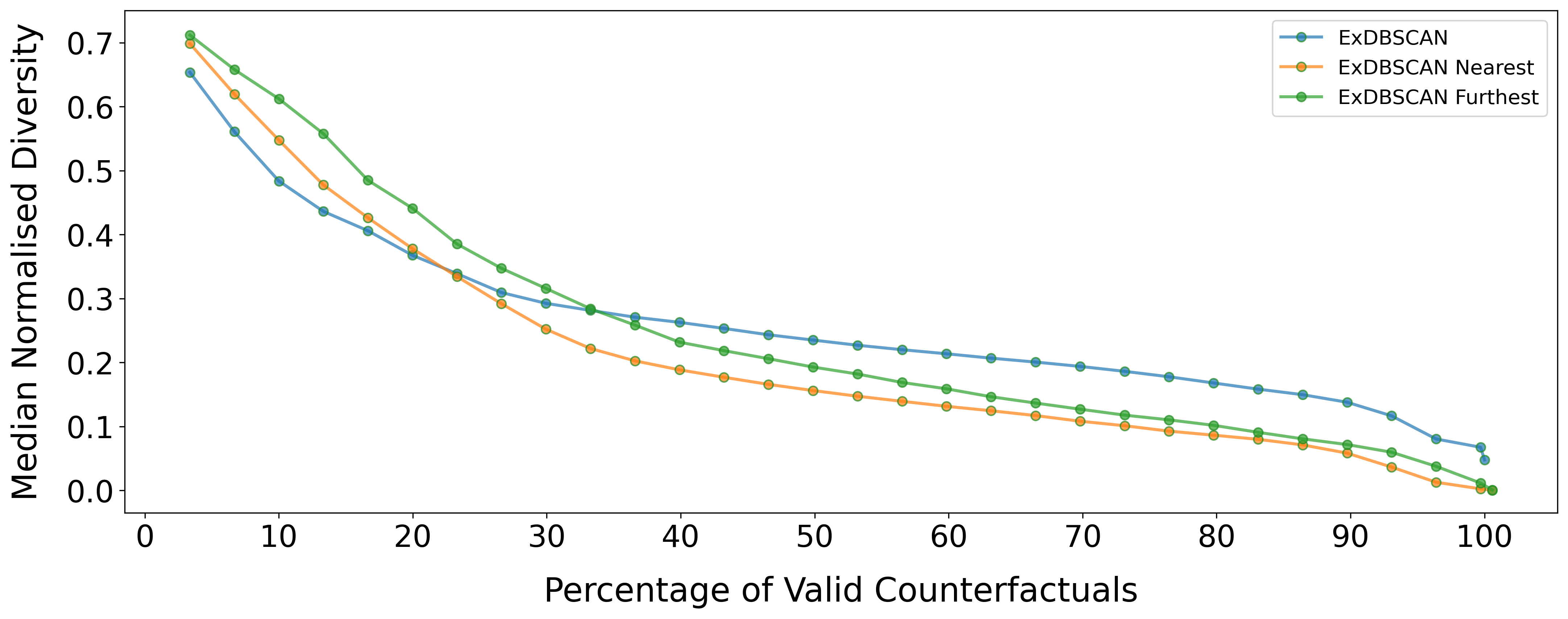}
        \caption{Diversity ablation. Cumulative percentage of successful queries (x-axis) vs.\ diversity (y-axis, $\uparrow$ is better). ExDBSCAN Furthest is more diverse than ExDBSCAN Nearest. On the most diverse queries both single-term variants slightly exceed ExDBSCAN, but across the majority of queries ExDBSCAN attains the highest diversity.}
        \Description{Empirical percentile curve showing the percentage of valid counterfactuals in relation to median normalised diversity.}
        \label{fig:diversity_new_baselines}
    \end{figure*}

\subsection{A Study On The California Housing Dataset}

In this section, we will first discuss how ExDBSCAN can be simply further expanded by implementing feature constraints like monotonicity and positive correlation. Afterwards, we will implement these feature constraints on the California Housing Dataset~\cite{tensorflow2015CaliforniaHousing, pace1997sparse}. 

We show that both constraints can be implemented by a straightforward filtering of set of core points that can be selected by the energy minimisation optimisation (Eq.~\ref{eq:optimisationProblem}).
\begin{itemize}
    \item \textbf{Monotonicity}; when a feature can only increase or only decrease e.g. Age, we filter out every core point whose entire $\varepsilon$ neighbourhood is outside the valid region from the set of core points that are inspected by the greedy algorithm solving Eq.~\ref{eq:optimisationProblem}. We consider the valid region the points in the data space where the feature of interest increases (or respectively decreases) with respect to its starting point value. When there is a core point whose $\varepsilon$ neighbourhood is partially inside the valid region, the core point can be selected from the greedy algorithm with the starting point that can be moved only inside the valid region of the $\varepsilon$ neighbourhood. 

    \item \textbf{Positive/Negative Correlation}; when two or more features are positively or negatively correlated, we adopt a filtering strategy analogous the above monotonicity case. When two features are for example positively correlated we consider the valid region the portion of the data space where both attributes either both increase or both decrease.
\end{itemize}

Next, we employ the California Housing Dataset as a use case to practically shoe the effect of applying the above mentioned constraints.

California Housing is a dataset derived from the $1990$ U.S. census reporting the following per-block features.

\begin{itemize}
    \item \textit{Latitude}: east-west position of the block.

    \item \textit{Longitude}: north-south position of the block.

    \item \textit{Housing Median Age}: the median age of houses within a block.

    \item \textit{Rooms}: number of total rooms in a block.

    \item \textit{Bedrooms}: number of total bedrooms in a block. 

    \item \textit{Population}: the total number of people living in the block.

    \item \textit{Households}: number of total households in a block.

    \item \textit{Median Income}: median income of households in tens of thousands of dollars.
    
\end{itemize}

We fit DBSCAN on the California Housing Dataset and we impose two constraints. First of all, $Housing Median Age$ can increase at will but it's difficult to decrease it, thus we impose that the median house age can decrease of maximum $2$ years, allowing wiggle room in case new buildings are built in the block. Furthermore, it is unrealistic that one of \textit{Rooms} and \textit{Bedrooms} increases with the other decreasing. Hence, we impose \textit{Rooms} and \textit{Bedrooms} to be positively correlated.

We run ExDBSCAN with and without constraints to better highlight the constraints' effect on the proposed counterfactuals. We analyse two data point, both labelled as noise by ExDBSCAN. 

The first noise point $[-122.1, 37.8, 28, 794, 111, 329, 109, 7.7]$, has as the first counterfactual toward cluster $1$ from the unconstrained ExDBSCAN the point $[-122.1,37.8,25.7,542,69,212,66,8.1]$. We notice how, while the positive correlation constraint is respected as both decrease, the constraint on the \textit{Age} feature is not as it drops by more than $2$ years. The constrained ExDBSCAN returns instead as the first counterfactual $[-122.1,37.7,28,2193,342,977,350,5.5]$ where we can see the \textit{Age} feature is brought back to the original value.

Similarly, the unconstrained ExDBSCAN produces as its second counterfactual for the noise point $[-122.3,37.9,52,1087,370,3337,350,1.4]$ toward cluster $2$, the point $[-122.2,37.8,31,2137,330,1484,330,5.7]$ violating the positive correlation constraint as the total number of rooms increase while he total number of bedrooms decrease. Constraint ExDBSCAN proposes instead as the fourth counterfactual the point $[-122.1,37.7,17,2779,371,351,318,8.4]$ with all the constraints that are now respected.

\subsection{Sparsity and Plausibility}

We present here sparsity and plausibility results for ExDBSCAN, BayCon Random and BayCon Closest. We represent sparsity as the proportion of the feature that change from the original point to the counterfactual. Thus, the lower the sparsity value reported the better. In line with counterfactual literature~\cite{zhou2025eace, kanamori2020dace}, we measure sparsity through the LOF score. LOF~\cite{Breunig00LOF} is a method scoring the outlierness of points based on density consideration: if a point is in a dense area its outlier score will be low and vice-versa. An LOF value of $1$ signals a point that is in-distribution and thus has high plausibility. As the LOF score increases, the point increasingly looks more like an outlier and its plausibility decreases.

Table~\ref{table:sparisty_plausibility} reports sparsity and plausibility for ExDBSCAN, BayCon Random and BayCon closest. The two BayCon-based baselines achieve  better sparsity. BayCon, in its algorithmic formulation optimises explicitly for sparsity, while ExDBSCAN does not, focusing on proximity and validity.

As far as plausibility, all three methods have very often an LOF score below $1.5$, which is usually taken as the threshold to decide whether a point is an outlier or not, signalling good plausibility. BayCon achieves better plausibility on a higher number of datasets. Analysing plausibility, it is important to take into consideration proximity as well. As shown in Figure~\ref{fig:proximity_validity_actionable}, ExDBSCAN outermors BayCon in terms of proximity. Thus, proposed counterfactuals will be more on the outside edge of target clusters. ExDBSCAN LOF scores, i.e.\ often below $1.5$, signal that while ExDBSCAN produces highly proximal counterfactuals, it does not sacrifice their plausibility.

\begin{table*}[tbp]
        \centering
        \caption{Plausibility and Sparsity scores across $30$ datasets. Values for sparsity represent the average proportion of features ($0-1$ range) that change from the original point of the counterfactual. Lower values indicate better sparsity. Values for plausibility represent the average LOF score of the proposed counterfactuals, values close to $1$ indicate that the sample has similar density to that of its neighbours, making it an inlier. The higher the LOF the higher the outlierness of the counterfactuals. For both metrics error margins represent the standard error of the mean. The best value for both plausibility and sparsity is highlighted in bold while the second-best value is underlined.}
        \label{table:sparisty_plausibility}
        \footnotesize
\begin{tabular}{lcccccc}
\toprule
Dataset & \multicolumn{2}{c}{ExDBSCAN} & \multicolumn{2}{c}{BayCon Random} & \multicolumn{2}{c}{BayCon Closest} \\
\cmidrule(lr){2-3}\cmidrule(lr){4-5}\cmidrule(lr){6-7}
 & Sparsity & Plausibility & Sparsity & Plausibility & Sparsity & Plausibility \\
\midrule
autoPrice & 0.950 $\pm$ 0.006 & \textbf{2.0 $\pm$ 0.2} & \underline{0.23 $\pm$ 0.02} & 2.2 $\pm$ 0.3 & \textbf{0.19 $\pm$ 0.02} & \underline{2.0 $\pm$ 0.2} \\
baskball & 0.960 $\pm$ 0.006 & 1.030 $\pm$ 0.003 & \underline{0.55 $\pm$ 0.01} & \underline{0.988 $\pm$ 0.002} & \textbf{0.53 $\pm$ 0.02} & \textbf{0.991 $\pm$ 0.002} \\
blood-transfusion & 0.946 $\pm$ 0.009 & 1.17 $\pm$ 0.02 & \underline{0.49 $\pm$ 0.02} & \textbf{1.12 $\pm$ 0.01} & \textbf{0.46 $\pm$ 0.02} & \underline{1.16 $\pm$ 0.03} \\
bodyfat & 0.985 $\pm$ 0.002 & \textbf{1.090 $\pm$ 0.010} & \textbf{0.172 $\pm$ 0.008} & 1.14 $\pm$ 0.02 & \underline{0.20 $\pm$ 0.02} & \underline{1.14 $\pm$ 0.02} \\
breast-w & 0.81 $\pm$ 0.01 & 1.43 $\pm$ 0.05 & \underline{0.43 $\pm$ 0.03} & \underline{1.36 $\pm$ 0.06} & \textbf{0.39 $\pm$ 0.03} & \textbf{1.33 $\pm$ 0.05} \\
chscase census2 & 1.000 $\pm$ 0.000 & 1.048 $\pm$ 0.006 & \underline{0.38 $\pm$ 0.02} & \underline{1.03 $\pm$ 0.01} & \textbf{0.36 $\pm$ 0.02} & \textbf{1.03 $\pm$ 0.01} \\
chscase census6 & \textbf{1.000 $\pm$ 0.000} & \textbf{1.01 $\pm$ 0.01} & -- & -- & -- & -- \\
chscase vine1 & 0.950 $\pm$ 0.005 & 1.046 $\pm$ 0.006 & \underline{0.48 $\pm$ 0.04} & \textbf{1.013 $\pm$ 0.005} & \textbf{0.46 $\pm$ 0.04} & \underline{1.019 $\pm$ 0.006} \\
confidence & 1.000 $\pm$ 0.000 & 5.5 $\pm$ 0.8 & \textbf{0.667 $\pm$ 0.000} & \textbf{2.2 $\pm$ 0.4} & \underline{0.667 $\pm$ 0.000} & \underline{3.0 $\pm$ 0.4} \\
diabetes & 0.81 $\pm$ 0.01 & \underline{1.39 $\pm$ 0.03} & \underline{0.30 $\pm$ 0.01} & \textbf{1.26 $\pm$ 0.02} & \textbf{0.19 $\pm$ 0.02} & 1.41 $\pm$ 0.04 \\
diabetes numeric & 0.994 $\pm$ 0.003 & \underline{1.11 $\pm$ 0.03} & \textbf{0.500 $\pm$ 0.000} & \textbf{1.06 $\pm$ 0.02} & \underline{0.500 $\pm$ 0.000} & 1.14 $\pm$ 0.04 \\
diggle table a1 & 0.96 $\pm$ 0.01 & \textbf{1.000 $\pm$ 0.006} & \textbf{0.49 $\pm$ 0.02} & \underline{1.01 $\pm$ 0.02} & \underline{0.53 $\pm$ 0.02} & 0.99 $\pm$ 0.02 \\
disclosure x noise & 1.000 $\pm$ 0.000 & 1.12 $\pm$ 0.01 & \underline{0.59 $\pm$ 0.02} & \textbf{1.047 $\pm$ 0.008} & \textbf{0.55 $\pm$ 0.02} & \underline{1.10 $\pm$ 0.02} \\
ecoli & 0.842 $\pm$ 0.005 & 1.8 $\pm$ 0.1 & \underline{0.33 $\pm$ 0.04} & \textbf{1.4 $\pm$ 0.4} & \textbf{0.24 $\pm$ 0.09} & \underline{1.4 $\pm$ 0.5} \\
glass & 0.88 $\pm$ 0.01 & 1.89 $\pm$ 0.07 & \underline{0.37 $\pm$ 0.03} & \textbf{1.53 $\pm$ 0.08} & \textbf{0.36 $\pm$ 0.03} & \underline{1.53 $\pm$ 0.10} \\
hayes-roth & 0.706 $\pm$ 0.009 & \underline{1.097 $\pm$ 0.004} & \underline{0.506 $\pm$ 0.009} & \textbf{1.096 $\pm$ 0.007} & \textbf{0.47 $\pm$ 0.01} & 1.14 $\pm$ 0.01 \\
heart-statlog & 0.63 $\pm$ 0.01 & 1.13 $\pm$ 0.01 & \underline{0.40 $\pm$ 0.03} & \textbf{1.050 $\pm$ 0.005} & \textbf{0.35 $\pm$ 0.03} & \underline{1.062 $\pm$ 0.007} \\
iris & 0.986 $\pm$ 0.004 & 3.7 $\pm$ 0.1 & \underline{0.65 $\pm$ 0.02} & \textbf{2.3 $\pm$ 0.2} & \textbf{0.63 $\pm$ 0.02} & \underline{2.4 $\pm$ 0.2} \\
liver-disorders & 0.956 $\pm$ 0.008 & 1.27 $\pm$ 0.02 & \textbf{0.45 $\pm$ 0.05} & \textbf{1.11 $\pm$ 0.02} & \underline{0.47 $\pm$ 0.05} & \underline{1.16 $\pm$ 0.03} \\
longley & 1.000 $\pm$ 0.000 & 1.04 $\pm$ 0.02 & \underline{0.49 $\pm$ 0.08} & \underline{0.97 $\pm$ 0.02} & \textbf{0.48 $\pm$ 0.08} & \textbf{0.97 $\pm$ 0.02} \\
machine cpu & 0.85 $\pm$ 0.02 & \underline{1.43 $\pm$ 0.04} & \underline{0.311 $\pm$ 0.009} & \textbf{1.31 $\pm$ 0.02} & \textbf{0.291 $\pm$ 0.008} & 1.59 $\pm$ 0.04 \\
mu284 & 0.951 $\pm$ 0.006 & \textbf{1.13 $\pm$ 0.02} & \underline{0.33 $\pm$ 0.02} & \underline{1.19 $\pm$ 0.04} & \textbf{0.31 $\pm$ 0.02} & 1.21 $\pm$ 0.05 \\
no2 & 0.973 $\pm$ 0.004 & \underline{1.12 $\pm$ 0.02} & \underline{0.30 $\pm$ 0.02} & \textbf{1.10 $\pm$ 0.02} & \textbf{0.26 $\pm$ 0.02} & 1.14 $\pm$ 0.03 \\
pm10 & 0.972 $\pm$ 0.002 & 1.135 $\pm$ 0.009 & \underline{0.37 $\pm$ 0.02} & \textbf{1.08 $\pm$ 0.01} & \textbf{0.34 $\pm$ 0.02} & \underline{1.10 $\pm$ 0.02} \\
prnn fglass & 0.88 $\pm$ 0.01 & 1.95 $\pm$ 0.08 & \underline{0.43 $\pm$ 0.03} & \textbf{1.58 $\pm$ 0.07} & \textbf{0.43 $\pm$ 0.03} & \underline{1.61 $\pm$ 0.09} \\
rabe 131 & 0.996 $\pm$ 0.001 & \textbf{1.016 $\pm$ 0.005} & \underline{0.45 $\pm$ 0.02} & \underline{1.019 $\pm$ 0.008} & \textbf{0.43 $\pm$ 0.02} & 1.02 $\pm$ 0.01 \\
sleep & 0.94 $\pm$ 0.02 & 1.32 $\pm$ 0.05 & \underline{0.49 $\pm$ 0.05} & \textbf{1.2 $\pm$ 0.1} & \textbf{0.49 $\pm$ 0.05} & \underline{1.2 $\pm$ 0.1} \\
strikes & 0.968 $\pm$ 0.002 & \underline{1.215 $\pm$ 0.005} & \underline{0.455 $\pm$ 0.005} & \textbf{1.210 $\pm$ 0.007} & \textbf{0.424 $\pm$ 0.006} & 1.233 $\pm$ 0.009 \\
vehicle & 0.961 $\pm$ 0.006 & 1.27 $\pm$ 0.03 & \underline{0.34 $\pm$ 0.10} & \textbf{1.16 $\pm$ 0.04} & \textbf{0.3 $\pm$ 0.1} & \underline{1.17 $\pm$ 0.04} \\
wine & 0.987 $\pm$ 0.002 & 1.180 $\pm$ 0.008 & \underline{0.30 $\pm$ 0.04} & \textbf{1.14 $\pm$ 0.01} & \textbf{0.28 $\pm$ 0.04} & \underline{1.16 $\pm$ 0.01} \\
\bottomrule
\end{tabular}
\end{table*}

\subsection{Spatial Datasets}
DBSCAN is often used to cluster spatial datasets~\cite{varghese2013spatial} as clustering permits a generalization of the spatial component like explicit location and extension of spatial objects which define implicit relations of spatial neighbourhood~\cite{varghese2013spatial}. Furthermore, geographical data often has a density-based structure, with positions clustered in hubs, making DBSCAN a natural fit. We compare here ExDBSCAN with both BayCon baselines on two datasets: UrbanGb~\cite{urbangb} and GPS trajectories~\cite{gps_trajectories_354} reporting longitude and latitude for car accidents and vehicles trajectories respectively.

We report average and standard error of the mean aggregated on the instances where both instances find a counterfactual. On \textit{UrbanGb}, ExDBSCAN achieves a proximity of $1.48 \pm 0.08$, BayCon Random has a proximity if $1.78 \pm 0.08$ and BayCon closest a proximity of $1.5 \pm 0.1$. The three methods have a validity of $0.74 \pm 0.03$, $0.44 \pm 0.02$ and $0.27 \pm 0.2$ respectively. Finally, BayCon achieves a validity of $0.31$ while ExDBSCAN achieves a validity of $1$.

In the GPS trajectories datasets instead, BayCon is unable to find a counterfactuals for any point and we thus reports only ExDBSCAN results, achieving $19.7 \pm 0.08$ proximity, $0.93 \pm 0.01$ diversity and a validity of $1$.